\documentclass{article} 
\usepackage{iclr2022_conference,times}

\usepackage{lifetime}
\usepackage{adjustbox}
\usepackage{multirow}
\usepackage{hyperref}
\usepackage{url}

\title{Leveraging Unlabeled Data to Predict \\ Out-of-Distribution Performance}

\author{%
  Saurabh Garg\thanks{Work done in part while Saurabh Garg was interning at Google} \\
  Carnegie Mellon University\\
  \small{\texttt{\href{mailto:sgarg2@andrew.cmu.edu}{sgarg2}@andrew.cmu.edu}}
   \And
  Sivaraman Balakrishnan\\
  Carnegie Mellon University\\
  \small{\texttt{\href{mailto:sbalakri@andrew.cmu.edu}{sbalakri}@andrew.cmu.edu}}
  \And 
  Zachary C. Lipton\\
  Carnegie Mellon University\\
  \small{\texttt{\href{mailto:zlipton@andrew.cmu.edu}{zlipton}@andrew.cmu.edu}}
  \And
  Behnam Neyshabur \\
  Google Research, Blueshift team\\
  \small{\texttt{\href{mailto:neyshabur@google.com}{neyshabur}@google.com}}\\
  \And
  Hanie Sedghi \\
  Google Research, Brain team\\
  \small{\texttt{\href{mailto:hsedghi@google.com}{hsedghi}@google.com}}
}

\newcommand{\update}[1]{\textcolor{black}{#1}}

\iclrfinalcopy 
\begin{document}
\maketitle

\begin{abstract}
    
Real-world machine learning deployments are characterized by mismatches between the source (training) and target (test) distributions that may cause performance drops. In this work, we investigate methods for predicting the target domain accuracy using only labeled source data and unlabeled target data. We propose Average Thresholded Confidence (ATC),
a practical method that learns a \emph{threshold} on the model's confidence, predicting accuracy as the fraction of unlabeled examples for which model confidence exceeds that threshold. ATC outperforms previous methods across several model architectures, types of distribution shifts 
(e.g., due to synthetic corruptions, dataset reproduction, or novel subpopulations), and datasets (\update{\textsc{Wilds}}, ImageNet, \breeds, CIFAR, and MNIST). In our experiments, ATC estimates target performance 
$2\text{--}4\times$ more accurately than prior methods. 
We also explore the theoretical foundations of the problem,
proving that, in general, identifying the accuracy
is just as hard as identifying the optimal predictor
and thus, the efficacy of any method
rests upon (perhaps unstated) assumptions 
on the nature of the shift. 
Finally, analyzing our method on some toy distributions,
we provide insights concerning when it works\footnote{Code is available at \url{https://github.com/saurabhgarg1996/ATC_code}.}.

\end{abstract}

\vspace{-5pt}
\section{Introduction}

Machine learning models deployed in the real world 
typically encounter examples 
from previously unseen distributions. 
While the IID assumption enables us to
evaluate models using held-out data 
from the \emph{source} distribution 
(from which training data is sampled),
this estimate is no longer valid
in presence of a distribution shift.
Moreover, under such shifts,
model accuracy tends to degrade
\citep{szegedy2013intriguing, recht2019imagenet,wilds2021}.
Commonly, the only data
available to the practitioner
are a labeled training set (source)
and unlabeled deployment-time data which makes the problem more difficult.
In this setting, detecting shifts 
in the distribution of covariates
is known to be possible (but difficult)
in theory~\citep{ramdas2015decreasing},
and in practice~\citep{rabanser2018failing}.
However, producing an optimal predictor 
using only labeled source and unlabeled target data
is well-known to be impossible absent further assumptions~\citep{ben2010impossibility,lipton2018detecting}. %

Two vital questions that remain are:
(i) the precise conditions under which 
we can estimate a classifier's target-domain accuracy;
and (ii) which methods are most practically useful.
To begin, the straightforward way 
to assess the performance 
of a model under distribution shift
would be to collect labeled (target domain) examples 
and then to evaluate the model on that data. 
However, collecting fresh labeled data 
from the target distribution 
is prohibitively expensive and time-consuming, 
especially if the target distribution is non-stationary. 
Hence, instead of using labeled data,
we aim to use 
unlabeled data from the target 
distribution, that is comparatively abundant, to predict model performance. 
Note that in this work, our focus 
is \emph{not} to improve performance on the target 
but, rather, to estimate the accuracy on the target for a given classifier.

Recently, numerous methods have been proposed 
for this purpose~\citep{deng2021labels,chen2021mandoline,jiang2021assessing,deng2021does,guillory2021predicting}. 
These methods either require calibration on the target domain 
to yield consistent estimates~\citep{jiang2021assessing,guillory2021predicting} 
or additional labeled data 
from several target domains 
to learn a linear regression function
on a distributional distance that then 
predicts model performance~\citep{deng2021does,deng2021labels,guillory2021predicting}. 
However, methods that require calibration 
on the target domain typically yield poor estimates 
since deep models trained and calibrated on 
source data are not, in general, calibrated on a 
(previously unseen) target domain~\citep{ovadia2019can}. 
Besides, methods that leverage  
labeled data from target domains rely
on the fact that unseen target domains exhibit strong 
linear correlation with seen target domains 
on the underlying distance measure and, hence, 
can be rendered ineffective when 
such target domains with labeled data 
are unavailable (in \secref{sec:exp_results} 
we demonstrate such a failure on a real-world 
distribution shift problem). 
Therefore, throughout the paper,
we assume access to labeled source data 
and only unlabeled data from target domain(s).

\begin{figure*}[t!]
    \centering 
    \includegraphics[width=\linewidth]{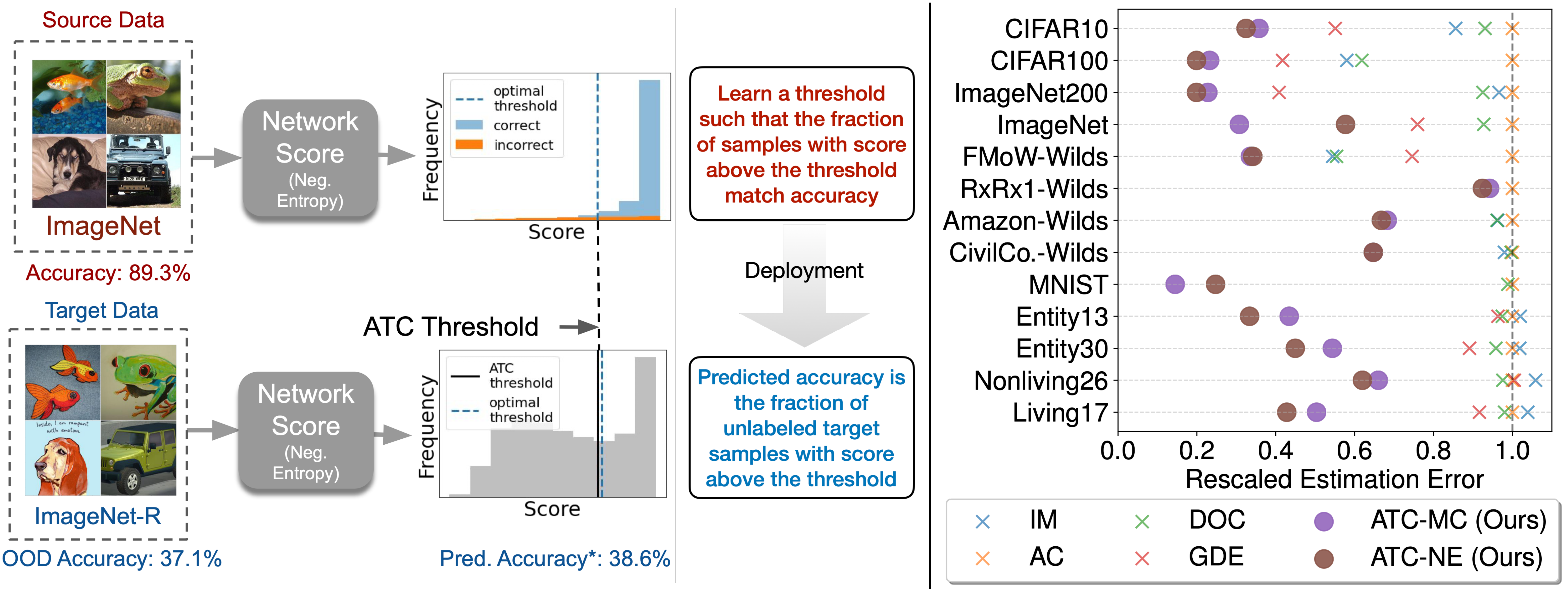}
    \vspace{-10pt}
    \caption{\emph{Illustration of our proposed method ATC.} {} \textbf{Left}: 
    using source domain validation data, 
    we identify a \emph{threshold}
    on a score (e.g. negative entropy)
    computed on model confidence
    such that fraction of examples above the threshold
    matches the validation set accuracy. 
    ATC estimates accuracy on unlabeled target data 
    as the fraction of examples 
    with the score above the threshold.
    Interestingly, this threshold yields accurate estimates
    on a wide set of target distributions 
    resulting from natural and synthetic shifts.{} 
    \textbf{Right}: \update{Efficacy of ATC 
    over previously proposed approaches on our testbed
    with a post-hoc calibrated model.} 
    To obtain errors on the same scale, 
    we rescale all errors with Average Confidence (AC) error. 
    Lower estimation error is better.
    See \tabref{table:error_estimation} for exact numbers 
    and comparison on various types of distribution shift. 
    See \secref{sec:exp} for details on our testbed.}
    \label{fig:intro}
    \vspace{-10pt}
  \end{figure*}

In this work, we first show that
{absent assumptions 
on the source classifier or the nature of the shift,
no method of estimating accuracy will work generally}
(even in non-contrived settings).
To estimate accuracy on target domain \emph{perfectly},  
we highlight that even given perfect knowledge 
of the labeled source distribution (i.e., $p_s(x,y)$) 
and unlabeled target distribution (i.e., $p_t(x)$),
we need restrictions on the nature of the shift 
such that we can uniquely identify 
the target conditional $p_t(y|x)$.  
Thus, in general, identifying the accuracy of the classifier
is as hard as identifying the optimal predictor.

Second, motivated by the superiority
of methods that use maximum softmax 
probability (or logit) of a model 
for Out-Of-Distribution (OOD) 
detection~\citep{hendrycks2016baseline,hendrycks2019scaling}, 
we propose a simple method 
that leverages softmax probability to
predict model performance.  
Our method, {Average Thresholded Confidence} (ATC), 
learns a threshold on a score 
(e.g., maximum confidence or negative entropy) 
of model confidence on validation source data 
and predicts target domain accuracy
as the fraction of unlabeled target points 
that receive a score above that threshold. 
ATC selects a threshold on validation source data 
such that the fraction of source examples   
that receive the score  
above the threshold match the 
accuracy of those examples.  
Our primary contribution in ATC
is the proposal of obtaining the threshold 
and observing its efficacy 
on (practical) accuracy estimation.  
Importantly, our work takes a step forward 
in positively answering the question raised 
in \citet{deng2021labels,deng2021does} 
about a practical strategy 
to select a threshold that enables 
accuracy prediction with thresholded model confidence.

ATC is simple to implement with existing frameworks, 
compatible with arbitrary model classes, 
and dominates other contemporary methods.
\update{Across several model architectures on a 
range of benchmark vision and 
language datasets, 
we verify that ATC outperforms
prior methods by at least $2$--$4\times$ in 
predicting target accuracy 
on a variety of distribution shifts.}
In particular, we consider 
shifts due to common corruptions (e.g., ImageNet-C), 
natural distribution shifts due to 
dataset reproduction (e.g., ImageNet-v2, ImageNet-R), 
shifts due to novel
subpopulations (e.g., \textsc{Breeds}), 
and distribution shifts faced in the wild (e.g., \textsc{Wilds}). 

As a starting point for theory development, 
we investigate ATC on a simple toy model
that models distribution shift 
with varying proportions of the population 
with spurious features,
as in \citet{nagarajan2020understanding}. 
Finally, we note that  
although ATC achieves 
superior performance in 
our empirical evaluation, like all methods, 
it must fail (returns inconsistent estimates) 
on certain types of distribution shifts,
per our impossibility result.

\vspace{-8pt}
\section{Prior Work}
\vspace{-5pt}
\textbf{Out-of-distribution detection. {} {}} 
The main goal of OOD detection is 
to identify previously unseen examples, 
i.e., samples out of the support
 of training distribution.  
To accomplish this, modern methods
utilize confidence or features learned 
by a deep network  trained 
on some source data. 
\citet{hendrycks2016baseline,geifman2017selective}
used the confidence score 
of an (already) trained deep model
to identify OOD points. 
\citet{lakshminarayanan2016simple} use entropy 
of an ensemble model to evaluate prediction 
uncertainty on OOD points. 
To improve OOD detection with model confidence, 
\citet{liang2017enhancing} propose to use 
temperature scaling and input perturbations.  
\citet{jiang2018trust} propose to use 
scores based on the relative 
distance of the predicted class 
to the second class. Recently, 
residual flow-based methods 
were used to obtain a density model 
for OOD detection~\citep{zhang2020hybrid}. 
\citet{ji2021predicting} proposed a method 
based on subfunction error bounds to compute 
unreliability per sample. Refer to 
\citet{ovadia2019can,ji2021predicting} for an overview 
and comparison of methods for prediction uncertainty 
on OOD data. 

\textbf{Predicting model generalization. {} {}}
Understanding generalization capabilities 
of overparameterized models on in-distribution 
data using conventional machine learning tools 
has been a focus of a long line of work; 
representative research includes~\citet{neyshabur2015norm,neyshabur2017exploring,neyshabur2017implicit,neyshabur2018role,dziugaite2017computing,bartlett2017spectrally,zhou2018non,long2019generalization, nagarajan2019deterministic}. 
At a high level, this line of research 
bounds the generalization gap directly 
with complexity measures calculated on
the trained model.
However, these bounds typically remain 
numerically loose relative to the true 
generalization error~\citep{zhang2016understanding,nagarajan2019uniform}. 
On the other hand, another line of research 
departs from complexity-based approaches to 
use unseen unlabeled data to predict in-distribution 
generalization~\citep{platanios2016estimating,platanios2017estimating,garg2021ratt,jiang2021assessing}.

Relevant to our work are methods for predicting 
the error of a classifier on OOD data 
based on unlabeled data from the target (OOD) domain. 
These methods can be characterized into two broad categories: 
(i) Methods which explicitly predict correctness of the model 
on individual unlabeled points~\citep{deng2021labels, jiang2021assessing,deng2021does, chen2021detecting}; and
(ii) Methods which directly obtain an estimate of error with unlabeled OOD data 
without making a point-wise prediction~\citep{chen2021mandoline,guillory2021predicting,chuang2020estimating}.

To achieve a consistent estimate of the target accuracy, 
~\citet{jiang2021assessing,guillory2021predicting} require
calibration on target domain. 
However, these methods typically yield poor estimates 
as deep models trained and calibrated on some source data 
are seldom calibrated on previously
unseen domains~\citep{ovadia2019can}.  
Additionally, \citet{deng2021labels,guillory2021predicting} 
derive model-based distribution statistics 
on unlabeled target set that correlate 
with the target accuracy and propose 
to use a subset of \emph{labeled} target 
domains to learn a (linear) regression 
function that predicts model performance. 
However, there are two drawbacks with this approach: 
(i) the correlation of these distribution 
statistics can vary substantially as we consider 
different nature of shifts 
(refer to \secref{sec:exp_results},
where we empirically demonstrate this failure);
(ii) even if there exists a (hypothetical) 
statistic with strong correlations, 
obtaining labeled target domains (even simulated ones) 
with strong correlations would require significant
\emph{a priori} knowledge about the nature of shift
that, in general, might not be available 
before models are deployed in the wild.  
Nonetheless, in our work, we only assume access 
to labeled data from the source domain presuming 
no access to labeled target domains or 
information about how to simulate them. 

Moreover, unlike the parallel work of \citet{deng2021does}, 
we do not focus on methods
that alter the training on source data 
to aid accuracy prediction on the target data. 
\citet{chen2021mandoline} propose an 
importance re-weighting based approach 
that leverages (additional) information 
about the axis along which distribution 
is shifting in form of ``slicing functions''. 
In our work, we make comparisons with 
importance re-weighting 
baseline from \citet{chen2021mandoline} 
as we do not have any additional 
information about the axis along which 
the distribution is shifting. 

\vspace{-8pt}
\section{Problem Setup} \label{sec:setup}
\vspace{-5pt}
\textbf{Notation.} By $\enorm{\cdot}$, 
and $\inner{\cdot}{\cdot}$
we denote the Euclidean norm
and inner product,
respectively.
For a vector $v\in \Real^d$, 
we use $v_j$ to denote its $j^\text{th}$ entry, 
and for an event $E$ we let $\indict{E}$ 
denote the binary indicator of the event.

Suppose we have a multi-class classification problem
with the input domain $\calX \subseteq \Real^d$ 
and label space $\calY = \{1, 2, \ldots, k\}$. For 
binary classification, 
we use $\calY = \{0,1\}$. 
By $\calD^\train$ and $\calD^\test$, 
we denote source and target distribution 
over $\calX \times \calY$. For distributions 
$\calD^\train$ and $\calD^\test$, we 
define $p_\train$ or $p_\test$ as the
corresponding probability density (or mass) 
functions. A dataset 
$S \defeq \{(x_i, y_i)\}_{i=1}^n \sim (\calD^\train)^n$
contains $n$ points sampled i.i.d. from $\calD^\train$.
Let $\calF$ be a class of hypotheses 
mapping $\calX$ to $\Delta^{k-1}$ where $\Delta^{k-1}$ is a simplex in $k$ dimensions. 
Given a classifier $f \in \calF$ and datum $(x, y)$,
we denote the 0-1 error 
(i.e., classification error)
on that point by 
$\error(f(x), y) \defeq \indict{ y\not\in \argmax_{j\in\calY} f_j(x) }$. 
Given a model $f \in \calF$, 
our goal in this work is to understand the performance 
of $f$ on $\calD^\test$ without access to 
labeled data from $\calD^\test$. Note that our goal
is not to adapt the model to the target data.  
Concretely, we aim to predict accuracy of $f$ 
on $\calD^\test$. Throughout this paper, 
we assume we have access to the following: 
(i) model $f$; (ii) previously-unseen (validation) data 
from $\calD^\train$; and
(iii) unlabeled data from target distribution $\calD^\test$. 

\vspace{-7pt}
\subsection{Accuracy Estimation: Possibility and Impossibility Results}
\vspace{-3pt}

First, we investigate the question 
of when it is possible 
to estimate the target accuracy
of an arbitrary classifier,
even given knowledge of the full
source distribution $p_s(x,y)$ 
and target marginal $p_t(x)$. 
Absent assumptions on the nature of shift,
estimating target accuracy is impossible. 
Even given access to $p_s(x,y)$ and $p_t(x)$,
the problem is fundamentally 
unidentifiable 
because $p_t(y|x)$ can shift arbitrarily. 
In the following proposition, 
we show that absent assumptions 
on the classifier $f$ 
(i.e., when $f$ can be any classifier 
in the space of all classifiers on $\calX$), 
we can estimate accuracy on the 
target data iff assumptions
on the nature of the shift,
together with $p_s(x, y)$ and $p_t(x)$,
uniquely identify the (unknown) target 
conditional $p_t(y|x)$. 
We relegate proofs from this section to 
\appref{app:proof_setup}. %

\begin{proposition} \label{prop:characterization}
Absent further assumptions,
accuracy on the 
target is identifiable 
iff $p_t(y|x)$ is uniquely identified 
given $p_s(x, y)$ and $p_t(x)$. 
\end{proposition}

\propref{prop:characterization} states that 
we need enough constraints on nature of shift 
such that $p_s(x, y)$ and $p_t(x)$
identifies unique $p_t(y|x)$. 
It also states that under some assumptions 
on the nature of the shift, 
we can hope to estimate 
the model's accuracy on target data. 
We will illustrate this
on two common assumptions 
made in domain adaptation literature:
(i) covariate shift~\citep{heckman1977sample,shimodaira2000improving} 
and (ii) label shift~\citep{saerens2002adjusting,zhang2013domain,lipton2018detecting}. 
Under covariate shift assumption, 
that the target marginal support $\textbf{supp}(p_t(x))$
is a subset of the source marginal support $\textbf{supp}(p_s(x))$
and that the conditional distribution of labels given inputs 
does not change within support,
i.e.,  $p_s(y|x) = p_t(y|x)$,
which, trivially, identifies 
a unique target conditional $p_t(y|x)$. 
Under label shift, the reverse holds, i.e.,
the class-conditional distribution 
does not change ($p_s(x|y) = p_t(x|y)$)
and, again, information about $p_t(x)$ uniquely determines 
the target conditional $p_t(y|x)$~\citep{lipton2018detecting,garg2020unified}. 
In these settings, one can estimate 
an arbitrary classifier's accuracy 
on the target domain either 
by using importance re-weighting 
with the ratio $p_t(x)/p_s(x)$
in case of covariate shift 
or by using importance re-weighting 
with the ratio $p_t(y)/p_s(y)$ 
in case of label shift. 
While importance ratios in the former case
can be obtained directly 
when $p_t(x)$ and $p_s(x)$ are known, 
the importance ratios in the latter case 
can be obtained by using techniques from ~\citet{saerens2002adjusting,lipton2018detecting,azizzadenesheli2019regularized, alexandari2019adapting}. 
In \appref{app:estimate_label_covariate},%
we explore accuracy estimation 
in the setting of these shifts and present extensions 
to generalized notions of label shift~\citep{tachet2020domain} 
and covariate shift~\citep{rojas2018invariant}. 

As a corollary of \propref{prop:characterization}, 
we now present a simple impossibility result,
demonstrating that no single method can work 
for all families of distribution shift. 

\begin{corollary} \label{corollary:impossible}
Absent assumptions on the classifier $f$, 
no method of estimating accuracy will 
work in all scenarios, i.e., for different 
nature of distribution shifts. 
\end{corollary}
Intuitively, this result states that 
every method of estimating accuracy on 
target data is tied up with some assumption
on the nature of the shift and might not 
be useful for estimating accuracy under a  
different assumption on the nature of the shift.
For illustration, consider a setting 
where we have access to distribution 
$p_s(x, y)$ and $p_t(x)$. 
Additionally, assume 
that the distribution can shift 
only due to covariate shift or
label shift without any knowledge 
about which one. 
Then \corollaryref{corollary:impossible} 
says that it is impossible to have a single 
method that will simultaneously for both label 
shift and covariate shift as in the following example
(we spell out the details in \appref{app:proof_setup}):%

\textbf{Example 1.}
Assume binary classification with  
$p_s(x) = \alpha \cdot \phi(\mu_1) + (1-\alpha)\cdot \phi(\mu_2)$, 
$p_s(x|y=0) = \phi(\mu_1)$, $p_s(x|y=1) =  \phi(\mu_2)$,  
and 
$p_t(x) = \beta\cdot \phi(\mu_1) + (1-\beta) \cdot \phi(\mu_2)$ 
where $\phi(\mu) = \calN(\mu, 1)$, $\alpha,\beta \in (0,1)$,
and $\alpha \ne \beta$. 
Error of a classifier $f$ on target data is given by 
$\calE_1=\Expt{(x,y)\sim p_s(x,y)}{\frac{p_t(x)}{p_s(x)} \indict{f(x)\ne y}}$
under covariate shift and by 
$\calE_2=\Expt{(x,y)\sim p_s(x,y)}{\left(\frac{\beta}{\alpha}\indict{y=0} + \frac{1-\beta}{1-\alpha}\indict{y=1} \right)\indict{f(x)\ne y}}$
under label shift. 
In \appref{app:proof_setup},
we show that $\calE_1 \ne \calE_2$ for all $f$. 
Thus, given access to $p_s(x,y)$, and $p_t(x)$, 
any method that consistently estimates 
error of a classifer under covariate shift will give an
incorrect estimate of error 
under label shift and vice-versa.   
The reason is that the same $p_t(x)$ and $p_s(x,y)$
can correspond to error $\calE_1$ (under covariate shift)
or error $\calE_2$ (under label shift) 
and determining which scenario one faces
requires further assumptions on the nature of shift.

\vspace{-8pt}
\section{Predicting accuracy with Average Thresholded Confidence}
\vspace{-5pt}
In this section, we present our method ATC 
that leverages a black box classifier $f$ 
and (labeled) validation source data to 
predict accuracy on target domain given
access to unlabeled target data. 
Throughout the discussion, we assume 
that the classifier $f$ is fixed.  

Before presenting our method, 
we introduce some terminology. 
Define a score function 
$s: \Delta^{k-1} \to \Real$ 
that takes in the softmax prediction
of the function $f$ 
and outputs a scalar. 
We want a score function 
such that if the score function 
takes a high value at a datum $(x,y)$ 
then $f$ is likely to be correct. 
In this work, we explore two such 
score functions: 
(i) Maximum confidence, i.e., 
$s(f(x)) = \underset{j \in \calY}{\max} f_j(x)$; 
and (ii) Negative Entropy, i.e.,
$s(f(x)) = \sum_j f_j(x) \log(f_j(x))$.
Our method identifies a threshold $t$ 
on source data $\calD^\train$
such that the expected number of points 
that obtain a score less than $t$ 
match the error of $f$ on $\calD^\train$, i.e., 
\begin{align}
    \Expt{x \sim \calD^\train }{\indict{ s(f(x)) < t}}  = \Expt{(x,y) \sim \calD^\train}{\indict{ \argmax_{j\in \out} f_j(x) \ne y}} \,, 
    \label{eq:ATC_thres}
\end{align}
and then our error estimate $\text{ATC}_{D^\test}(s)$
on the target domain $\calD^\test$ is given 
by the expected number of target points
that obtain a score less than $t$, i.e., 
\begin{align}
    \text{ATC}_{\calD^\test}(s) =  \Expt{x \sim \calD^\test}{\indict{ s(f(x)) < t}}  \,.
    \label{eq:ATC_pred}
\end{align}
In short, in \eqref{eq:ATC_thres},
ATC selects a threshold on the score function 
such that the error in the source domain 
matches the expected number of points 
that receive a score below $t$ 
and in \eqref{eq:ATC_pred},
ATC predicts error on the target domain 
as the fraction of unlabeled points 
that obtain a score below that threshold $t$. 
Note that, in principle, there exists 
a different threshold $t^\prime$
on the target distribution $\calD^\test$ 
such that \eqref{eq:ATC_thres}
is satisfied on $\calD^\test$. 
However, in our experiments,
the same threshold performs remarkably well.
The main empirical contribution of our work 
is to show that the threshold obtained with 
\eqref{eq:ATC_thres} might be used effectively
in condunction with modern deep networks
in a wide range of settings
to estimate error on the target data.
In practice, to obtain the threshold with ATC, 
we minimize the difference between the 
expression on two sides of \eqref{eq:ATC_thres}
using finite samples. 
In the next section,  
we show that ATC precisely
predicts accuracy on the OOD data  
on the desired line $y=x$. 
In \appref{app:interpretation}, %
we discuss an alternate 
interpretation of the method and make 
connections with OOD detection methods.

\vspace{-8pt}
\section{Experiments} \label{sec:exp}
\vspace{-7pt}

\begin{figure}[t]
\vspace{-20pt}
    \centering
    \subfigure{\includegraphics[width=0.32\linewidth]{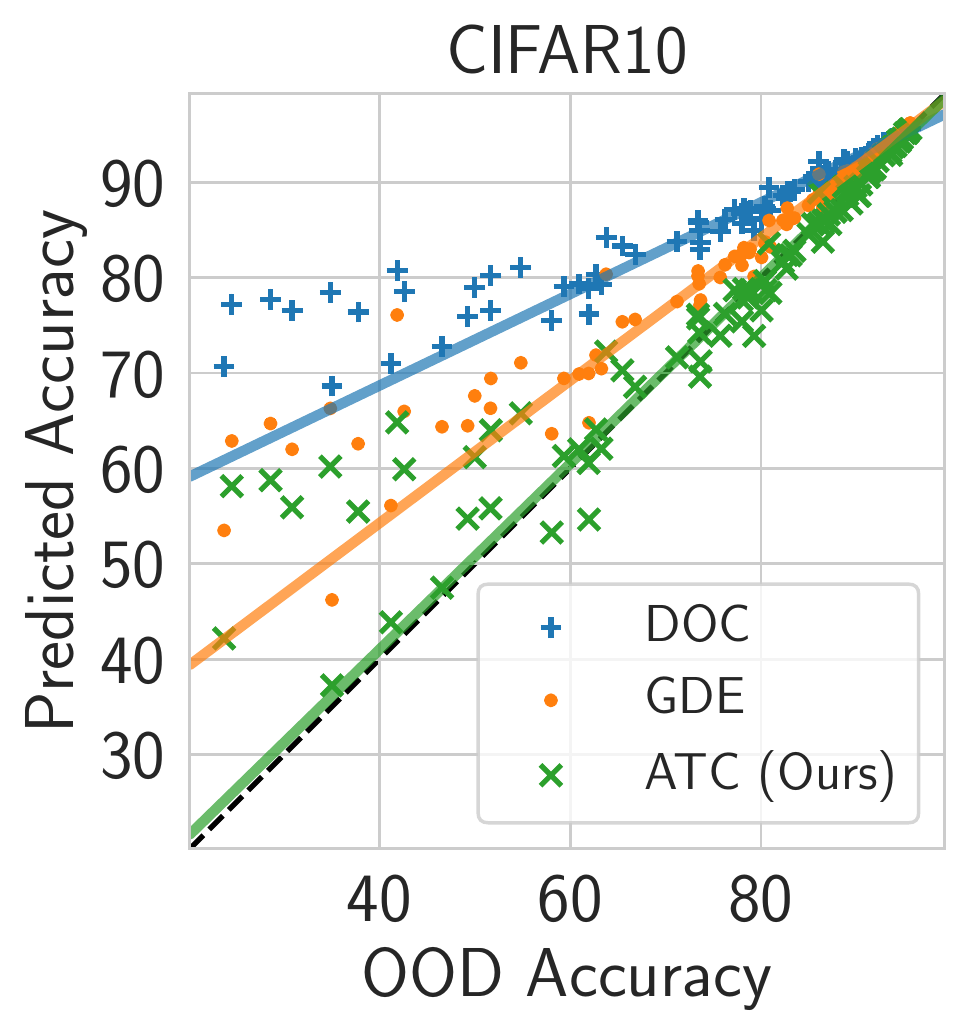}} \hfil
    \subfigure{\includegraphics[width=0.32\linewidth]{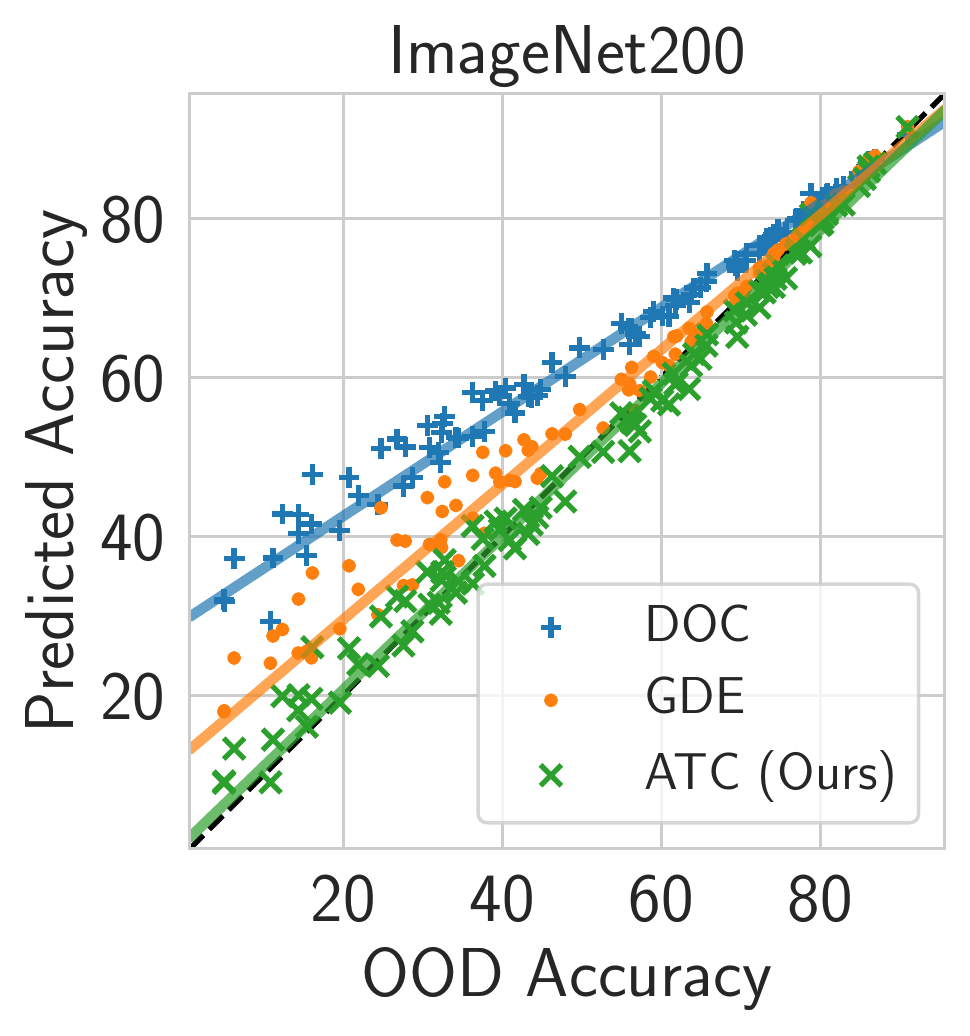}} \hfil   
    \subfigure{\includegraphics[width=0.32\linewidth]{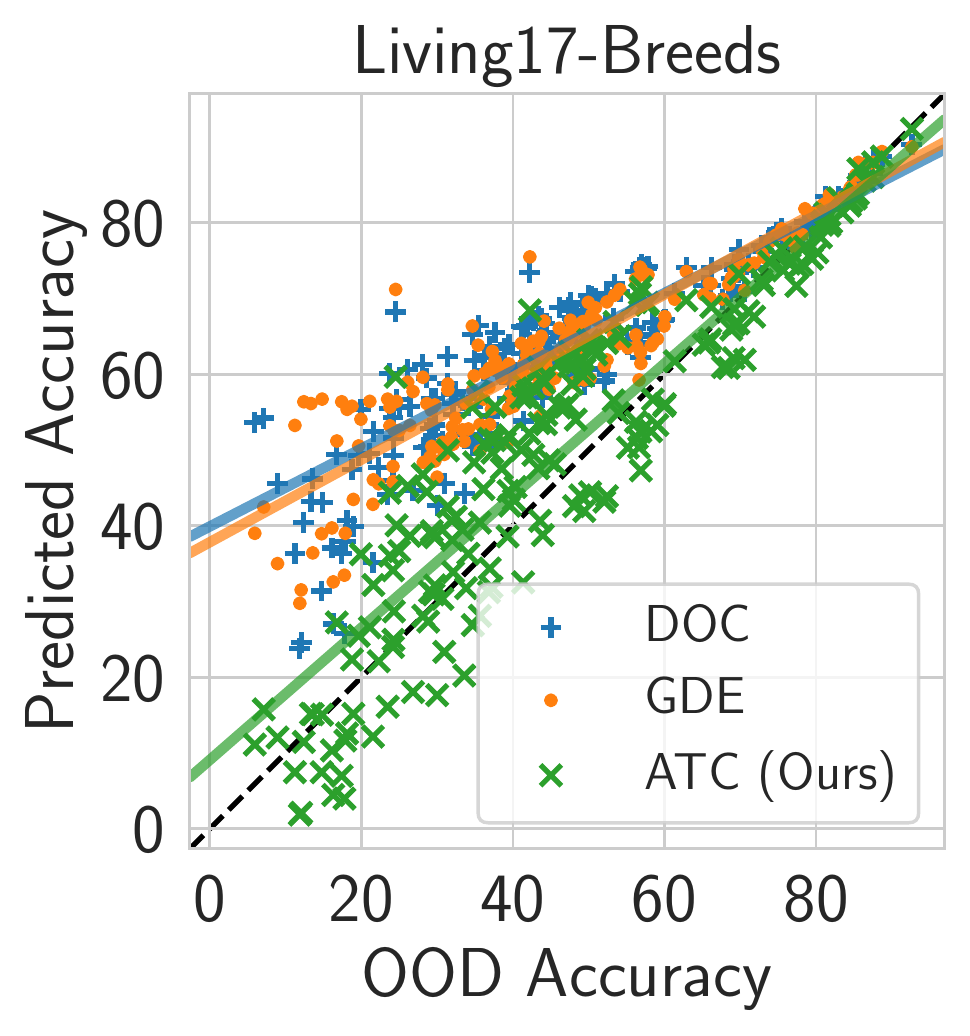}} 
    \vspace{-12pt}
    \caption{\emph{Scatter plot of predicted accuracy versus (true) OOD accuracy.}
    Each point denotes a different OOD dataset, 
    all evaluated with the same DenseNet121 model.
    We only plot the best three methods. 
    With ATC (ours), we refer to ATC-NE. 
    We observe that ATC significantly outperforms other methods 
    and with ATC, we recover the desired line $y=x$ with a robust linear fit. 
    Aggregated estimation error in \tabref{table:error_estimation} 
    and plots for other datasets and architectures in \appref{app:results}.} 
    \vspace{-15pt}
    \label{fig:scatter_plot}
  \end{figure}

We now empirical evaluate 
ATC and compare it with existing methods. 
In each of our main experiment, 
keeping the underlying model fixed, 
we vary target datasets and make a 
prediction of the target accuracy 
with various methods given access to 
only unlabeled data from the target. 
Unless noted otherwise, all models are  
trained only on samples from the source 
distribution with the main exception of pre-training
on a different distribution. 
We use labeled examples from the 
target distribution to only obtain 
true error estimates. 

\textbf{Datasets. {}} First, we consider 
synthetic shifts induced due to 
different visual corruptions (e.g., shot noise, 
motion blur etc.) under 
ImageNet-C~\citep{hendrycks2019benchmarking}. 
Next, we consider natural shifts 
due to differences in the data 
collection process of ImageNet~\citep{russakovsky2015imagenet}, 
e.g, ImageNetv2~\citep{recht2019imagenet}. 
We also consider images with artistic 
renditions of object classes, i.e., 
ImageNet-R~\citep{hendrycks2021many} 
and ImageNet-Sketch~\citep{wang2019learning}. 
Note that renditions dataset only contains 
a subset $200$ 
classes from ImageNet. 
To include renditions dataset
in our testbed, we include results on 
ImageNet restricted 
to these $200$ classes (which we call 
ImageNet-200) along with full ImageNet. 

Second, we consider \textsc{Breeds}~\citep{santurkar2020breeds} 
to assess robustness to subpopulation shifts, in particular, 
to understand how accuracy estimation 
methods behave when novel subpopulations 
not observed during training are introduced. 
\textsc{Breeds} leverages class hierarchy 
in ImageNet to create 4 datasets \textsc{Entity-13}, \textsc{Entity-30},
\textsc{Living-17}, \textsc{Non-living-26}. We focus on 
natural and synthetic shifts as in ImageNet on same and different
subpopulations in BREEDs. 
\update{Third, from \textsc{Wilds}~\citep{wilds2021} benchmark, we consider 
FMoW-\textsc{Wilds}~\citep{christie2018functional}, 
RxRx1-\textsc{Wilds}~\citep{taylor2019rxrx1}, 
Amazon-\textsc{Wilds}~\citep{ni2019justifying}, 
CivilComments-\textsc{Wilds}~\citep{borkan2019nuanced}
to consider distribution shifts faced in the wild.
}

Finally, similar to ImageNet, 
we consider (i) synthetic shifts (CIFAR-10-C) due to 
common corruptions; and (ii) natural shift 
(i.e., CIFARv2~\citep{recht2018cifar})
on CIFAR-10~\citep{krizhevsky2009learning}.  
On CIFAR-100, we just have  
synthetic shifts due to common corruptions. 
For completeness, we also consider natural shifts 
on MNIST~\citep{lecun1998mnist}
as in the prior work~\citep{deng2021labels}. 
We use three real shifted datasets, i.e., USPS~\citep{hull1994database},
SVHN~\citep{netzer2011reading} and QMNIST~\citep{qmnist-2019}. 
We give a detailed overview of our setup in \appref{app:dataset}.  %

\textbf{Architectures and Evaluation. {}{}}
\update{For ImageNet, \textsc{Breeds}, CIFAR, FMoW-\textsc{Wilds}, 
RxRx1-\textsc{Wilds}
datasets, we use DenseNet121~\citep{huang2017densely} 
and ResNet50~\citep{he2016deep} architectures. 
For Amazon-\textsc{Wilds} and CivilComments-\textsc{Wilds},
we fine-tune a  DistilBERT-base-uncased~\citep{Sanh2019DistilBERTAD} model. }
For MNIST, we train a fully 
connected multilayer perceptron.  
We use standard training
with benchmarked hyperparameters. 
To compare methods, we report average 
absolute difference between 
the true accuracy on the target data and the 
estimated accuracy on the same unlabeled examples. 
We refer to this metric as Mean Absolute estimation Error~(MAE). 
Along with MAE, we also show scatter plots 
to visualize performance at individual target sets. 
Refer to \appref{app:exp_setup} for additional %
details on the setup.

\begin{table}[t]
    \begin{adjustbox}{width=\columnwidth,center}
    \centering
    \tabcolsep=0.12cm
    \renewcommand{\arraystretch}{1.2}
    \begin{tabular}{@{}*{13}{c}@{}}
    \toprule
    \multirow{2}{*}{Dataset} & \multirow{2}{*}{Shift} & \multicolumn{2}{c}{IM} & \multicolumn{2}{c}{AC} & \multicolumn{2}{c}{DOC} & GDE & \multicolumn{2}{c}{ATC-MC (Ours)} & \multicolumn{2}{c}{ATC-NE (Ours)} \\
    & & Pre T & Post T & Pre T & Post T  & Pre T & Post T & Post T & Pre T & Post T & Pre T & Post T \\
    \midrule
    \multirow{2}{*}{\parbox{1.2cm}{CIFAR10} }  & Natural & $6.60$ &$5.74$ &$9.88$ &$6.89$ &$7.25$ &$6.07$ &$4.77$ &$3.21$ &$3.02$ &$2.99$ & $\bf 2.85$   \\
    & Synthetic & $12.33$ &$10.20$ &$16.50$ &$11.91$ &$13.87$ &$11.08$ &$6.55$ &$4.65$ &$4.25$ &$4.21$ &$\bf 3.87$ \\ 
    \midrule 
    CIFAR100 & Synthetic & $13.69$ &$11.51$ &$23.61$ &$13.10$ &$14.60$ &$10.14$ &$9.85$ &$5.50$ &$\bf 4.75$ &$\bf 4.72$ &$4.94$ \\
    \midrule  
    \multirow{2}{*}{\parbox{1.8cm}{ImageNet200} }  & Natural & $12.37$ &$8.19$ &$22.07$ &$8.61$ &$15.17$ &$7.81$ &$5.13$ &$4.37$ &$2.04$ &$3.79$ & $\bf 1.45$ \\
    & Synthetic & $19.86$ &$12.94$ &$32.44$ &$13.35$ &$25.02$ &$12.38$ &$5.41$ &$5.93$ &$3.09$ &$5.00$ &$\bf 2.68$  \\ 
    \midrule 
    \multirow{2}{*}{\parbox{1.8cm}{\centering ImageNet} }  & Natural &
    $7.77$ &$6.50$ &$18.13$ &$6.02$ &$8.13$ &$5.76$ &$6.23$ &$3.88$ &$2.17$ &$2.06$ &$\bf 0.80$ \\
        & Synthetic &$13.39$ &$10.12$ &$24.62$ &$8.51$ &$13.55$ &$7.90$ &$6.32$ &$3.34$ &$\bf 2.53$ &$\bf 2.61$ &$4.89$ \\ 
    \midrule 
    FMoW-\textsc{wilds} & Natural & $5.53$ &$4.31$ &$33.53$ &$12.84$ &$5.94$ &$4.45$ &$5.74$ &$3.06$ &$\bf 2.70$ &$3.02$ &$\bf 2.72$  \\
    \midrule 
    RxRx1-\textsc{wilds} & Natural & $5.80$ &$5.72$ &$7.90$ &$4.84$ &$5.98$ &$5.98$ &$6.03$ &$4.66$ &$\bf 4.56$ &$\bf 4.41$ &$\bf 4.47$ \\
    \midrule 
    Amazon-\textsc{wilds} & Natural & $2.40$ &$2.29$ &$8.01$ &$2.38$ &$2.40$ &$2.28$ &$17.87$ &$1.65$ &$\bf 1.62$ &$ \bf 1.60$ &$\bf 1.59$ \\
    \midrule 
    CivilCom.-\textsc{wilds} & Natural & $12.64$ &$10.80$ &$16.76$ &$11.03$ &$13.31$ &$10.99$ &$16.65$ & \multicolumn{4}{c}{$\bf 7.14$} \\
    \midrule 
    MNIST & Natural &$18.48$ &$15.99$ &$21.17$ &$14.81$ &$20.19$ &$14.56$ &$24.42$ &$5.02$ &$\bf 2.40$ &$3.14$ &$3.50$ \\
    \midrule 
    \multirow{2}{*}{\parbox{1.8cm}{\textsc{Entity-13}} } & Same & $16.23$ &$11.14$ &$24.97$ &$10.88$ &$19.08$ &$10.47$ &$10.71$ &$5.39$ &$\bf 3.88$ &$4.58$ &$4.19$ \\
    & Novel & $28.53$ &$22.02$ &$38.33$ &$21.64$ &$32.43$ &$21.22$ &$20.61$ &$13.58$ &$10.28$ &$12.25$ &$\bf 6.63$ \\
    \midrule 
    \multirow{2}{*}{\parbox{1.8cm}{\textsc{Entity-30}}} & Same & $18.59$ &$14.46$ &$28.82$ &$14.30$ &$21.63$ &$13.46$ &$12.92$ &$9.12$ &$\bf 7.75$ &$8.15$ &$ \bf 7.64$ \\
    & Novel & $32.34$ &$26.85$ &$44.02$ &$26.27$ &$36.82$ &$25.42$ &$23.16$ &$17.75$ &$14.30$ &$15.60$ &$\bf 10.57$  \\
    \midrule
    \multirow{2}{*}{{\textsc{Nonliving-26}}} & Same & $18.66$ &$17.17$ &$26.39$ &$16.14$ &$19.86$ &$15.58$ &$16.63$ &$10.87$ &$\bf 10.24$ &$10.07$ &$\bf 10.26$ \\
    & Novel &$33.43$ &$31.53$ &$41.66$ &$29.87$ &$35.13$ &$29.31$ &$29.56$ &$21.70$ &$20.12$ &$19.08$ &$\bf 18.26$  \\ 
    \midrule 
    \multirow{2}{*}{\parbox{1.8cm}{\textsc{Living-17}}} & Same & $12.63$ &$11.05$ &$18.32$ &$10.46$ &$14.43$ &$10.14$ &$9.87$ &$4.57$ &$\bf  3.95$ &$\bf 3.81$ &$4.21$   \\
    & Novel & $29.03$ &$26.96$ &$35.67$ &$26.11$ &$31.73$ &$25.73$ &$23.53$ &$16.15$ &$14.49$ &$12.97$ &$\bf 11.39$ \\ 
    \bottomrule 
    \end{tabular}
  \end{adjustbox}  
    \vspace{-5pt}
    \caption{
    \emph{Mean Absolute estimation Error (MAE) results for different datasets in our setup grouped by the nature of shift.} 
    `Same' refers to same subpopulation shifts and `Novel' refers novel subpopulation shifts. \update{We include details about the target sets considered in each shift in \tabref{table:dataset}.} 
    Post T denotes use of TS calibration on source.  
    Across all datasets, we observe that ATC achieves superior performance (lower MAE is better).  
    \update{For language datasets, we use DistilBERT-base-uncased, for vision dataset we report results with DenseNet model with the exception of MNIST where we use FCN.} 
    We include results on other architectures in \appref{app:results}. 
    For GDE post T and pre T estimates match since TS doesn't alter the argmax prediction. \update{ Results reported by aggregating MAE numbers over $4$ different seeds. We include results with standard deviation values in \tabref{table:error_estimation_std}.}
    }\label{table:error_estimation}
    \vspace{-10pt}
\end{table}

\textbf{Methods {} {}} With ATC-NE, we denote ATC with negative entropy score function and with ATC-MC, 
we denote ATC with maximum confidence score function. 
For all methods, we implement \emph{post-hoc} calibration
on validation source data 
with Temperature Scaling (TS;~\citet{guo2017calibration}). 
Below we briefly discuss baselines methods 
compared in our work and relegate details to \appref{app:baselines}.  %

\emph{Average Confidence (AC). {}{}} Error is estimated as the expected value of the maximum softmax confidence  on the target data, i.e, $\AC_{\calD^\test} = \Expt{x \sim \calD^\test}{ \max_{j\in\calY} f_j(x)}$. 

\emph{Difference Of Confidence (DOC). {}{}} We estimate error on target by subtracting difference of confidences on source and target (as a surrogate to distributional distance~\citet{guillory2021predicting}) from the error on source distribution, i.e, $\DOC_{\calD^\test} = \Expt{x \sim \calD^\train}{ \indict{\argmax_{j\in\calY} f_j(x) \ne y}} + \Expt{x \sim \calD^\test}{ \max_{j\in\calY} f_j(x)} - \Expt{x \sim \calD^\train}{ \max_{j\in\calY} f_j(x)}$. This is referred to as DOC-Feat in \citep{guillory2021predicting}. 

\emph{Importance re-weighting (IM). {} {}} 
We estimate the error of the classifier 
with importance 
re-weighting of 0-1 error in the 
pushforward space 
of the classifier. 
This corresponds to \textsc{Mandolin} 
using one slice based on the underlying 
classifier confidence~\citet{chen2021mandoline}. 

\emph{Generalized Disagreement Equality (GDE). {} {}} 
Error is estimated as the expected disagreement of two models (trained on the same training set but with different randomization) on target data~\citep{jiang2021assessing}, i.e., $\GDE_{\calD^\test} = \Expt{x \sim \calD^\test}{\indict{f(x) \ne f^\prime(x)}}$ where $f$ and $f^\prime$ are the two models. Note that GDE requires two models trained independently, doubling the computational overhead while training.

\vspace{-10pt}
\subsection{Results} \label{sec:exp_results}
\vspace{-5pt}

\begin{figure}[t]
\vspace{-10pt}
    \centering
    \subfigure{\includegraphics[width=0.3\linewidth]{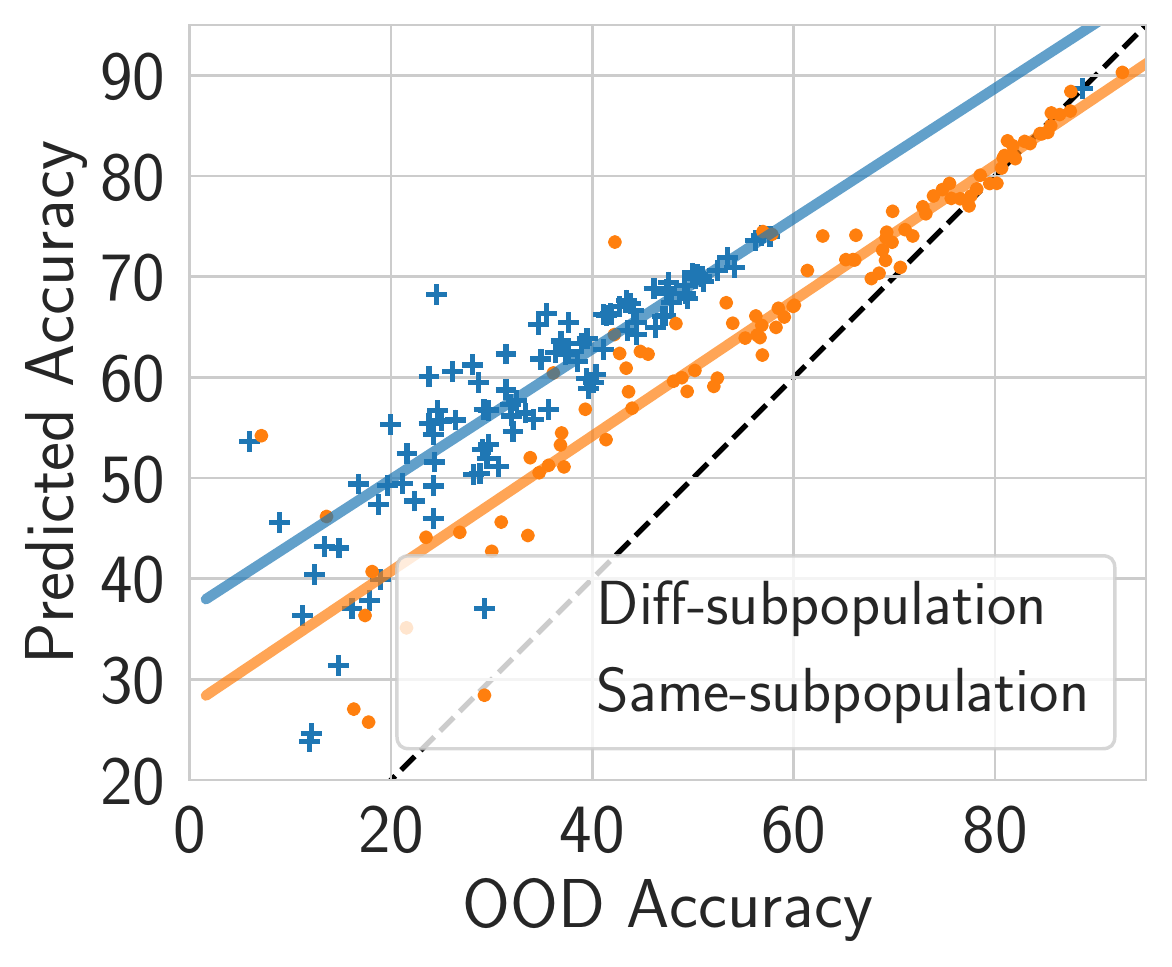}} \hfil
    \subfigure{\includegraphics[width=0.3\linewidth]{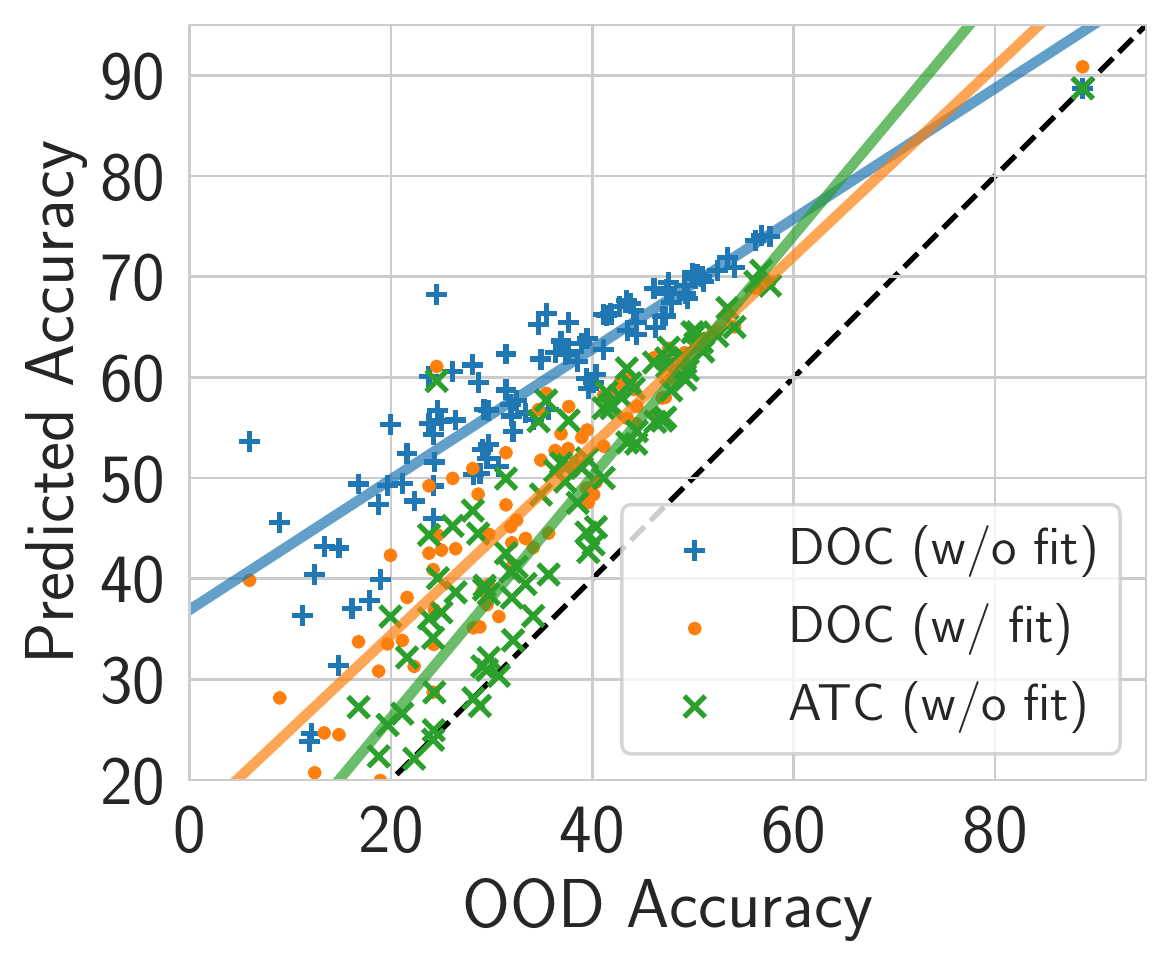}} \hfil   
    \subfigure{\includegraphics[width=0.35\linewidth]{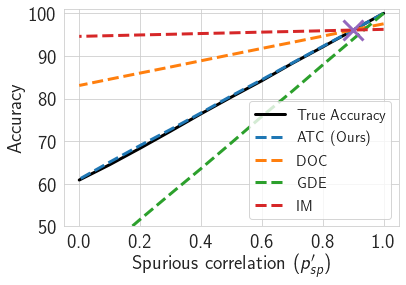}} 
    \vspace{-5pt}
    \caption{\textbf{Left:} Predicted accuracy with DOC on Living17 \breeds~dataset. We observe a substantial gap in the linear fit of same and different subpopulations highlighting poor correlation. \textbf{Middle:} After fitting a robust linear model for DOC on same subpopulation, we show predicted accuracy on different subpopulations with fine-tuned DOC (i.e., DOC (w/ fit)) and compare with ATC without any regression model, i.e., ATC (w/o fit).  While observe substantial improvements in MAE from $24.41$ with DOC (w/o fit) to $13.26$ with DOC (w/ fit), ATC (w/o fit) continues to outperform even DOC (w/ fit) with MAE $10.22$. We show parallel results with other \breeds~datasets in \appref{app:breeeds_ablation}. \textbf{Right :} Empirical validation of our toy model. We show that ATC perfectly estimates target performance as we vary the degree of spurious correlation in target. `$\times$' represents accuracy on source. }
    \vspace{-10pt}
    \label{fig:ablation}
\end{figure}  
  
In \tabref{table:error_estimation}, we report   
MAE results aggregated by 
the nature of the shift in our testbed. 
In \figref{fig:scatter_plot} 
and \figref{fig:intro}(right), we show scatter plots 
for predicted accuracy versus OOD accuracy on 
several datasets. We include 
scatter plots for all datasets and parallel 
results with other architectures in \appref{app:results}. %
In \appref{app:cifar_result}, we also perform ablations on CIFAR using a pre-trained model 
and observe that pre-training doesn't change the efficacy of ATC. 

We predict accuracy on the target
data before and after 
calibration with TS. First, we observe that 
both ATC-NE and ATC-MC (even without TS) obtain 
significantly lower MAE when compared with
other methods (even with TS). 
Note that with TS we observe substantial 
improvements in MAE for all methods.
Overall, ATC-NE (with TS) typically achieves the smallest MAE improving by more than $2\times$ on CIFAR and by $3$--$4\times$ on ImageNet over GDE (the next best alternative to ATC). Alongside, we also observe that a \update{linear fit with robust regression~\citep{siegel1982robust}} on the scatter plot recovers a line close to $x =y$ for ATC-NE with TS while the line is far away from $x=y$ for other methods (\figref{fig:scatter_plot} and \figref{fig:intro}(right)). 
Remarkably, MAE is in the range of $0.4$--$5.8$ with ATC for CIFAR, ImageNet, MNIST, and Wilds. However, MAE is much higher on \textsc{Breeds} benchmark with novel subpopulations. While we observe a small MAE (i.e., comparable to our observations on other datasets) on \textsc{Breeds} with natural and synthetic shifts from the same sub-population, MAE on shifts with novel population is significantly higher with all methods. Note that even on novel populations, ATC continues to dominate all other methods across all datasets in \textsc{Breeds}.

Additionally, for different subpopulations in \breeds~setup, we observe 
a poor linear correlation of the estimated performance with the actual performance as shown in \figref{fig:ablation} (left)(we notice a similar gap in the linear fit for all other methods). Hence in such a setting, we would expect methods that fine-tune a regression model on labeled target examples from shifts with one subpopulation will perform poorly on shifts with different subpopulations. Corroborating this intuition, next, we show that even after fitting a regression model for DOC on natural and synthetic shifts with source subpopulations, ATC without regression model continues to outperform DOC with regression model on shifts with novel subpopulation. 

\textbf{Fitting a regression model on \breeds~with DOC. } 
Using label target data from natural and synthetic shifts for the same subpopulation (same as source), we fit a robust linear regression model~\citep{siegel1982robust} to fine-tune DOC as in \citet{guillory2021predicting}. We then evaluate the fine-tuned DOC (i.e., DOC with linear model) on natural and synthetic shifts from novel subpopulations on \breeds~benchmark. Although we observe significant improvements in the performance of fine-tuned DOC when compared with DOC (without any fine-tuning), ATC without any regression model continues to perform better (or similar) to that of fine-tuned DOC on novel subpopulations (\figref{fig:ablation} (middle)). Refer to \appref{app:breeeds_ablation} for details and \tabref{table:breeds_regression} for MAE on \breeds~with regression model.  %

\vspace{-8pt}
\section{Investigating ATC on Toy Model} \label{sec:toy}
\vspace{-7pt}
In this section, we propose 
and analyze a simple theoretical 
model that distills empirical phenomena
from the previous section  
and highlights efficacy of ATC. 
Here, our aim is not to obtain a general 
model that captures complicated real 
distributions on high dimensional input space 
as the images in ImageNet. Instead to further 
our understanding, we focus on an
\emph{easy-to-learn} binary classification task from 
\citet{nagarajan2020understanding}
with linear classifiers, 
that is rich enough to  
exhibit some of the same phenomena
as with deep networks on real data distributions. 

Consider a easy-to-learn binary classification problem with 
two features $x = [x_{\inv}, x_{\spr}] \in \Real^2$ where 
$x_\inv$ is fully predictive 
invariant feature with a margin $\gamma >0$
and $x_\spr \in \{-1, 1\}$ is a spurious feature (i.e., 
a feature that is correlated but
not predictive of the true label). 
Conditional on $y$, the distribution
over $x_\inv$ is 
given as follows: $x_\inv | (y=1)  \sim U[\gamma, c]$ 
and  $x_\inv | (y = 0) \sim U[-c, -\gamma]$, 
where $c$ is a fixed constant greater than $\gamma$. 
For simplicity, we assume that label distribution 
on source is uniform on $\{-1, 1\}$.   
$x_\spr$ is distributed such that 
$P_s[ x_\spr \cdot (2y-1) > 0] = p_\spr $, where 
$p_\spr \in (0.5, 1.0)$ controls 
the degree of spurious correlation. 
To model distribution shift, 
we simulate target data with different
degree of spurious correlation, i.e., in target distribution 
$P_t[ x_\spr \cdot (2y-1) > 0] = p_\spr^\prime \in [0,1]$.  
Note that here we do not consider shifts in the label 
distribution but our result extends to arbitrary shifts 
in the label distribution as well. 

In this setup, we examine linear sigmoid classifiers of the form  
$f(x) = \left[ \frac{1}{1 + e^{w^T x}}, \frac{ e^{w^T x}}{1 + e^{w^T x}}\right]$
where $w = [w_{\inv}, w_\spr] \in \Real^2$. 
While there exists a linear classifier with 
$w = [1,0]$ that correctly classifies all the points
with a margin $\gamma$,  \citet{nagarajan2020understanding}
demonstrated that a linear classifier will 
typically have a dependency on the spurious feature, 
i.e., $w_\spr \ne 0$. They show that due to geometric skews, 
despite having positive dependencies on the invariant feature, 
a max-margin classifier trained on finite samples
relies on the spurious feature. 
Refer to \appref{app:toy_model} 
for more details on these skews. 
In our work, we show that given a linear classifier 
that relies on the spurious feature 
and achieves a non-trivial performance on 
the source (i.e., $w_\inv > 0$),  
ATC with maximum confidence score function  
\emph{consistently}
estimates the accuracy on the target distribution. 

\begin{theorem}[Informal] \label{thm:toy_theory}
Given any classifier with $w_\inv > 0$ in the above setting, 
the threshold obtained in \eqref{eq:ATC_thres} together with 
ATC as in \eqref{eq:ATC_pred}
with maximum confidence score function
obtains a consistent estimate of the target accuracy. 
\end{theorem}

Consider a classifier that depends positively 
on the spurious feature (i.e., $w_{\spr} > 0$). 
Then as the spurious correlation 
decreases in the target data, 
the classifier accuracy on the target will drop and 
vice-versa if the spurious correlation increases on the target data. 
\thmref{thm:toy_theory} shows that the threshold identified with 
ATC as in \eqref{eq:ATC_thres} remains invariant as the distribution shifts
and hence ATC as in \eqref{eq:ATC_pred}
will correctly estimate the 
accuracy with shifting distributions. 
Next, we illustrate \thmref{thm:toy_theory} 
by simulating the setup empirically. 
First we pick a arbitrary classifier
(which can also be obtained by training 
on source samples), 
tune the threshold on hold-out source examples 
and predict accuracy with different methods as
we shift the distribution by varying the degree 
of spurious correlation. 

\textbf{Empirical validation and comparison with other methods. {} {}}
\figref{fig:ablation}(right) shows that as the degree of spurious correlation varies,
our method accurately estimates the target performance where  
all other methods fail to accurately estimate the target performance. 
Understandably, due to poor calibration of the sigmoid linear 
classifier AC, DOC and GDE fail. While in principle IM 
can perfectly estimate the accuracy on target in this case, 
we observe that it is highly sensitive to the number bins and choice of 
histogram binning (i.e., uniform mass or equal width binning).
We elaborate more on this in \appref{app:toy_model}.  %

\textbf{Biased estimation with ATC. {} {}} 
Now we discuss changes in the above 
setup where ATC yields 
inconsistent estimates. 
We assumed that both in 
source and target $x_{\inv}|y=1$ 
is uniform between $[\gamma, c]$ 
and $x|y=-1$ is uniform 
between $[-c, -\gamma]$. 
Shifting the support of target class 
conditional $p_t(x_{\inv}|y)$ may 
introduce a bias in ATC estimates, e.g.,  
shrinking the support to 
$c_1$($<c$) (while maintaining 
uniform distribution) 
in the target will lead to an over-estimation 
of the target performance with ATC.  
In \appref{app:general_result}, %
we elaborate on this 
failure and present a general
(but less interpretable)  
classifier dependent distribution shift 
condition where ATC is guaranteed to yield consistent 
estimates.

\vspace{-8pt}
\section{Conclusion and future work}
\vspace{-7pt}
In this work, we proposed ATC, a simple method for 
estimating target domain accuracy 
based on unlabeled target 
(and labeled source data).
ATC achieves remarkably low estimation error 
on several synthetic and natural shift benchmarks 
in our experiments. 
Notably, our work draws inspiration from recent 
state-of-the-art methods that use softmax confidences 
below a certain threshold for OOD detection~\citep{hendrycks2016baseline,hendrycks2019scaling}   
and takes a step forward in answering questions
raised in \citet{deng2021labels} 
about the practicality of threshold based methods.   

Our distribution shift toy model justifies ATC 
on an easy-to-learn binary classification task. 
In our experiments, we also observe
that calibration significantly improves estimation with ATC. 
Since in binary classification, post hoc calibration with TS 
does not change the effective threshold, 
in future work, we hope to extend our theoretical model 
to multi-class classification to understand the efficacy of calibration. 
Our theory establishes that a classifier's accuracy 
is not, in general identified, from labeled source 
and unlabeled target data alone, 
absent considerable additional constraints
on the target conditional $p_t(y|x)$.
In light of this finding, 
we also hope to extend our understanding 
beyond the simple theoretical toy model
to characterize broader sets of conditions 
under which ATC might be guaranteed %
to obtain consistent estimates. 
Finally, we should note that
while ATC outperforms previous approaches, 
it still suffers from large estimation error
on datasets with novel populations, e.g., \breeds. 
We hope that our findings can lay the groundwork 
for future work for improving accuracy estimation on such datasets.


\paragraph{Reproducibility Statement}
Our code to reproduce all the results is available at \url{https://github.com/saurabhgarg1996/ATC_code}.
We have been careful to ensure 
that our results are reproducible. 
We have stored all models 
and logged all hyperparameters and seeds 
to facilitate reproducibility. 
Note that throughout our work,
we do not perform any hyperparameter tuning, 
instead, using benchmarked hyperparameters and training procedures 
to make our results easy to reproduce. 
While, we have not released code yet, 
the appendix provides all the necessary details 
to replicate our experiments and results.

\section*{Acknowledgement}
Authors would like to thank Ariel Kleiner and Sammy Jerome as the problem formulation and motivation of this paper was highly influenced by initial discussions with them.

\bibliography{iclr2022_conference}
\bibliographystyle{iclr2022_conference}

\newpage 
\appendix
\section*{Appendix}
\section{Proofs from ~\secref{sec:setup}} \label{app:proof_setup}

Before proving results from \secref{sec:setup}, we introduce some notations. 
Define $\error(f(x), y) \defeq \indict{ y\not\in \argmax_{j\in\calY} f_j(x) }$. 
We express the \emph{population error} on distribution $\calD$ as
$\error_\calD (f) \defeq \Expt{(x,y) \sim \calD}{\error(f(x),y)}$. 

\begin{proof}[Proof of \propref{prop:characterization}] Consider a binary classification problem. Assume $\calP$ be the set of possible target conditional distribution of labels given $p_s(x,y)$ and $p_t(x)$. 
    
    The forward direction is simple. If $\calP = \{p_t(y|x)\}$ is singleton given $p_s(x,y)$ and $p_t(x)$, then the error of any classifier $f$ on the target domain is identified and is given by 
    \begin{align}
       \error_{\calD^T}(f) =  \Expt{x\sim p_t(x), y\sim p_t(y|x)}{\indict{ \argmax_{j\in \calY} f_j(x) \ne y}} \,.        
    \end{align}

    For the reverse direction assume that given $p_t(x)$ and $p_s(x,y)$, we have 
    two possible distributions $\calD^T$ and $\calD^{T^\prime}$ with $p_t(y|x), p_t^\prime(y|x) \in \calP$ such that on some $x$ with $p_t(x) > 0$, we have  $p_t(y|x) \ne p_t^\prime(y|x)$. 
    Consider $\calX_M = \{ x\in \calX | p_t(x) > 0 \text{ and } p_t(y=1|x) \ne p_t^\prime(y=1|x)\}$ be the set of all input covariates where the two distributions differ. We will now  choose a classifier $f$ such that the error on the two distributions differ. On a subset $\calX_M^1 =\{ x\in \calX | p_t(x) > 0 \text{ and } p_t(y=1|x) > p_t^\prime(y=1|x)\}$, assume $f(x) = 0$ and on   a subset $\calX_M^2 =\{ x\in \calX | p_t(x) > 0 \text{ and } p_t(y=1|x) < p_t^\prime(y=1|x)\}$, assume $f(x) =1$. We will show that the error of $f$ on distribution with $p_t(y|x)$  is strictly greater than the error of $f$ on distribution with $p_t^\prime(y|x)$. Formally, 
    \begin{align*}
        &\error_{\calD^T}(f) - \error_{\calD^{T^\prime}}(f) \\ 
        &=\Expt{x\sim p_t(x), y\sim p_t(y|x)}{\indict{ \argmax_{j\in \calY} f_j(x) \ne y}} - \Expt{x\sim p_t(x), y\sim p_t^\prime(y|x)}{\indict{ \argmax_{j\in \calY} f_j(x) \ne y}} \\ &= \int_{x\in \calX_M} \indict{f(x) \ne 0} \left( p_t(y=0|x) - p_t^\prime(y=0|x) \right) p_t(x) dx \\
        &\qquad + \int_{x\in \calX_M} \indict{f(x) \ne 1} \left( p_t(y=1|x) - p_t^\prime(y=1|x) \right) p_t(x) dx \\
        &=  \int_{x\in \calX_M^2} \left( p_t(y=0|x) - p_t^\prime(y=0|x) \right) p_t(x) dx + \int_{x\in \calX_M^1} \left( p_t(y=1|x) - p_t^\prime(y=1|x) \right) p_t(x) dx \\
        &>0\,, \numberthis \label{eq:proof_prop1}
    \end{align*}
    where the last step follows by construction of the set $\calX_M^1$ and $\calX_M^2$. Since $\error_{\calD^T}(f) \ne \error_{\calD^{T^\prime}}(f)$, 
    given the information of $p_t(x)$ and $p_s(x,y)$ it is impossible to distinguish
    the two values of the error with classifier $f$. Thus, we obtain a contradiction on the assumption that $p_t(y|x) \ne p_t^\prime(y|x)$. Hence, we must pose restrictions on the nature of shift such that $\calP$ is singleton to to identify accuracy on the target. 
\end{proof}

\begin{proof}[Proof of \corollaryref{corollary:impossible}]
    The corollary follows directly from \propref{prop:characterization}. 
Since two different target conditional distribution can lead to different error estimates without assumptions on the classifier, no method can estimate two different quantities from the same given information. We illustrate this in Example 1 next.   
\end{proof}

\section{Estimating accuracy in covariate shift or label shift} \label{app:estimate_label_covariate}

\textbf{Accuracy estimation under covariate shift assumption {} {}} 
Under the assumption that $p_t(y|x) = p_s(y|x)$, accuracy on the target domain can be estimated as follows: 
\begin{align}
    \error_{\calD^\test} (f) &= \Expt{(x,y) \sim \calD^\train}{ \frac{p_t(x,y)}{p_s(x,y)} \indict{f(x) \ne y}} \\
    &= \Expt{(x,y) \sim \calD^\train}{ \frac{p_t(x)}{p_s(x)} \indict{f(x) \ne y}} \,.\label{eq:covariate_error}
\end{align}

Given access to $p_t(x)$ and $p_s(x)$, one can directly estimate the expression in \eqref{eq:covariate_error}. 

\textbf{Accuracy estimation under label shift assumption {} {} } 
Under the assumption that $p_t(x|y) = p_s(x|y)$, accuracy on the target domain can be estimated as follows: 
\begin{align}
    \error_{\calD^\test} (f) &= \Expt{(x,y) \sim \calD^\train}{ \frac{p_t(x,y)}{p_s(x,y)} \indict{f(x) \ne y}} \\
    &= \Expt{(x,y) \sim \calD^\train}{ \frac{p_t(y)}{p_s(y)} \indict{f(x) \ne y}} \,.\label{eq:label_error}
\end{align}

Estimating importance ratios $p_t(x)/p_s(x)$ is straightforward under covariate shift assumption when the distributions $p_t(x)$ and $p_s(x)$ are known. 
For label shift, one can leverage moment matching approach called BBSE~\citep{lipton2018detecting} or likelihood minimization approach MLLS~\citep{garg2020unified}. Below we discuss the objective of MLLS:  
\begin{align}
    w = \argmax_{w\in \calW} \Expt{x \sim p_t(x)}{ \log p_s(y|x) ^T w} \label{eq:mlls} \,,
\end{align}
where ${\calW = \{ w \; |\; \forall y \,,  w_y \ge 0 \text{ and } \sum_{y=1}^{k} w_y p_s(y) = 1 \}}$. 
MLLS objective is guaranteed to obtain consistent estimates for the importance ratios $w^*(y) = p_t(y)/ p_s(y)$ under the following condition. 
\begin{theorem}[Theorem 1~\citep{garg2020unified}]
If the distributions $\{p(x)|y)\,:\,y=1,\ldots,k\}$ are strictly linearly independent, then $\w^*$ is the unique maximizer of the MLLS objective~\eqref{eq:mlls}.
\label{theorem:mlls}
\end{theorem}
We refer interested reader to \citet{garg2020unified} for details. 

Above results of accuracy estimation under label shift and covariate shift 
can be extended to a generalized label shift and covariate shift settings. 
Assume a function $h: \calX \to \calZ$ such that $y$ is independent of $x$ given $h(x)$. In other words $h(x)$ contains all the information needed to predict label $y$. With help of $h$, we can extend estimation to following settings: (i) \emph{Generalized covariate shift}, i.e., $p_s(y | h(x)) = p_t(y|h(x))\,$ and $p_s(h(x)) > 0$ for all $x \in \calX_t$; (ii) \emph{Generalized label shift}, i.e., $p_s(h(x) | y) = p_t(h(x)|y)\,$ and $p_s(y) > 0$ for all $y \in \calY_t$. By simply replacing 
$x$ with $h(x)$ in \eqref{eq:covariate_error} and \eqref{eq:mlls}, we will obtain consistent error estimates under these generalized conditions. 

\begin{proof}[Proof of Example 1]
Under covariate shift using \eqref{eq:covariate_error}, we get 
\begin{align*}
    \calE_1&=\Expt{(x,y)\sim p_s(x,y)}{\frac{p_t(x)}{p_s(x)} \indict{f(x)\ne y}}\\ 
    &=  \Expt{x\sim p_s(x,y = 0)}{\frac{p_t(x)}{p_s(x)} \indict{f(x)\ne 0}} + \Expt{x\sim p_s(x,y=1)}{\frac{p_t(x)}{p_s(x)} \indict{f(x)\ne 1}} \\
    &= \int \indict{f(x)\ne 0} p_t(x) p_s(y=0|x) dx + \int \indict{f(x)\ne 1} p_t(x) p_s(y=1|x) dx
\end{align*}     
 Under label shift using \eqref{eq:label_error}, we get 
\begin{align*}
    \calE_2&= \Expt{(x,y) \sim \calD^\train}{ \frac{p_t(y)}{p_s(y)} \indict{f(x) \ne y}} \\ 
    &= \Expt{x\sim p_s(x,y=0)}{\frac{\beta}{\alpha} \indict{f(x)\ne 0}} + \Expt{x\sim p_s(x,y=1)}{\frac{1- \beta}{1-\alpha} \indict{f(x)\ne 1}} \\ 
    &= \int \indict{f(x)\ne 0} \frac{\beta}{\alpha} p_s(y=0|x) p_s(x) dx + \int \indict{f(x)\ne 1} \frac{(1-\beta)}{(1-\alpha)} p_s(y=1|x) p_s(x) dx
\end{align*}

Then $\calE_1 - \calE_2$ is given by 
\begin{align*}
    \calE_1 - \calE_2 &= \int \indict{f(x)\ne 0} p_s(y=0|x) \left[p_t(x) - \frac{\beta}{\alpha} p_s(x)\right] dx \\ &+ \int \indict{f(x)\ne 1} p_s(y=1|x) \left[ p_t(x) - \frac{(1-\beta)}{(1-\alpha)} p_s(x) \right] dx\\
    &=  \int \indict{f(x)\ne 0} p_s(y=0|x) \frac{(\alpha - \beta )}{\alpha} \phi(\mu_2) dx\\ &+ \int \indict{f(x)\ne 1} p_s(y=1|x) \frac{(\alpha - \beta)}{1-\alpha} \phi(\mu_1) dx \,. \numberthis \label{eq:error_diff}
\end{align*}

If $\alpha > \beta$, then $\calE_1 > \calE_2$ and if $\alpha < \beta$, then $\calE_1 < \calE_2$.  Since $\calE_1 \ne \calE_2$ 
for arbitrary $f$, given access to $p_s(x,y)$, and $p_t(x)$, 
any method that consistently estimates error under 
covariate shift will give an 
incorrect estimate under label shift and vice-versa.   
The reason being that the same $p_t(x)$ and $p_s(x,y)$
can correspond to error $\calE_1$ (under covariate shift)
or error $\calE_2$ (under label shift) either of 
which is not discernable absent further assumptions 
on the nature of shift. 
\end{proof}

\section{Alternate interpretation of ATC } \label{app:interpretation}

\update{Consider the following framework: 
Given a datum $(x,y)$,
define a binary classification
problem of whether the 
model prediction $\argmax f(x)$
was correct or incorrect. 
In particular, if the model 
prediction matches the true label, 
then we assign a label 1 (positive) 
and conversely, if the model prediction 
doesn't match the true label then we assign 
a label 0 (negative).}

\update{Our method can be interpreted as 
identifying examples for correct and 
incorrect prediction based on the value of the 
score function $s(f(x))$, i.e., 
if the score $s(f(x))$ is greater 
than or equal to the threshold $t$ then our method
predicts that the classifier correctly 
predicted datum $(x,y)$ and vice-versa if 
the score is less than $t$.
A method that can solve this task
will perfectly estimate the target performance.
However, such an expectation is unrealistic. 
Instead, ATC expects that \emph{most} of the 
examples with score above threshold are correct 
and most of the examples below the threshold are 
incorrect. 
More importantly, ATC selects a threshold  
such that the number of 
falsely identified correct 
predictions match falsely 
identified incorrect predictions 
on source distribution, thereby 
balancing incorrect predictions. 
We expect useful estimates of accuracy 
with ATC if the threshold transfers to target,
i.e.   
if the number of 
falsely identified correct 
predictions match falsely 
identified incorrect predictions on target. }
\update{This interpretation relates our method 
to the OOD detection literature where
\citet{hendrycks2016baseline,hendrycks2019scaling} 
highlight that classifiers 
tend to assign higher confidence to 
in-distribution examples and leverage 
maximum softmax confidence (or logit) 
to perform OOD detection. }

\section{Details on the Toy Model} \label{app:toy_model}

\begin{figure}[t]
    \centering
    \subfigure[]{\includegraphics[width=0.4\linewidth]{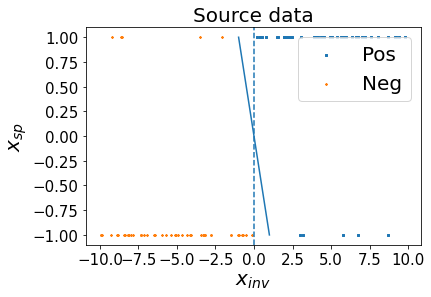}} \hfil
    \subfigure[]{\includegraphics[width=0.4\linewidth]{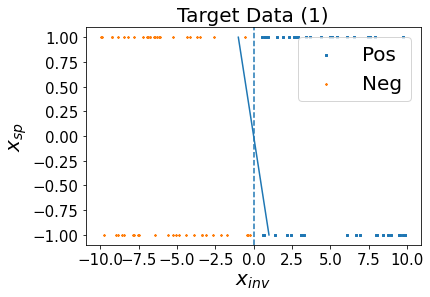}} \\
    \subfigure[]{\includegraphics[width=0.4\linewidth]{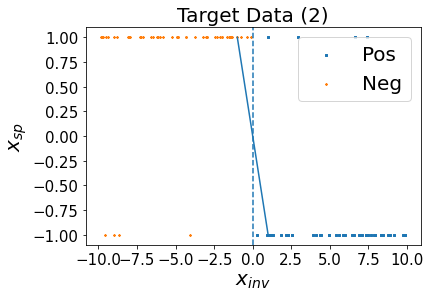}} \hfil
    \subfigure[]{\includegraphics[width=0.4\linewidth]{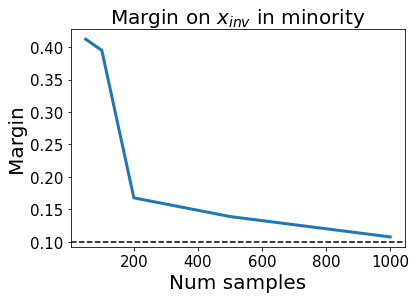}} 
    \caption[]{Illustration of toy model. (a) Source data at $n=100$. (b) Target data with $p_s^\prime = 0.5$. (b) Target data with $p_s^\prime = 0.9$. (c) Margin of $x_\inv$ in the minority group in source data. As sample size increases the margin saturates to true margin $\gamma = 0.1$.   } 
    \label{fig:toy_model}
  \end{figure}

\textbf{Skews observed in this toy model} In \figref{fig:toy_model}, we illustrate the toy model used in our empirical experiment. 
In the same setup,  we empirically observe that the margin on population with less density is large, i.e., margin is much greater than $\gamma$ when the number of observed samples is small (in \figref{fig:toy_model} (d)). Building on this observation, \citet{nagarajan2020understanding} showed in cases when margin decreases with number of samples, a max margin classifier trained on finite samples is bound to depend on the spurious features in such cases. They referred to this skew as \emph{geometric skew}. 

Moreover, even when the number of samples are large so that we do not observe geometric skews, \citet{nagarajan2020understanding} showed that training for finite number of epochs, a linear classifier will have a non zero dependency on the spurious feature. They referred to this skew as \emph{statistical skew}. Due both of these skews, we observe that a linear classifier obtained with training for finite steps on training data with finite samples, will have a non-zero dependency on the spurious feature. We refer interested reader to \citet{nagarajan2020understanding} for more details.  

\textbf{Proof of \thmref{thm:toy_theory} {} {}} 
\update{Recall, we consider a easy-to-learn binary classification problem with 
two features $x = [x_{\inv}, x_{\spr}] \in \Real^2$ where 
$x_\inv$ is fully predictive 
invariant feature with a margin $\gamma >0$
and $x_\spr \in \{-1, 1\}$ is a spurious feature (i.e., 
a feature that is correlated but
not predictive of the true label). 
Conditional on $y$, the distribution
over $x_\inv$ is 
given as follows: 
\begin{equation}
    x_\inv | y \sim \begin{cases}
        U[\gamma, c] & y = 1\\
        U[-c, -\gamma] & y = -1 
    \end{cases} \,, 
\end{equation}
where $c$ is a fixed constant greater than $\gamma$. 
For simplicity, we assume that label distribution 
on source is uniform on $\{-1, 1\}$.   
$x_\spr$ is distributed such that 
$P_s[ x_\spr \cdot (2y-1) > 0] = p_\spr $, where 
$p_\spr \in (0.5, 1.0)$ controls 
the degree of spurious correlation. 
To model distribution shift, 
we simulate target data with different
degree of spurious correlation, i.e., in target distribution 
$P_t[ x_\spr \cdot (2y-1) > 0] = p_\spr^\prime \in [0,1]$.  
Note that here we do not consider shifts in the label 
distribution but our result extends to arbitrary shifts 
in the label distribution as well. }

\update{In this setup, we examine linear sigmoid classifiers of the form  
$f(x) = \left[ \frac{1}{1 + e^{w^T x}}, \frac{ e^{w^T x}}{1 + e^{w^T x}}\right]$
where $w = [w_{\inv}, w_\spr] \in \Real^2$. 
We show that given a linear classifier 
that relies on the spurious feature 
and achieves a non-trivial performance on 
the source (i.e., $w_\inv > 0$),  
ATC with maximum confidence score function  
\emph{consistently}
estimates the accuracy on the target distribution.
Define $X_M = \{x | x_\spr \cdot (2y -1) < 0\}$ 
and $X_C = \{x | x_\spr \cdot (2y -1) > 0\}$. 
Notice that in target distributions, we are 
changing the fraction of examples in $X_M$ 
and $X_C$ but we are not changing the 
distribution of examples within individual set.} 
\update{\begin{theorem} \label{thm:toy_theory_repeat}
    Given any classifier $f$ with $w_\inv > 0$ in the above setting,   
    assume that the threshold $t$ is obtained with finite sample 
    approximation of \eqref{eq:ATC_thres}, i.e., $t$ is 
    selected such that\footnote{Note that this is possible 
    because a linear classifier with sigmoid activation 
    assigns a unique score to each point in source distribution.}  
    \begin{align}
        \sum_{i=1}^n \left[ \indict{ \max_{j\in \out} f_j(x_i) < t} \right] = \sum_{i=1}^n \left[\indict{ \argmax_{j\in \out} f_j(x_i) \ne y_i} \right] \,, \label{eq:ATC_finite}
    \end{align}
    where $\{(x_i, y_i)\}_{i=1}^n \sim (\calD^\train)^n$ are $n$ samples from source distribution.  
    Fix a $\delta > 0$. Assuming $n\ge {{2\log(4/\delta)}/{(1 - p_\spr)^2}}$,
    then the estimate of accuracy by  
    ATC as in \eqref{eq:ATC_pred} satisfies the following with probability at least $1-\delta$,  
    \begin{align}
        \abs{ \Expt{x \sim \calD^\test}{\indict{ s(f(x)) < t}} - \Expt{(x,y) \sim \calD^\test}{\indict{ \argmax_{j\in \out} f_j(x) \ne y}} } \le \sqrt{\frac{\log(8/\delta)}{n\cdot c_\spr}} \,, \label{eq:thm3_final}
    \end{align}
    where $\calD^\test$ is any target distribution considered in our setting and $c_\spr = (1 - p_\spr)$ if $w_\spr > 0$ and $c_\spr = p_\spr$ otherwise.   
    \end{theorem}}
    \update{\begin{proof}
        First we consider the case of $w_{\spr} >0$. The proof follows in two simple steps. First we notice that the classifier will make an error only on some points in $X_M$ and the threshold $t$ will be selected such that the fraction of points in $X_M$ with maximum confidence less than the threshold $t$ will match the error of the classifier on $X_M$.  Classifier with $w_{\spr} >0$ and $w_{\inv} >0$ will classify all the points in $X_C$ correctly.
        Second, since the distribution of points is not changing within $X_M$ and $X_C$, the same threshold continues to work for arbitrary shift in the fraction of examples in $X_M$, i.e., $p^\prime_\spr$. 
        \\[10pt]
        Note that when $w_\spr >0$, the classifier makes no error on points in $X_C$ and makes an error on a subset $X_\err = \{x | x_\spr \cdot (2y -1) < 0 \, \& \, (w_\inv x_\inv + w_\spr x_\spr) \cdot (2y -1) \le 0 \}$ of $X_M$, i.e., $X_\err \subseteq X_M$. Consider $X_\thres = \{ x| \argmax_{y\in \calY} f_y(x) \le t \}$ as the set of points that obtain a score less than or equal to $t$. 
        Now we will show that ATC chooses a threshold $t$ such that all points in $X_C$ gets a score above $t$, i.e.,  $X_\thres \subseteq X_M$. First note that the score of points close to the true separator in $X_C$, i.e., at $x_1 = (\gamma, 1)$ and $x_2 = (-\gamma, -1)$ match. In other words, score at $x_1$ matches with the score of $x_2$ by symmetricity, i.e., 
        \begin{align}
            \argmax_{y \in \calY} f_y(x_1) = \argmax_{y \in \calY} f_y(x_2) = \frac{e^{w_\inv \gamma + w_\spr}}{(1 + e^{w_\inv \gamma + w_\spr})} \,.             
        \end{align}
        Hence, if $t \ge \argmax_{y \in \calY} f_y(x_1)$ then we will have $\abs{X_\err} < \abs{X_\thres}$  which is contradiction violating definition of $t$ as in \eqref{eq:ATC_finite}. Thus $X_\thres \subseteq X_M$. 
        \\[10pt]
        Now we will relate LHS and RHS of \eqref{eq:ATC_finite} with their expectations using Hoeffdings and DKW inequality to conclude \eqref{eq:thm3_final}. Using Hoeffdings' bound, we have with probability at least $1-\delta/4$
        \begin{align}
        \abs{ \sum_{i\in X_M} \frac{\left[\indict{ \argmax_{j\in \out} f_j(x_i) \ne y_i} \right]}{\abs{X_M}} - \Exp_{(x,y)\sim \calD^\test} \left[\indict{ \argmax_{j\in \out} f_j(x) \ne y} \right]} \le \sqrt{\frac{\log(8/\delta)}{2\abs{X_M}}} \,. \label{eq:step1_th}
        \end{align}
        With DKW inequality, we have with probability at least $1-\delta/4$ 
        \begin{align}
        \abs{ \sum_{i\in X_M} \frac{\left[ \indict{ \max_{j\in \out} f_j(x_i) < t^\prime} \right]}{\abs{X_M}} - \Exp_{(x,y)\sim \calD^\test} \left[ \indict{ \max_{j\in \out} f_j(x) < t^\prime} \right]} \le \sqrt{\frac{\log(8/\delta)}{2\abs{X_M}}} \,,\label{eq:step2_th}
        \end{align}
        for all $t^\prime > 0 $. Combining \eqref{eq:step1_th} and \eqref{eq:step2_th} at $t^\prime = t$ with definition \eqref{eq:ATC_finite}, we have with probability at least $1- \delta/2$
        \begin{align}
            \abs{ \Expt{x \sim \calD^\test}{\indict{ s(f(x)) < t}} - \Expt{(x,y) \sim \calD^\test}{\indict{ \argmax_{j\in \out} f_j(x) \ne y}} } \le \sqrt{\frac{\log(8/\delta)}{2\abs{X_M}}} \,. \label{eq:bound_Xm}
        \end{align} 
        \\
        Now for the case of $w_\spr < 0$, we can use the same arguments on $X_C$. That is, since now all the error will be on points in $X_C$ and classifier will make no error $X_M$, we can show that threshold $t$ will be selected such that the fraction of points in $X_C$ with maximum confidence less than the threshold $t$ will match the error of the classifier on $X_C$. Again, since the distribution of points is not changing within $X_M$ and $X_C$, the same threshold continues to work for arbitrary shift in the fraction of examples in $X_M$, i.e., $p^\prime_\spr$.  Thus with similar arguments, we have 
        \begin{align}
            \abs{ \Expt{x \sim \calD^\test}{\indict{ s(f(x)) < t}} - \Expt{(x,y) \sim \calD^\test}{\indict{ \argmax_{j\in \out} f_j(x) \ne y}} } \le \sqrt{\frac{\log(8/\delta)}{2\abs{X_C}}} \,. \label{eq:bound_Xc}
        \end{align} 
        \\[10pt]    
        Using Hoeffdings' bound, with probability at least $1-\delta/2$, we have 
        \begin{align}
            \abs{ {X_M} - n\cdot (1 -p_\spr)} \le \sqrt{\frac{n\cdot log(4/\delta)}{2}} \,. \label{eq:bound_XM}           
        \end{align}
        With probability at least $1-\delta/2$, we have 
        \begin{align}
            \abs{ {X_C} - n\cdot p_\spr} \le \sqrt{\frac{n\cdot log(4/\delta)}{2}} \,. \label{eq:bound_XC}           
        \end{align}
        Combining \eqref{eq:bound_XM} and \eqref{eq:bound_Xm}, we get the desired result for $w_\spr >0$. For $w_\spr < 0$, we combine \eqref{eq:bound_XC} and \eqref{eq:bound_Xc} to get the desired result. 
    \end{proof}}

\textbf{Issues with IM in toy setting {} {}} As described in \appref{app:baselines}, we observe that IM is sensitive to binning strategy. 
In the main paper, we include IM result with uniform mass binning with $100$ bins. Empirically, we observe that we recover the true performance with IM if we use equal width binning with number of bins greater than 5.

\begin{figure}[t]
    \centering
    \subfigure[]{\includegraphics[width=0.4\linewidth]{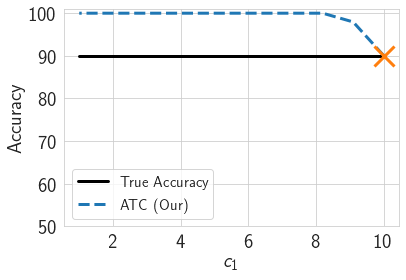}} \hfil
    \caption[]{Failure of ATC in our toy model. Shifting the support of target class 
    conditional $p_t(x_{\inv}|y)$ may 
    introduce a bias in ATC estimates, e.g.,  
    shrinking the support to 
    $c_1$($<c$) (while maintaining 
    uniform distribution) 
    in the target leads to overestimation bias.} 
    \label{fig:toy_model_failure}
  \end{figure}  

\textbf{Biased estimation with ATC in our toy model {} {}}
We assumed that both in 
source and target $x_{\inv}|y=1$ 
is uniform between $[\gamma, c]$ 
and $x|y=-1$ is uniform 
between $[-c, -\gamma]$. 
Shifting the support of target class 
conditional $p_t(x_{\inv}|y)$ may 
introduce a bias in ATC estimates, e.g.,  
shrinking the support to 
$c_1$($<c$) (while maintaining 
uniform distribution) 
in the target will lead to an over-estimation 
of the target performance with ATC.  
We show this failure in \figref{fig:toy_model_failure}. 
The reason being that with 
the same threshold that we 
see more examples falsely identified as correct 
as compared to  examples falsely identified as 
incorrect.

\subsection{A More General Result} \label{app:general_result}

Recall, for a given threshold $t$, we categorize an example $(x,y)$ as a falsely identified correct prediction (ficp) if the predicted label $\wh y = \argmax f(x)$ is not the same as $y$ but the predicted score $f_{\wh y}(x)$ is greater than $t$. Similarly, an example is falsely identified incorrect prediction (fiip) if the predicted label $\wh y$ is the same as $y$ but the predicted score $f_{\wh y} (x)$ is less than $t$. 

In general, we believe that our method will obtain consistent 
estimates in scenarios 
where the relative distribution of covariates doesn't change 
among examples that are falsely identified as 
incorrect and examples that are falsely identified as correct. 
In other words, ATC is expected to work if the distribution shift is such that 
falsely identified incorrect predictions match falsely identified correct prediction.

\subsection{ATC produces consistent estimate on source distribution}

\begin{proposition}
Given labeled validation data $\{(x_i,y_i)\}_{i=1}^n$ from a distribution $\calD^S$ and a model $f$, choose a threshold $t$ as in \eqref{eq:ATC_thres}. Then for $\delta > 0$, with probability at least $1-\delta$, we have 
\begin{align*}
    \Exp_{(x,y)\sim \calD} \left[ \indict{ \max_{j\in \out} f_j(x) < t} - \indict{ \argmax_{j\in \out} f_j(x) \ne y} \right] \le 2 \sqrt{\frac{\log(4/\delta)}{2n}} \numberthis
\end{align*}
\end{proposition}
\begin{proof}
The proof uses (i) Hoeffdings' inequality to relate the accuracy with expected accuracy; and (ii) DKW inequality to show the concentration of the estimated accuracy with our proposed method. Finally, we combine (i) and (ii) using the fact that at selected threshold $t$ the number of false positives is equal to the number of false negatives. 

Using Hoeffdings' bound, we have with probability at least $1-\delta/2$
\begin{align}
    \abs{ \sum_{i=1}^n \left[\indict{ \argmax_{j\in \out} f_j(x_i) \ne y_i} \right] - \Exp_{(x,y)\sim \calD} \left[\indict{ \argmax_{j\in \out} f_j(x) \ne y} \right]} \le \sqrt{\frac{\log(4/\delta)}{2n}} \,. \label{eq:step1}
\end{align}
With DKW inequality, we have with probability at least $1-\delta/2$ 
\begin{align}
    \abs{ \sum_{i=1}^n \left[ \indict{ \max_{j\in \out} f_j(x_i) < t^\prime} \right] - \Exp_{(x,y)\sim \calD} \left[ \indict{ \max_{j\in \out} f_j(x) < t^\prime} \right]} \le \sqrt{\frac{\log(4/\delta)}{2n}} \,,\label{eq:step2}
\end{align}
for all $t^\prime > 0 $. Finally by definition, we have 
\begin{align}
    \sum_{i=1}^n \left[ \indict{ \max_{j\in \out} f_j(x_i) < t^\prime} \right] = \sum_{i=1}^n \left[\indict{ \argmax_{j\in \out} f_j(x_i) \ne y_i} \right] \label{eq:step3}
\end{align}

Combining \eqref{eq:step1}, \eqref{eq:step2} at $t^\prime = t$, and \eqref{eq:step3}, we have the desired result.  
\end{proof}

\section{Basline Methods} \label{app:baselines}

\textbf{Importance-re-weighting (IM) {} {}} If we can estimate the importance-ratios
$\frac{p_t(x)}{p_s(x)}$ with just the unlabeled data from the target and 
validation labeled data from source, then we can estimate the accuracy as on target as follows: 
\begin{align}
\error_{\calD^\test} (f) = \Expt{(x,y) \sim \calD^\train}{ \frac{p_t(x)}{p_s(x)} \indict{f(x) \ne y}} \,.    
\end{align}

As previously discussed, this is particularly useful in the setting of covariate shift (within support) where importance ratios estimation has been explored in the literature in the past. Mandolin \citep{chen2021mandoline} extends this approach. They estimate importance-weights with use of extra supervision about the axis along which the distribution is shifting.  

\update{In our work, we experiment with uniform mass binning and equal width binning
with the number of bins in $[5, 10, 50]$. Overall, we observed that equal width binning works the best with $10$ bins. Hence throughout this paper we perform equal width binning with $10$ bins to include results with IM.}

\textbf{Average Confidence (AC) {} {}} If we expect the classifier to be argmax calibrated on the target then average confidence is equal to accuracy of the classifier. Formally, by definition of argmax calibration of $f$ on any distribution $\calD$, we have 
\begin{align}
\error_{\calD} (f) = \Expt{(x,y) \sim \calD}{  \indict{ y\not\in \argmax_{j\in\calY} f_j(x) }} = \Expt{(x,y) \sim \calD}{ \max_{j\in\calY} f_j(x)} \,. 
\end{align}

\textbf{Difference Of Confidence {}{}} We estimate the error on target by subtracting difference of confidences on source and target (as a distributional distance~\citep{guillory2021predicting}) from expected error on source distribution, i.e, $\DOC_{\calD^\test} = \Expt{x \sim \calD^\train}{ \indict{\argmax_{j\in\calY} f_j(x) \ne y}} + \Expt{x \sim \calD^\test}{ \max_{j\in\calY} f_j(x)} - \Expt{x \sim \calD^\train}{ \max_{j\in\calY} f_j(x)}$. This is referred to as DOC-Feat in \citep{guillory2021predicting}.

\textbf{Generalized Disagreement Equality (GDE) {}{}} \citet{jiang2021assessing} proposed average disagreement of two models (trained on the same training set but with different initialization and/or different data ordering) as a approximate measure of accuracy on the underlying data, i.e., 
\begin{align}
    \error_{\calD}(f) = \Expt{(x,y) \sim \calD}{\indict{f(x) \ne f^\prime(x)}} \,.
\end{align}

They show that marginal calibration of the model is sufficient to have expected test error equal to the expected of average disagreement of two models where the latter expectation is also taken over the models used to calculate disagreement.

\section{Details on the Dataset Setup} \label{app:dataset}

\begin{table}[h!]
    \begin{adjustbox}{width=\columnwidth,center}
    \centering
    \small
    \tabcolsep=0.12cm
    \renewcommand{\arraystretch}{1.2}
    \begin{tabular}{@{}*{3}{c}@{}}
    \toprule
    Train (Source) & Valid (Source) & Evaluation~(Target)  \\
    \midrule
    MNIST (train) & MNIST (valid) & USPS, SVHN and Q-MNIST \\ 
    CIFAR10 (train) & CIFAR10 (valid) & CIFAR10v2, 95 CIFAR10-C datasets (Fog and Motion blur, etc. ) \\
    CIFAR100 (train) & CIFAR100 (valid) & 95 CIFAR100-C datasets (Fog and Motion blur, etc. )  \\
    FMoW (2002-12) (train) & FMoW (2002-12) (valid) & \thead{FMoW \{(2013-15, 2016-17) $\times$ \\(All, Africa, Americas, Oceania, Asia,
and Europe)\} } \\
    RxRx1 (train) & RxRx1(id-val) & RxRx1 (id-test, OOD-val, OOD-test)  \\
    Amazon (train) & Amazon (id-val) & Amazon (OOD-val, OOD-test)  \\
    CivilComments (train) & CivilComments (id-val) & \thead{CiviComments (8 demographic identities male, female, LGBTQ, \\ Christian, Muslim, other religions, Black, and White)} \\ 
    ImageNet (train) & ImageNet (valid) &  \thead{3 ImageNetv2 datasets, ImageNet-Sketch,\\ 95 ImageNet-C datasets}  \\
    ImageNet-200 (train) & ImageNet-200 (valid) &  \thead{3 ImageNet-200v2 datasets,  ImageNet-R, \\ ImageNet200-Sketch, 95 ImageNet200-C datasets}  \\
    \textsc{Breeds} (train)& \breeds~(valid) & \thead{Same subpopulations as train but unseen images from natural \\ and synthetic shifts in ImageNet, Novel subpopulations on \\ natural and synthetic shifts} \\ 
    \bottomrule 
    \end{tabular}  
    \end{adjustbox}  
    \caption{ \update{Details of the test datasets considered in our evaluation.} 
    }\label{table:dataset}
 \end{table}

\update{In our empirical evaluation, we consider 
both natural and synthetic distribution shifts.
We consider shifts on ImageNet~\citep{russakovsky2015imagenet}, 
CIFAR~\citet{krizhevsky2009learning}, 
FMoW-\textsc{Wilds}~\citep{christie2018functional}, 
RxRx1-\textsc{Wilds}~\citep{taylor2019rxrx1}, 
Amazon-\textsc{Wilds}~\citep{ni2019justifying}, 
CivilComments-\textsc{Wilds}~\citep{borkan2019nuanced},
and MNIST~\cite{lecun1998mnist} datasets. }

\emph{ImageNet setup.} First, we consider 
synthetic shifts induced to simulate $19$ 
different visual corruptions (e.g., shot noise, 
motion blur, pixelation etc.) each with $5$ 
different intensities giving us a total 
of $95$ datasets under 
ImageNet-C~\citep{hendrycks2019benchmarking}. 
Next, we consider natural distribution
shifts due to differences in the data 
collection process. In particular, 
we consider $3$ ImageNetv2~\citep{recht2019imagenet}
datasets each using a different strategy 
to collect test sets. We also evaluate 
performance on images with artistic 
renditions of object classes, i.e., 
ImageNet-R~\citep{hendrycks2021many} 
and ImageNet-Sketch~\citep{wang2019learning} 
with hand drawn sketch images. Note that 
renditions dataset only contains $200$ 
classes from ImageNet. Hence, in the main 
paper we include results on 
ImageNet restricted to these $200$
classes, which we call as 
ImageNet-200, and relegate results 
on ImageNet with $1$k classes to appendix. 

We also consider \textsc{Breeds} benchmark~\citep{santurkar2020breeds} 
in our evaluation to assess robustness to 
subpopulation shifts, in particular, 
to understand how accuracy estimation 
methods behave when novel subpopulations 
not observed during training are introduced. 
\textsc{Breeds} leverages class hierarchy 
in ImageNet to repurpose original classes 
to be the subpopulations and defines a  
classification task on superclasses.  
Subpopulation shift is induced 
by directly making the subpopulations present in
the training and test distributions disjoint. 
Overall, \textsc{Breeds}  benchmark 
contains 4 datasets \textsc{Entity-13}, \textsc{Entity-30},
\textsc{Living-17}, \textsc{Non-living-26}, each focusing on 
different subtrees in the hierarchy. To generate \breeds~dataset on top of ImageNet, we 
use the open source library: \url{https://github.com/MadryLab/BREEDS-Benchmarks}. 
We focus on 
natural and synthetic shifts as in ImageNet on same and different
subpopulations in BREEDs. Thus for both the subpopulation (same or novel), we obtain a total of $99$ target datasets. 

\emph{CIFAR setup.} Similar to the ImageNet setup, 
we consider (i) synthetic shifts (CIFAR-10-C) due to 
common corruptions; and (ii) natural distribution shift 
(i.e., CIFARv2~\citep{recht2018cifar,torralba2008tinyimages})
due to differences in data collection strategy on 
on CIFAR-10~\citep{krizhevsky2009learning}.  
On CIFAR-100, we just have  
synthetic shifts due to common corruptions. 

\emph{FMoW-\textsc{Wilds} setup.} 
In order to consider distribution shifts 
faced in the wild, we consider 
FMoW-\textsc{wilds}~\citep{wilds2021,christie2018functional}
from \textsc{Wilds} benchmark,  
which contains satellite images taken 
in different geographical regions and at
different times. \update{We obtain $12$ different 
OOD target sets by considering images between
years $2013$--$2016$ and $2016$--$2018$ and by considering 
five geographical regions as subpopulations (Africa, Americas, Oceania, Asia,
and Europe) separately and together.}

\update{\emph{RxRx1--\textsc{Wilds} setup.} Similar to FMoW, we consider 
RxRx1-\textsc{Wilds}~\citep{taylor2019rxrx1} from \textsc{Wilds} benchmark, 
which contains image of cells obtained by fluorescent microscopy and the task is to genetic treatments the cells received. We obtain $3$ target datasets with shift induced by batch effects which make it difficult to draw conclusions from data across experimental
batches. 
}

\update{\emph{Amazon-\textsc{Wilds} setup.} For natural language task, we consider Amazon-\textsc{Wilds}~\citep{ni2019justifying} dataset from \textsc{Wilds} benchmark, 
which contains review text and the task is get a corresponding star rating from $1$ to $5$. We obtain 2 target datasets by considered shifts induced due to different set of reviewers than the training set. 
}

\update{\emph{CivilComments-\textsc{Wilds} setup.}  
We also consider CivilComments-\textsc{Wilds}~\citep{borkan2019nuanced} from \textsc{Wilds} benchmark,  which contains text comments and the task is to classify them for toxicity. We obtain $18$ target datasets depending on whether a comment mentions each of the 8 demographic identities male, female, LGBTQ, Christian, Muslim, other religions, Black, and White. 
}

\emph{MNIST setup.} 
For completeness, we also consider distribution shifts 
on MNIST~\citep{lecun1998mnist} digit classification
as in the prior work~\citep{deng2021labels}. 
We use three real shifted datasets, i.e., USPS~\citep{hull1994database},
SVHN~\citep{netzer2011reading} and QMNIST~\citep{qmnist-2019}. 

\section{Details on the Experimental Setup} \label{app:exp_setup}
\newcommand\tab[1][1cm]{\hspace*{#1}}
All experiments were run on NVIDIA Tesla V100 GPUs. We used PyTorch~\citep{NEURIPS2019a9015} for experiments. 

\textbf{Deep nets{} {}} We consider a 4-layered MLP. The PyTorch code for 4-layer MLP is as follows: 

\texttt{ nn.Sequential(nn.Flatten(), \\
\tab        nn.Linear(input\_dim, 5000, bias=True),\\
\tab        nn.ReLU(),\\
\tab        nn.Linear(5000, 5000, bias=True),\\
\tab        nn.ReLU(),\\
\tab        nn.Linear(5000, 50, bias=True),\\
\tab        nn.ReLU(),\\
\tab        nn.Linear(50, num\_label, bias=True)\\
\tab        )}

We mainly experiment convolutional nets. In particular, we use ResNet18 \citep{he2016deep}, ResNet50, and DenseNet121~\citep{huang2017densely} architectures
with their default implementation in PyTorch. Whenever we initial our models with pre-trained models, we again use default models in PyTorch. 

\textbf{Hyperparameters and Training details {} {}} As mentioned in the main text 
we do not alter the standard training procedures and hyperparameters for each task. 
We present results at final model, however, we observed that the same results extend to an early stopped model as well. For completeness, we include these details below: 

\emph{CIFAR10 and CIFAR100 {} {}} We train DenseNet121 and ResNet18 architectures from scratch. We use SGD training with momentum of $0.9$ for $300$ epochs. We start with learning rate $0.1$ and decay it by multiplying it with $0.1$ every $100$ epochs. We use a weight decay of $5\time 10^-4$. We use batch size of $200$. 
For CIFAR10, we also experiment with the same models pre-trained on ImageNet.

\emph{ImageNet {} {}} For training, we use Adam with a batch size of $64$ and learning rate $0.0001$. Due to huge size of ImageNet, we could only train two models needed for GDE for $10$ epochs. Hence, for relatively small scale experiments, we also perform experiments on ImageNet subset with $200$ classes, which we call as ImageNet-200 with the same training procedure. These $200$ classes are the same classes as in ImageNet-R dataset. This not only allows us to train ImageNet for $50$ epochs but also allows us to use ImageNet-R in our testbed. On the both the datasets, we observe a similar superioriy with ATC. Note that all the models trained here were initialized with 
a pre-trained ImageNet model with the last layer replaced with random weights. 

\emph{FMoW-\textsc{wilds} {} {}} 
For all experiments, we follow \citet{wilds2021} and use two architectures DenseNet121 and ResNet50, both pre-trained on ImageNet. We use the Adam optimizer~\citep{kingma2014adam} with an initial learning rate of $10^{-4}$ that decays by $0.96$ per epoch, and train for $50$ epochs and with a batch size of $64$. 

\update{\emph{RxRx1-\textsc{wilds} {} {}} For all experiments, we follow \citet{wilds2021} and use two architectures DenseNet121 and ResNet50, both pre-trained on ImageNet. We use Adam optimizer with a learning rate of $1e - 4$ and L2-regularization strength of $1e - 5$ with a batch size of 75 for 90 epochs. We linearly increase the learning rate for 10 epochs, then decreasing it following a cosine learning rate schedule. Finally, we pick the model that obtains highest in-distribution validation accuracy. 
}

\update{\emph{Amazon-\textsc{wilds} {} {}} For all experiments, we follow \citet{wilds2021} and finetuned DistilBERT-base-uncased models~\citep{Sanh2019DistilBERTAD}, using
the implementation from \citet{wolf-etal-2020-transformers}, and with the following hyperparameter settings: batch size $8$; learning rate $1e - 5$ with the AdamW optimizer~\citep{loshchilov2017decoupled}; L2-regularization
strength $0.01$; $3$ epochs with early stopping; and a maximum number of tokens of $512$. 
}

\update{\emph{CivilComments-\textsc{wilds} {} {}}  For all experiments, we follow \citet{wilds2021} and fine-tuned DistilBERT-base-uncased models~\citep{Sanh2019DistilBERTAD}, using 
the implementation from \citet{wolf-etal-2020-transformers} and with the following hyperparameter settings: batch size $16$; learning rate $1e - 5$ with the AdamW optimizer~\citep{loshchilov2017decoupled} for 5 epochs; L2-regularization strength $0.01$; and a maximum number of tokens of $300$. }

\emph{Living17 and Nonliving26 from \breeds~{} {}} For training,
we use SGD with a batch size of $128$, weight decay of $10^{-4}$, and learning rate $0.1$. Models were trained until convergence. Models were trained for a total of $450$ epochs, with 10-fold learning rate drops every $150$ epochs. Note that since we want to evaluate models for novel subpopulations no pre-training was used. We train two architectures DenseNet121 and ResNet50. 

\emph{Entity13 and Entity30 from \breeds~{} {}} For training,
we use  SGD with  a batch size of $128$, weight decay of $10^{-4}$, and learning rate $0.1$. Models were trained until convergence. Models were trained for a total of $300$ epochs, with 10-fold learning rate drops every $100$ epochs. Note that since we want to evaluate models for novel subpopulations no pre-training was used. We train two architectures DenseNet121 and ResNet50. 

\emph{MNIST {}{}} For MNIST, we train a MLP described above with SGD with momentum $0.9$ and learning rate $0.01$ for $50$ epochs. We use weight decay of $10^{-5}$ and batch size as $200$.

We have a single number for CivilComments because it is a binary classification task. For multiclass problems, ATC-NE and ATC-MC can lead to different ordering of examples when ranked with the corresponding scoring function. Temperature scaling on top can further alter the ordering of examples. The changed ordering of examples yields different thresholds and different accuracy estimates. However for binary classification, the two scoring functions are the same as entropy (i.e. $p\log(p) + (1-p) \log(p)$) has a one-to-one mapping to the max conf for $p\in [0,1]$. Moreover, temperature scaling also doesn't change the order of points for binary classification problems. Hence for the binary classification problems, both the scoring functions with and without temperature scaling yield the same estimates. We have made this clear in the updated draft. 

\textbf{Implementation for Temperature Scaling {} {}} We use temperature scaling implementation from \url{https://github.com/kundajelab/abstention}. We use validation set (the same we use to obtain ATC threshold or DOC source error estimate) to tune a single temperature parameter. 

\update{\subsection{Details on \figref{fig:intro}~(right) setup}} \label{app:fig1_details}

\update{For vision datasets, we train a DenseNet model with the exception of FCN model for MNIST dataset. For language datasets, we fine-tune a DistilBERT-base-uncased model. For each of these models, we use the exact same setup as described
\secref{app:exp_setup}. Importantly, to obtain errors on the same scale, we rescale all the errors by subtracting the error of Average Confidence method for each model. Results are reported as mean of the re-scaled errors over $4$ seeds.}

\newpage
\section{Supplementary Results} \label{app:results}

\begin{table}[t]
    \begin{adjustbox}{width=\columnwidth,center}
    \centering
    \tabcolsep=0.12cm
    \renewcommand{\arraystretch}{1.2}
    \begin{tabular}{@{}*{13}{c}@{}}
    \toprule
    \multirow{2}{*}{Dataset} & \multirow{2}{*}{Shift} & \multicolumn{2}{c}{IM} & \multicolumn{2}{c}{AC} & \multicolumn{2}{c}{DOC} & GDE & \multicolumn{2}{c}{ATC-MC (Ours)} & \multicolumn{2}{c}{ATC-NE (Ours)} \\
    & & Pre T & Post T & Pre T & Post T  & Pre T & Post T & Post T & Pre T & Post T & Pre T & Post T \\
    \midrule
    \multirow{4}{*}{\parbox{1.2cm}{CIFAR10} }  & \multirow{2}{*}{Natural} & $6.60$ &$5.74$ &$9.88$ &$6.89$ &$7.25$ &$6.07$ &$4.77$ &$3.21$ &$3.02$ &$2.99$ & $\bf 2.85$   \\
    & & $(0.35)$ & $(0.30)$ & $(0.16)$ & $(0.13)$ & $(0.15)$ & $(0.16)$ & $(0.13)$ & $(0.49)$ & $(0.40)$ & $(0.37)$ & $(0.29)$  \\
    & \multirow{2}{*}{Synthetic} & $12.33$ &$10.20$ &$16.50$ &$11.91$ &$13.87$ &$11.08$ &$6.55$ &$4.65$ &$4.25$ &$4.21$ &$\bf 3.87$ \\
    & & $(0.51)$ & $(0.48)$ & $(0.26)$ & $(0.17)$ & $(0.18)$ & $(0.17)$ & $(0.35)$ & $(0.55)$ & $(0.55)$ & $(0.55)$ & $(0.75)$ \\
    \midrule 
    \multirow{2}{*}{CIFAR100} & \multirow{2}{*}{Synthetic} & $13.69$ &$11.51$ &$23.61$ &$13.10$ &$14.60$ &$10.14$ &$9.85$ &$5.50$ &$\bf 4.75$ &$\bf 4.72$ &$4.94$ \\
    & & $(0.55)$ & $(0.41)$ & $(1.16)$ & $(0.80)$ & $(0.77)$ & $(0.64)$ & $(0.57)$ & $(0.70)$ & $(0.73)$ & $(0.74)$ & $(0.74)$ \\
    \midrule  
    \multirow{4}{*}{\parbox{1.8cm}{ImageNet200} }  & \multirow{2}{*}{Natural} & $12.37$ &$8.19$ &$22.07$ &$8.61$ &$15.17$ &$7.81$ &$5.13$ &$4.37$ &$2.04$ &$3.79$ & $\bf 1.45$ \\
    & & $(0.25)$ & $(0.33)$ & $(0.08)$ & $(0.25)$ & $(0.11)$ & $(0.29)$ & $(0.08)$ & $(0.39)$ & $(0.24)$ & $(0.30)$ & $(0.27)$ \\
    & \multirow{2}{*}{Synthetic} & $19.86$ &$12.94$ &$32.44$ &$13.35$ &$25.02$ &$12.38$ &$5.41$ &$5.93$ &$3.09$ &$5.00$ &$\bf 2.68$  \\ 
    & & $(1.38)$ & $(1.81)$ & $(1.00)$ & $(1.30)$ & $(1.10)$ & $(1.38)$ & $(0.89)$ & $(1.38)$ & $(0.87)$ & $(1.28)$ & $(0.45)$ \\
    \midrule 
    \multirow{4}{*}{\parbox{1.8cm}{\centering ImageNet} }  & \multirow{2}{*}{Natural} &
    $7.77$ &$6.50$ &$18.13$ &$6.02$ &$8.13$ &$5.76$ &$6.23$ &$3.88$ &$2.17$ &$2.06$ &$\bf 0.80$ \\
    & & $(0.27)$ & $(0.33)$ & $(0.23)$ & $(0.34)$ & $(0.27)$ & $(0.37)$ & $(0.41)$ & $(0.53)$ & $(0.62)$ & $(0.54)$ & $(0.44)$ \\
        & \multirow{2}{*}{Synthetic} &$13.39$ &$10.12$ &$24.62$ &$8.51$ &$13.55$ &$7.90$ &$6.32$ &$3.34$ &$\bf 2.53$ &$\bf 2.61$ &$4.89$ \\ 
        & & $(0.53)$ & $(0.63)$ & $(0.64)$ & $(0.71)$ & $(0.61)$ & $(0.72)$ & $(0.33)$ & $(0.53)$ & $(0.36)$ & $(0.33)$ & $(0.83)$ \\
    \midrule 
    \multirow{2}{*}{FMoW-\textsc{wilds}} & \multirow{2}{*}{Natural} & $5.53$ &$4.31$ &$33.53$ &$12.84$ &$5.94$ &$4.45$ &$5.74$ &$3.06$ &$\bf 2.70$ &$3.02$ &$\bf 2.72$  \\
    & & $(0.33)$ & $(0.63)$ & $(0.13)$ & $(12.06)$ & $(0.36)$ & $(0.77)$ & $(0.55)$ & $(0.36)$ & $(0.54)$ & $(0.35)$ & $(0.44)$ \\
    \midrule 
    \multirow{2}{*}{RxRx1-\textsc{wilds}} & \multirow{2}{*}{Natural} & $5.80$ &$5.72$ &$7.90$ &$4.84$ &$5.98$ &$5.98$ &$6.03$ &$4.66$ &$\bf 4.56$ &$\bf 4.41$ &$\bf 4.47$ \\
    & & $(0.17)$ & $(0.15)$ & $(0.24)$ & $(0.09)$ & $(0.15)$ & $(0.13)$ & $(0.08)$ & $(0.38)$ & $(0.38)$ & $(0.31)$ & $(0.26)$ \\
    \midrule 
    \multirow{2}{*}{Amazon-\textsc{wilds}} & \multirow{2}{*}{Natural} & $2.40$ &$2.29$ &$8.01$ &$2.38$ &$2.40$ &$2.28$ &$17.87$ &$1.65$ &$\bf 1.62$ &$ \bf 1.60$ &$\bf 1.59$ \\
    & & $(0.08)$ & $(0.09)$ & $(0.53)$ & $(0.17)$ & $(0.09)$ & $(0.09)$ & $(0.18)$ & $(0.06)$ & $(0.05)$ & $(0.14)$ & $(0.15)$ \\
    \midrule 
    \multirow{2}{*}{CivilCom.-\textsc{wilds}} & \multirow{2}{*}{Natural} & $12.64$ &$10.80$ &$16.76$ &$11.03$ &$13.31$ &$10.99$ &$16.65$ & \multicolumn{4}{c}{$\bf 7.14$} \\
    & & $(0.52)$ & $(0.48)$ & $(0.53)$ & $(0.49)$ & $(0.52)$ & $(0.49)$ & $(0.25)$ &  \multicolumn{4}{c}{$(0.41)$}  \\
    \midrule 
    \multirow{2}{*}{MNIST} & \multirow{2}{*}{Natural} &$18.48$ &$15.99$ &$21.17$ &$14.81$ &$20.19$ &$14.56$ &$24.42$ &$5.02$ &$\bf 2.40$ &$3.14$ &$3.50$ \\
    & & $(0.45)$ & $(1.53)$ & $(0.24)$ & $(3.89)$ & $(0.23)$ & $(3.47)$ & $(0.41)$ & $(0.44)$ & $(1.83)$ & $(0.49)$ & $(0.17)$ \\
    \midrule 
    \multirow{4}{*}{\parbox{1.8cm}{\textsc{Entity-13}} } & \multirow{2}{*}{Same} & $16.23$ &$11.14$ &$24.97$ &$10.88$ &$19.08$ &$10.47$ &$10.71$ &$5.39$ &$\bf 3.88$ &$4.58$ &$4.19$ \\
    & & $(0.77)$ & $(0.65)$ & $(0.70)$ & $(0.77)$ & $(0.65)$ & $(0.72)$ & $(0.74)$ & $(0.92)$ & $(0.61)$ & $(0.85)$ & $(0.16)$ \\
    & \multirow{2}{*}{Novel} & $28.53$ &$22.02$ &$38.33$ &$21.64$ &$32.43$ &$21.22$ &$20.61$ &$13.58$ &$10.28$ &$12.25$ &$\bf 6.63$ \\
    & & $(0.82)$ & $(0.68)$ & $(0.75)$ & $(0.86)$ & $(0.69)$ & $(0.80)$ & $(0.60)$ & $(1.15)$ & $(1.34)$ & $(1.21)$ & $(0.93)$ \\
    \midrule 
    \multirow{4}{*}{\parbox{1.8cm}{\textsc{Entity-30}}} & \multirow{2}{*}{Same} & $18.59$ &$14.46$ &$28.82$ &$14.30$ &$21.63$ &$13.46$ &$12.92$ &$9.12$ &$\bf 7.75$ &$8.15$ &$ \bf 7.64$ \\
    & & $(0.51)$ & $(0.52)$ & $(0.43)$ & $(0.71)$ & $(0.37)$ & $(0.59)$ & $(0.14)$ & $(0.62)$ & $(0.72)$ & $(0.68)$ & $(0.88)$ \\
    & \multirow{2}{*}{Novel} & $32.34$ &$26.85$ &$44.02$ &$26.27$ &$36.82$ &$25.42$ &$23.16$ &$17.75$ &$14.30$ &$15.60$ &$\bf 10.57$  \\
    & & $(0.60)$ & $(0.58)$ & $(0.56)$ & $(0.79)$ & $(0.47)$ & $(0.68)$ & $(0.12)$ & $(0.76)$ & $(0.85)$ & $(0.86)$ & $(0.86)$ \\
    \midrule
    \multirow{4}{*}{{\textsc{Nonliving-26}}} & \multirow{2}{*}{Same} & $18.66$ &$17.17$ &$26.39$ &$16.14$ &$19.86$ &$15.58$ &$16.63$ &$10.87$ &$\bf 10.24$ &$10.07$ &$\bf 10.26$ \\
    & & $(0.76)$ & $(0.74)$ & $(0.82)$ & $(0.81)$ & $(0.67)$ & $(0.76)$ & $(0.45)$ & $(0.98)$ & $(0.83)$ & $(0.92)$ & $(1.18)$ \\
    & \multirow{2}{*}{Novel} &$33.43$ &$31.53$ &$41.66$ &$29.87$ &$35.13$ &$29.31$ &$29.56$ &$21.70$ &$20.12$ &$19.08$ &$\bf 18.26$  \\ 
    & & $(0.67)$ & $(0.65)$ & $(0.67)$ & $(0.71)$ & $(0.54)$ & $(0.64)$ & $(0.21)$ & $(0.86)$ & $(0.75)$ & $(0.82)$ & $(1.12)$ \\
    \midrule 
    \multirow{4}{*}{\parbox{1.8cm}{\textsc{Living-17}}} & \multirow{2}{*}{Same} & $12.63$ &$11.05$ &$18.32$ &$10.46$ &$14.43$ &$10.14$ &$9.87$ &$4.57$ &$\bf  3.95$ &$\bf 3.81$ &$4.21$   \\
    & & $(1.25)$ & $(1.20)$ & $(1.01)$ & $(1.12)$ & $(1.11)$ & $(1.16)$ & $(0.61)$ & $(0.71)$ & $(0.48)$ & $(0.22)$ & $(0.53)$ \\
    & \multirow{2}{*}{Novel} & $29.03$ &$26.96$ &$35.67$ &$26.11$ &$31.73$ &$25.73$ &$23.53$ &$16.15$ &$14.49$ &$12.97$ &$\bf 11.39$ \\ 
    & & $(1.44)$ & $(1.38)$ & $(1.09)$ & $(1.27)$ & $(1.19)$ & $(1.35)$ & $(0.52)$ & $(1.36)$ & $(1.46)$ & $(1.52)$ & $(1.72)$ \\
    \bottomrule 
    \end{tabular}
  \end{adjustbox}  
    \vspace{-5pt}
    \caption{
     \update{\emph{Mean Absolute estimation Error (MAE) results for different datasets in our setup grouped by the nature of shift.} 
    `Same' refers to same subpopulation shifts and `Novel' refers novel subpopulation shifts. \update{We include details about the target sets considered in each shift in \tabref{table:dataset}.} 
    Post T denotes use of TS calibration on source.  
    For language datasets, we use DistilBERT-base-uncased, for vision dataset we report results with DenseNet model with the exception of MNIST where we use FCN. 
    Across all datasets, we observe that ATC achieves superior performance (lower MAE is better).  
    For GDE post T and pre T estimates match since TS doesn't alter the argmax prediction. Results reported by aggregating MAE numbers over $4$ different seeds. Values in parenthesis (i.e., $(\cdot)$) denote standard deviation values.}
    }\label{table:error_estimation_std}
\end{table}

\begin{table}[t]
    \begin{adjustbox}{width=\columnwidth,center}
    \centering
    \tabcolsep=0.12cm
    \renewcommand{\arraystretch}{1.2}
    \begin{tabular}{@{}*{13}{c}@{}}
    \toprule
    \multirow{2}{*}{Dataset} & \multirow{2}{*}{Shift} & \multicolumn{2}{c}{IM} & \multicolumn{2}{c}{AC} & \multicolumn{2}{c}{DOC} & GDE & \multicolumn{2}{c}{ATC-MC (Ours)} & \multicolumn{2}{c}{ATC-NE (Ours)} \\
    & & Pre T & Post T & Pre T & Post T  & Pre T & Post T & Post T & Pre T & Post T & Pre T & Post T \\
    \midrule
    \multirow{4}{*}{\parbox{1.2cm}{CIFAR10} }  & \multirow{2}{*}{Natural} & $7.14$ & $6.20$ & $10.25$ & $7.06$ & $7.68$ & $6.35$ & $5.74$ & $4.02$ & $3.85$ & $3.76$ & $\bf 3.38$    \\
    & & $(0.14)$ & $(0.11)$ & $(0.31)$ & $(0.33)$ & $(0.28)$ & $(0.27)$ & $(0.25)$ & $(0.38)$ & $(0.30)$ & $(0.33)$ & $(0.32)$ \\
    & \multirow{2}{*}{Synthetic} & $12.62$ & $10.75$ & $16.50$ & $11.91$ & $13.93$ & $11.20$ & $7.97$ & $5.66$ & $5.03$ & $4.87$ & $\bf 3.63$ \\
    & & $(0.76)$ & $(0.71)$ & $(0.28)$ & $(0.24)$ & $(0.29)$ & $(0.28)$ & $(0.13)$ & $(0.64)$ & $(0.71)$ & $(0.71)$ & $(0.62)$ \\
    \midrule 
    \multirow{2}{*}{CIFAR100} & \multirow{2}{*}{Synthetic} & $12.77$ & $12.34$ & $16.89$ & $12.73$ & $11.18$ & $9.63$ & $12.00$ & $5.61$ & $\bf 5.55$ & $5.65$ & $5.76$ \\
    & & $(0.43)$ & $(0.68)$ & $(0.20)$ & $(2.59)$ & $(0.35)$ & $(1.25)$ & $(0.48)$ & $(0.51)$ & $(0.55)$ & $(0.35)$ & $(0.27)$  \\
    \midrule  
    \multirow{4}{*}{\parbox{1.8cm}{ImageNet200} } & \multirow{2}{*}{Natural} & $12.63$ & $7.99$ & $23.08$ & $7.22$ & $15.40$ & $6.33$ & $5.00$ & $4.60$ & $1.80$ & $4.06$ & $\bf 1.38$ \\
    & & $(0.59)$ & $(0.47)$ & $(0.31)$ & $(0.22)$ & $(0.42)$ & $(0.24)$ & $(0.36)$ & $(0.63)$ & $(0.17)$ & $(0.69)$ & $(0.29)$  \\
    & \multirow{2}{*}{Synthetic} & $20.17$ & $11.74$ & $33.69$ & $9.51$ & $25.49$ & $8.61$ & $4.19$ & $5.37$ & $2.78$ & $4.53$ & $3.58$   \\ 
    & & $(0.74)$ & $(0.80)$ & $(0.73)$ & $(0.51)$ & $(0.66)$ & $(0.50)$ & $(0.14)$ & $(0.88)$ & $(0.23)$ & $(0.79)$ & $(0.33)$  \\
    \midrule 
    \multirow{4}{*}{\parbox{1.8cm}{\centering ImageNet} }  & \multirow{2}{*}{Natural}
   & $8.09$ & $6.42$ & $21.66$ & $5.91$ & $8.53$ & $5.21$ & $5.90$ & $3.93$ & $1.89$ & $2.45$ & $\bf 0.73$  \\
    & & $(0.25)$ & $(0.28)$ & $(0.38)$ & $(0.22)$ & $(0.26)$ & $(0.25)$ & $(0.44)$ & $(0.26)$ & $(0.21)$ & $(0.16)$ & $(0.10)$  \\
        & \multirow{2}{*}{Synthetic} & $13.93$ & $9.90$ & $28.05$ & $7.56$ & $13.82$ & $6.19$ & $6.70$ & $3.33$ & $2.55$ & $2.12$ & $5.06$  \\ 
        & & $(0.14)$ & $(0.23)$ & $(0.39)$ & $(0.13)$ & $(0.31)$ & $(0.07)$ & $(0.52)$ & $(0.25)$ & $(0.25)$ & $(0.31)$ & $(0.27)$  \\
    \midrule 
    \multirow{2}{*}{FMoW-\textsc{wilds}} & \multirow{2}{*}{Natural} & $5.15$ & $3.55$ & $34.64$ & $5.03$ & $5.58$ & $3.46$ & $5.08$ & $2.59$ & $2.33$ & $2.52$ & $\bf 2.22$  \\
    & & $(0.19)$ & $(0.41)$ & $(0.22)$ & $(0.29)$ & $(0.17)$ & $(0.37)$ & $(0.46)$ & $(0.32)$ & $(0.28)$ & $(0.25)$ & $(0.30)$  \\
    \midrule 
    \multirow{2}{*}{RxRx1-\textsc{wilds}} & \multirow{2}{*}{Natural}  & $6.17$ & $6.11$ & $21.05$ & $\bf 5.21$ & $6.54$ & $6.27$ & $6.82$ & $5.30$ & $\bf 5.20$ & $\bf 5.19$ & $5.63$ \\
   & & $(0.20)$ & $(0.24)$ & $(0.31)$ & $(0.18)$ & $(0.21)$ & $(0.20)$ & $(0.31)$ & $(0.30)$ & $(0.44)$ & $(0.43)$ & $(0.55)$  \\
    \midrule 
    \multirow{4}{*}{\parbox{1.8cm}{\textsc{Entity-13}} } & \multirow{2}{*}{Same} & $18.32$ & $14.38$ & $27.79$ & $13.56$ & $20.50$ & $13.22$ & $16.09$ & $9.35$ & $7.50$ & $7.80$ & $\bf 6.94$ \\
   & & $(0.29)$ & $(0.53)$ & $(1.18)$ & $(0.58)$ & $(0.47)$ & $(0.58)$ & $(0.84)$ & $(0.79)$ & $(0.65)$ & $(0.62)$ & $(0.71)$ 
 \\
    & \multirow{2}{*}{Novel} & $28.82$ & $24.03$ & $38.97$ & $22.96$ & $31.66$ & $22.61$ & $25.26$ & $17.11$ & $13.96$ & $14.75$ & $\bf 9.94$ \\
    & & $(0.30)$ & $(0.55)$ & $(1.32)$ & $(0.59)$ & $(0.54)$ & $(0.58)$ & $(1.08)$ & $(0.84)$ & $(0.93)$ & $(0.64)$ & $(0.78)$ \\
    \midrule 
    \multirow{4}{*}{\parbox{1.8cm}{\textsc{Entity-30}}} & \multirow{2}{*}{Same} & $16.91$ & $14.61$ & $26.84$ & $14.37$ & $18.60$ & $13.11$ & $13.74$ & $8.54$ & $7.94$ & $\bf 7.77$ & $8.04$  \\
    & & $(1.33)$ & $(1.11)$ & $(2.15)$ & $(1.34)$ & $(1.69)$ & $(1.30)$ & $(1.07)$ & $(1.47)$ & $(1.38)$ & $(1.44)$ & $(1.51)$  \\
    & \multirow{2}{*}{Novel} & $28.66$ & $25.83$ & $39.21$ & $25.03$ & $30.95$ & $23.73$ & $23.15$ & $15.57$ & $13.24$ & $12.44$ & $\bf 11.05$   \\
    & & $(1.16)$ & $(0.88)$ & $(2.03)$ & $(1.11)$ & $(1.64)$ & $(1.11)$ & $(0.51)$ & $(1.44)$ & $(1.15)$ & $(1.26)$ & $(1.13)$  \\
    \midrule
    \multirow{4}{*}{{\textsc{Nonliving-26}}} & \multirow{2}{*}{Same} & $17.43$ & $15.95$ & $27.70$ & $15.40$ & $18.06$ & $14.58$ & $16.99$ & $10.79$ & $\bf 10.13$ & $\bf 10.05$ & $10.29$ \\
    & & $(0.90)$ & $(0.86)$ & $(0.90)$ & $(0.69)$ & $(1.00)$ & $(0.78)$ & $(1.25)$ & $(0.62)$ & $(0.32)$ & $(0.46)$ & $(0.79)$  \\
    & \multirow{2}{*}{Novel} & $29.51$ & $27.75$ & $40.02$ & $26.77$ & $30.36$ & $25.93$ & $27.70$ & $19.64$ & $17.75$ & $16.90$ & $\bf 15.69$   \\ 
    & & $(0.86)$ & $(0.82)$ & $(0.76)$ & $(0.82)$ & $(0.95)$ & $(0.80)$ & $(1.42)$ & $(0.68)$ & $(0.53)$ & $(0.60)$ & $(0.83)$ \\
    \midrule 
    \multirow{4}{*}{\parbox{1.8cm}{\textsc{Living-17}}} & \multirow{2}{*}{Same} & $14.28$ & $12.21$ & $23.46$ & $11.16$ & $15.22$ & $10.78$ & $10.49$ & $4.92$ & $\bf 4.23$ & $\bf 4.19$ & $4.73$   \\
    & & $(0.96)$ & $(0.93)$ & $(1.16)$ & $(0.90)$ & $(0.96)$ & $(0.99)$ & $(0.97)$ & $(0.57)$ & $(0.42)$ & $(0.35)$ & $(0.24)$ \\
    & \multirow{2}{*}{Novel} & $28.91$ & $26.35$ & $38.62$ & $24.91$ & $30.32$ & $24.52$ & $22.49$ & $15.42$ & $13.02$ & $12.29$ & $\bf 10.34$ \\ 
    & & $(0.66)$ & $(0.73)$ & $(1.01)$ & $(0.61)$ & $(0.59)$ & $(0.74)$ & $(0.85)$ & $(0.59)$ & $(0.53)$ & $(0.73)$ & $(0.62)$  \\
    \bottomrule 
    \end{tabular}
  \end{adjustbox}  
    \caption{
     \update{\emph{Mean Absolute estimation Error (MAE) results for different datasets in our setup grouped by the nature of shift for ResNet model.} 
    `Same' refers to same subpopulation shifts and `Novel' refers novel subpopulation shifts. \update{We include details about the target sets considered in each shift in \tabref{table:dataset}.} 
    Post T denotes use of TS calibration on source.  
    Across all datasets, we observe that ATC achieves superior performance (lower MAE is better).  
    For GDE post T and pre T estimates match since TS doesn't alter the argmax prediction. Results reported by aggregating MAE numbers over $4$ different seeds. Values in parenthesis (i.e., $(\cdot)$) denote standard deviation values.}
    }\label{table:error_estimation_resnet_std}
\end{table}

\subsection{CIFAR pretraining Ablation} \label{app:cifar_result}
\begin{figure}[H]
    \centering
    \subfigure[]{\includegraphics[width=0.4\linewidth]{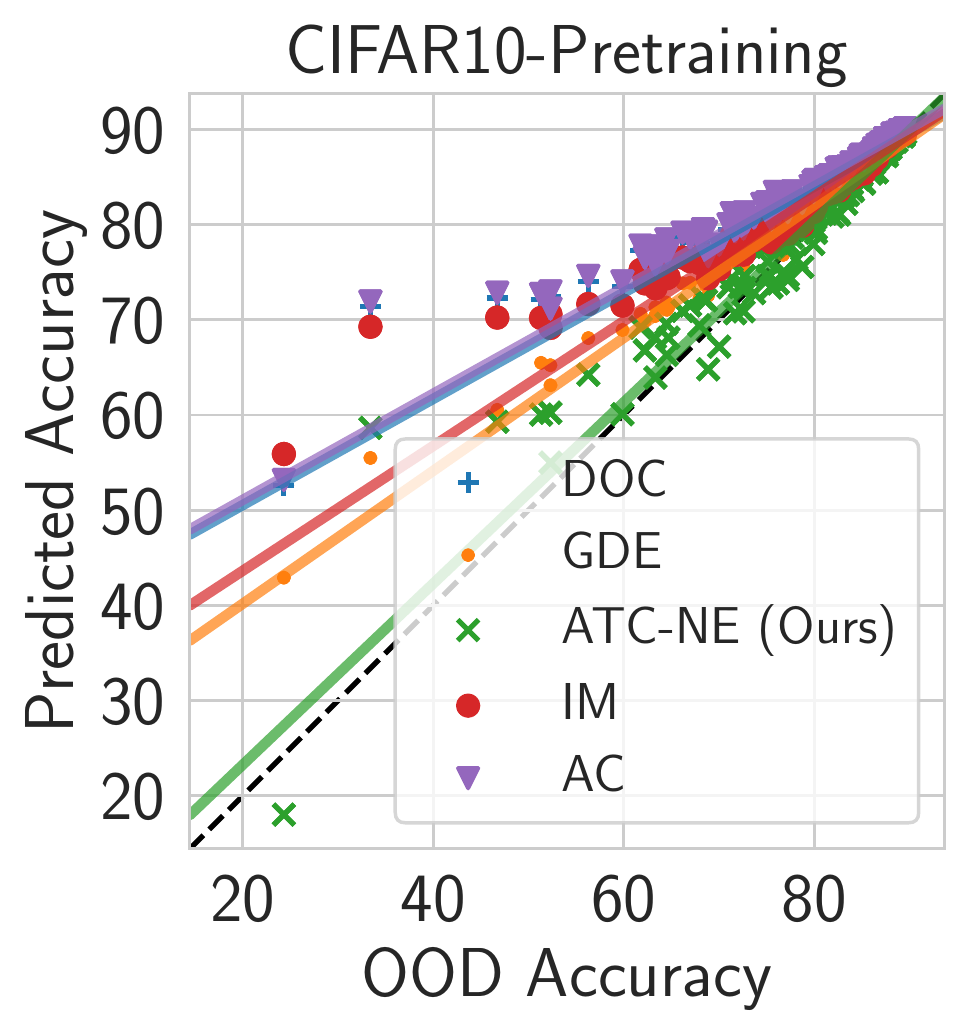}} \hfil
    \caption[]{Results with a pretrained DenseNet121 model on CIFAR10. We observe similar behaviour as that with a model trained from scratch.} 
    \label{fig:cifar_pretraiining}
  \end{figure}  

\subsection{\breeds~ results with regression model} \label{app:breeeds_ablation}

\begin{figure}[H]
    \centering
    \subfigure{\includegraphics[width=0.32\linewidth]{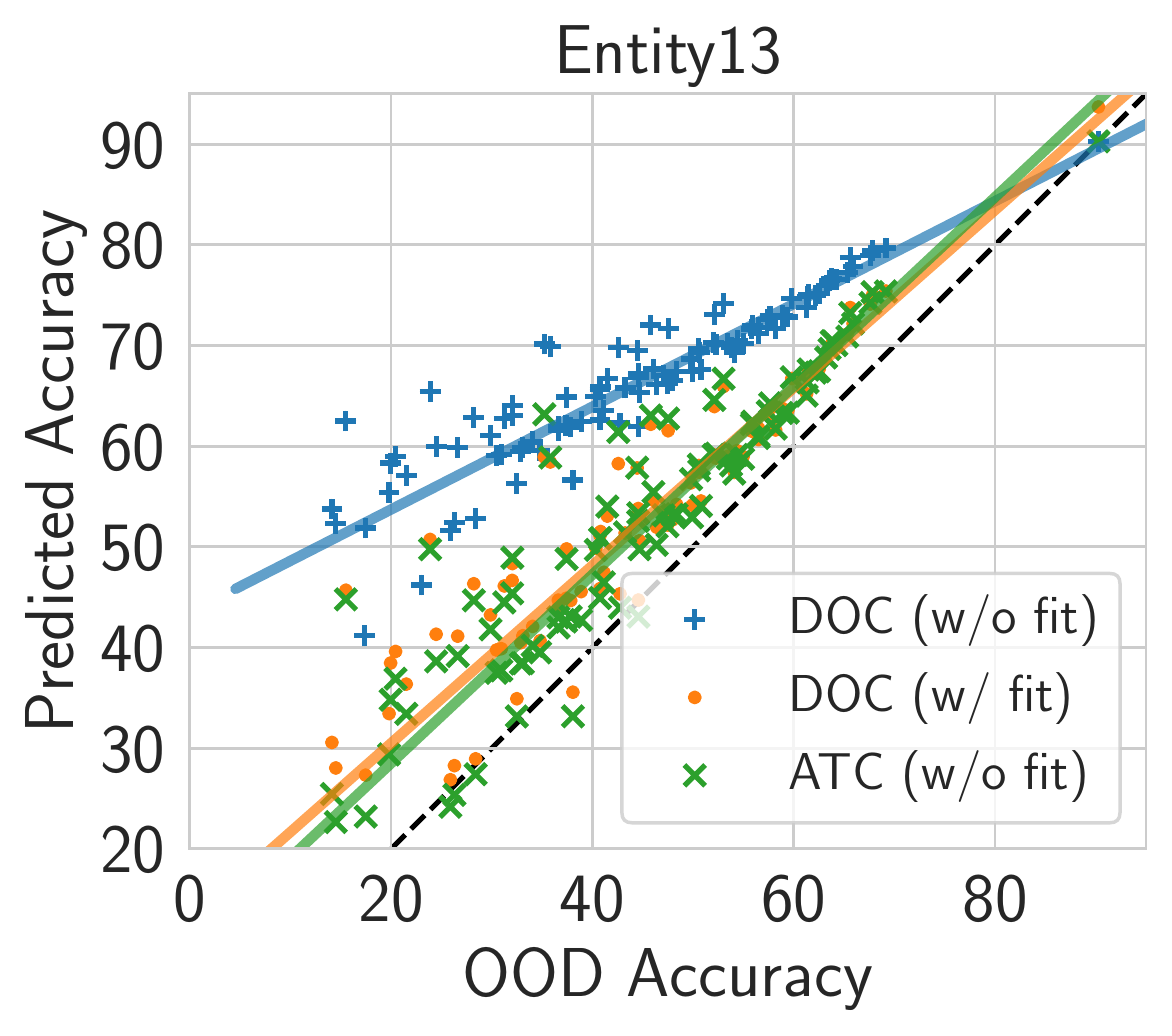}} \hfil
    \subfigure{\includegraphics[width=0.32\linewidth]{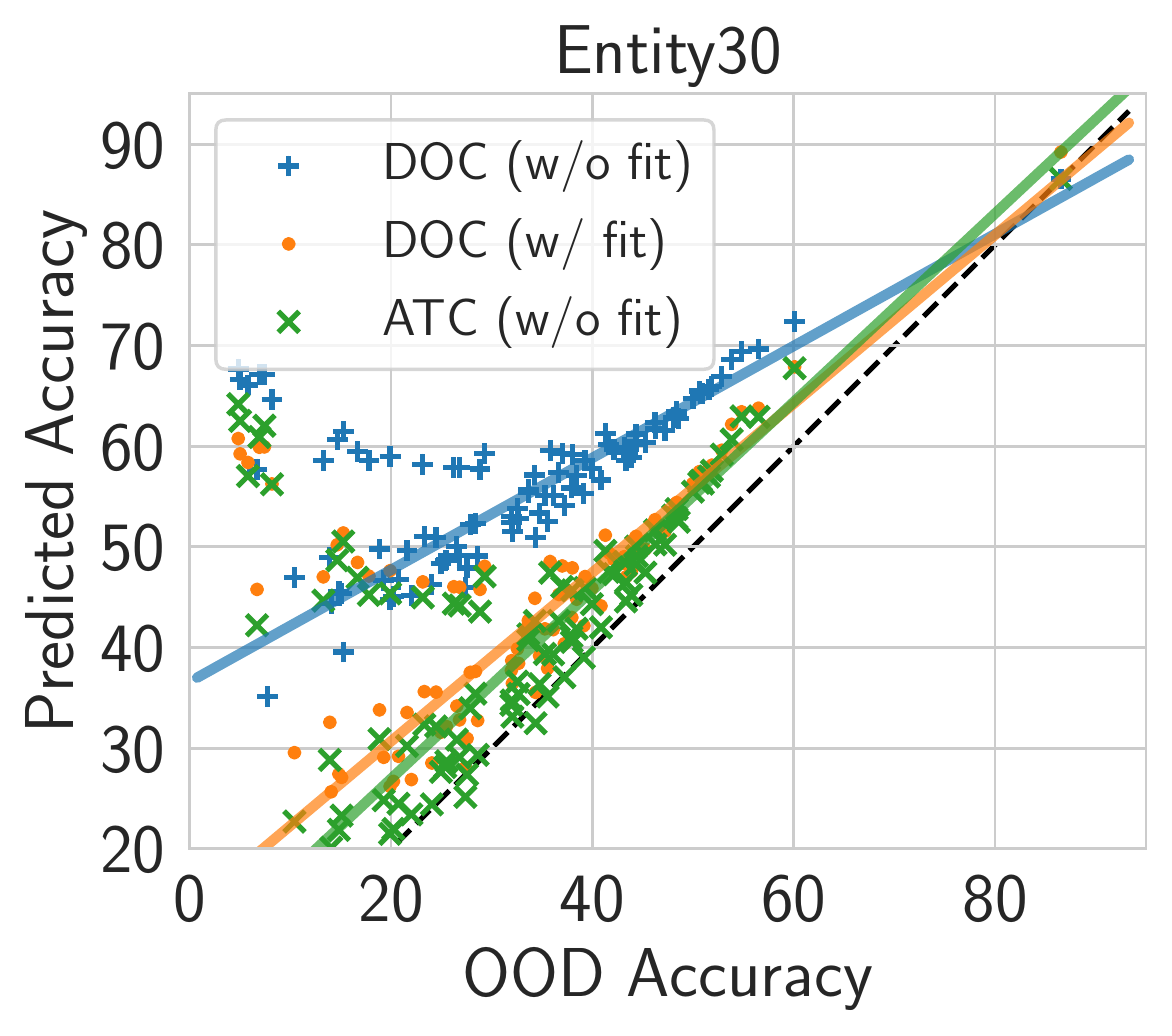}} \hfil   
    \subfigure{\includegraphics[width=0.32\linewidth]{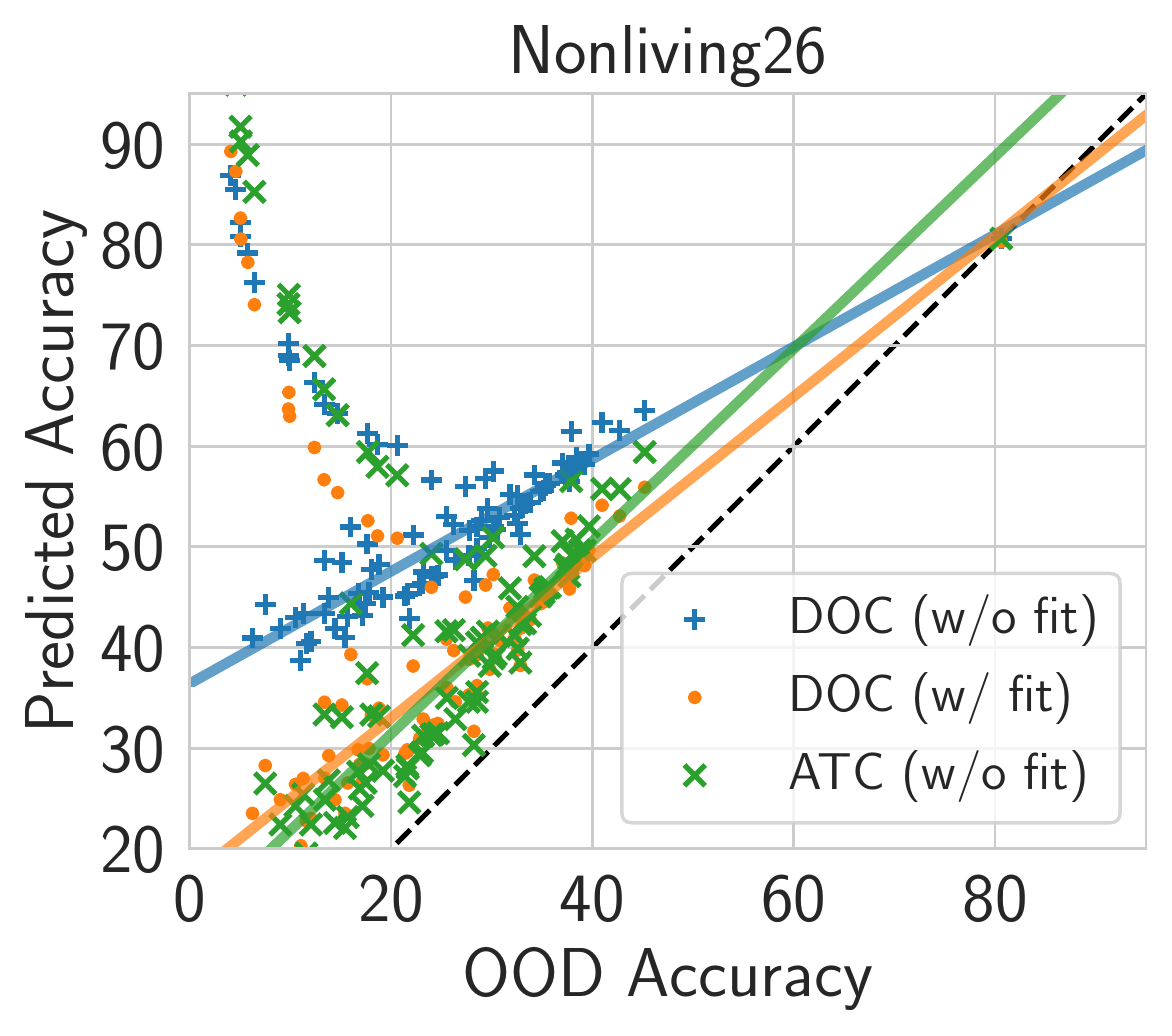}} 
    \caption{Scatter plots for DOC with linear fit. Results parallel to \figref{fig:ablation}(Middle) on other \breeds~dataset.}
    \vspace{-10pt}
    \label{fig:ablation_breeds}
\end{figure}  
  
\begin{table}[H]
    \centering
    \small
    \tabcolsep=0.12cm
    \renewcommand{\arraystretch}{1.2}
    \begin{tabular}{@{}*{4}{c}@{}}
    \toprule
    {Dataset}  &  {DOC (w/o fit)} & {DOC (w fit)}  & {ATC-MC (Ours) (w/o fit)}  \\
    \midrule
    \textsc{Living-17} & $24.32$ & $13.65$  &$\bf 10.07$\\
    \midrule
    \textsc{Nonliving-26} & $29.91$ &$\bf 18.13$ &$19.37$ \\
    \midrule
    \textsc{Entity-13} & $22.18$ & $8.63$ &$8.01$ \\
    \midrule
    \textsc{Entity-30} & $24.71$ & $12.28$ &$\bf 10.21$\\
    \bottomrule 
    \end{tabular}
    \caption{
    \emph{Mean Absolute estimation Error (MAE) results for BREEDs datasets with novel populations in our setup.} 
    After fitting a robust linear model for DOC on same subpopulation, we show predicted accuracy on different subpopulations with fine-tuned DOC (i.e., DOC (w/ fit)) and compare with ATC without any regression model, i.e., ATC (w/o fit).  While observe substantial improvements in MAE from DOC (w/o fit) to DOC (w/ fit), ATC (w/o fit) continues to outperform even DOC (w/ fit).}\label{table:breeds_regression}
\end{table}

\newpage

\begin{figure}[H]
    \centering
    \subfigure{\includegraphics[width=0.32\linewidth]{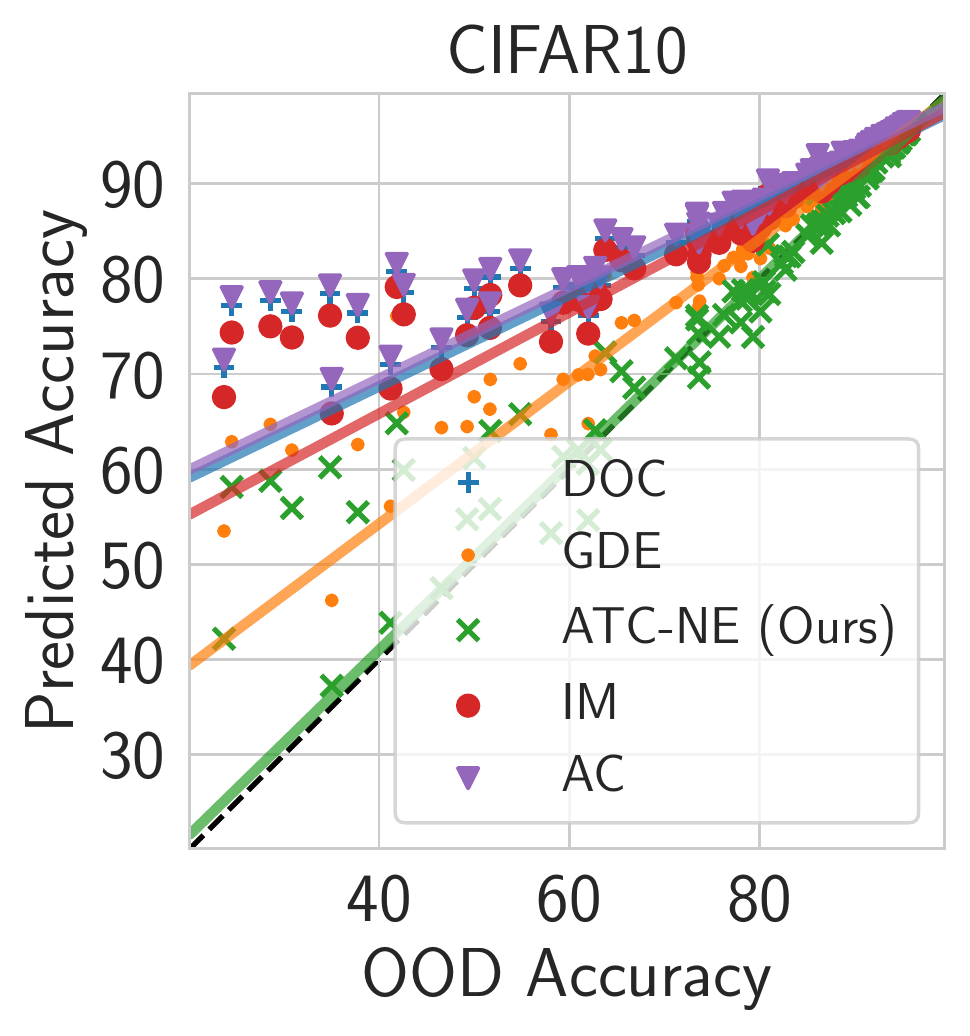}} \hfil
    \subfigure{\includegraphics[width=0.32\linewidth]{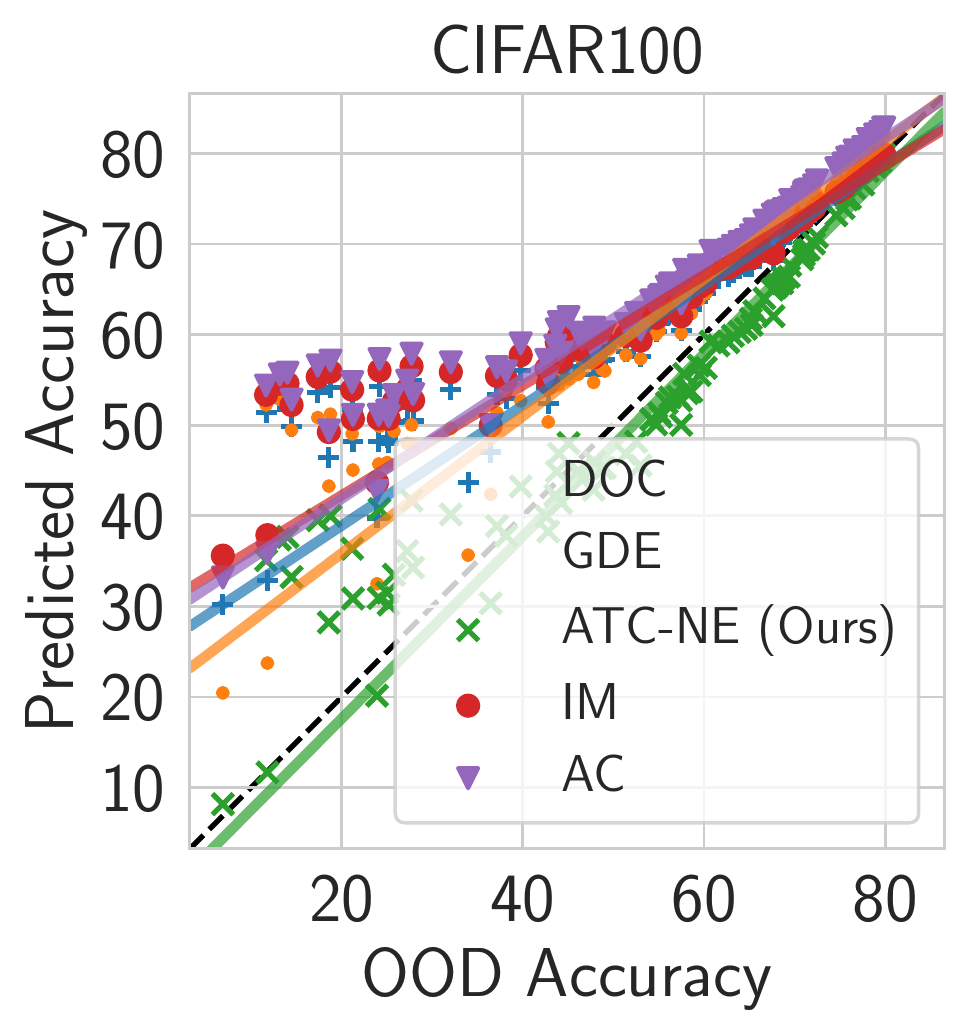}} \hfil
    \subfigure{\includegraphics[width=0.32\linewidth]{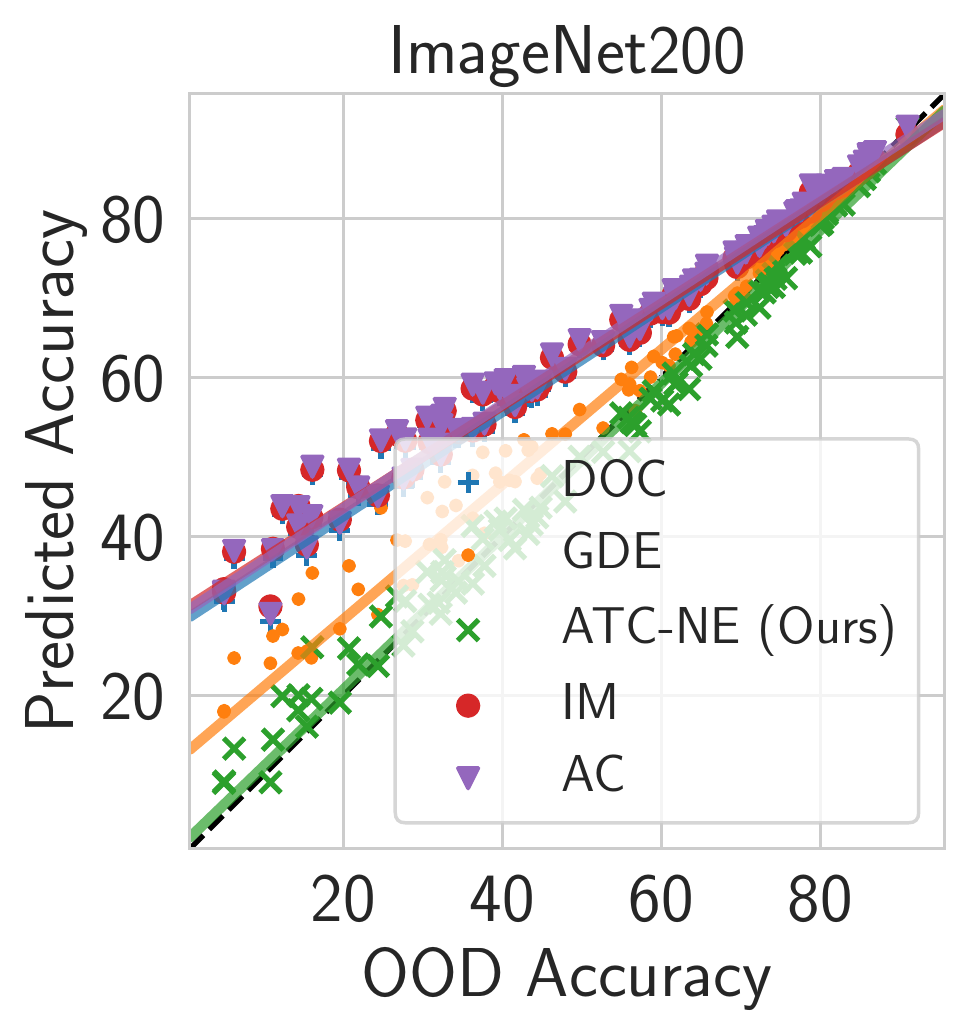}} \hfil   
    \subfigure{\includegraphics[width=0.32\linewidth]{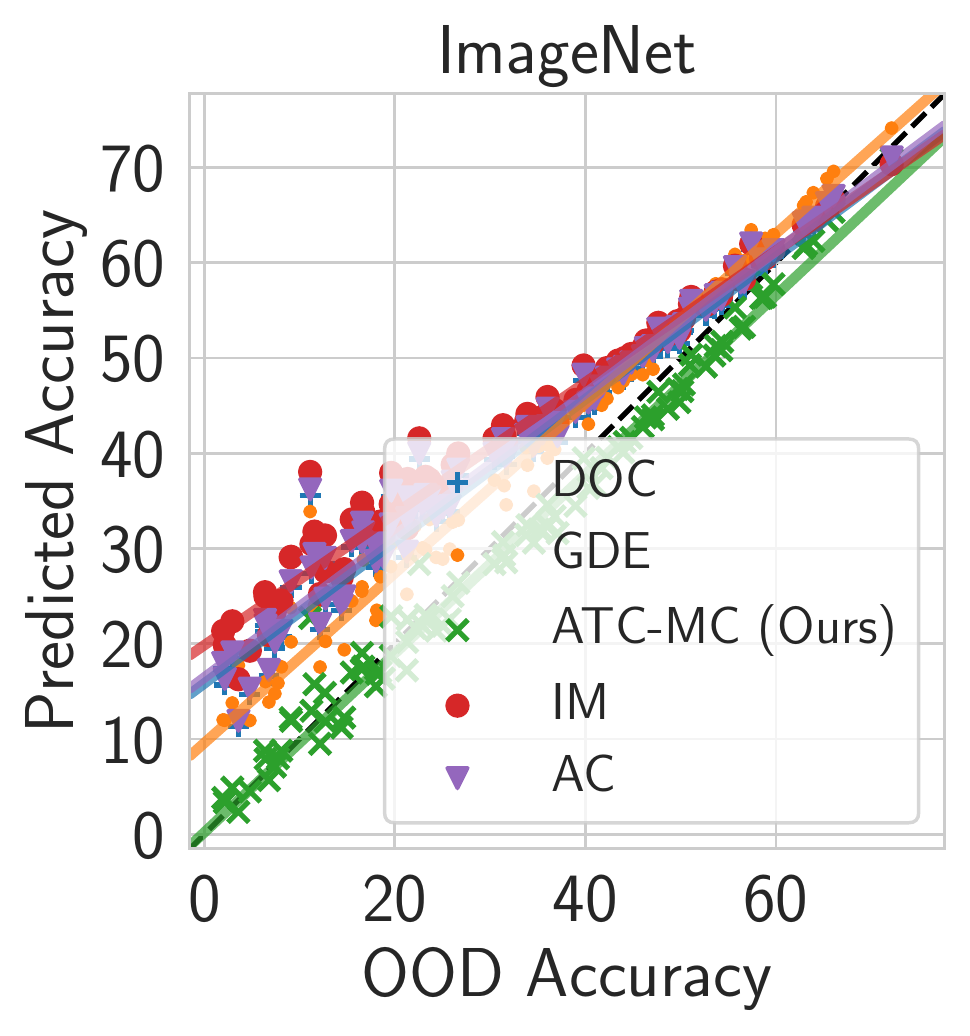}} \hfil
    \subfigure{\includegraphics[width=0.32\linewidth]{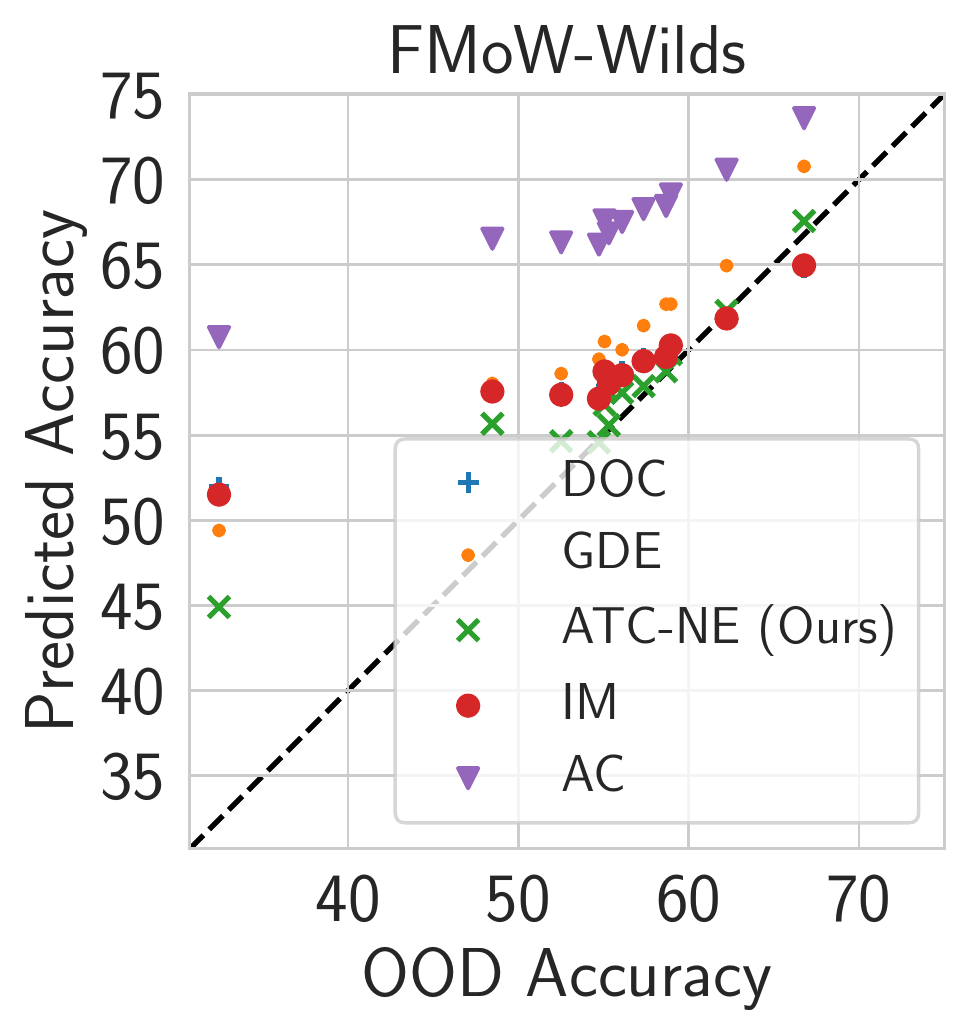}} \hfil
    \subfigure{\includegraphics[width=0.32\linewidth]{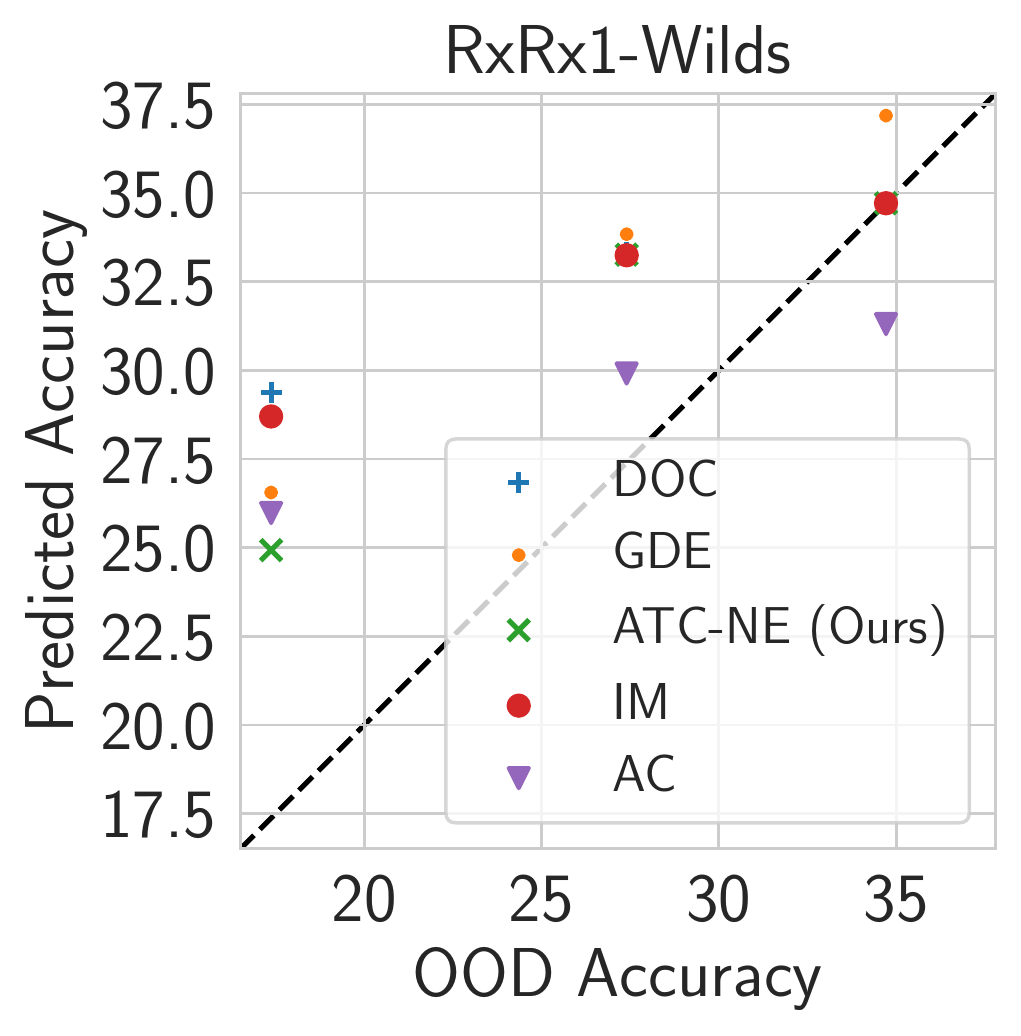}} \hfil
    \subfigure{\includegraphics[width=0.32\linewidth]{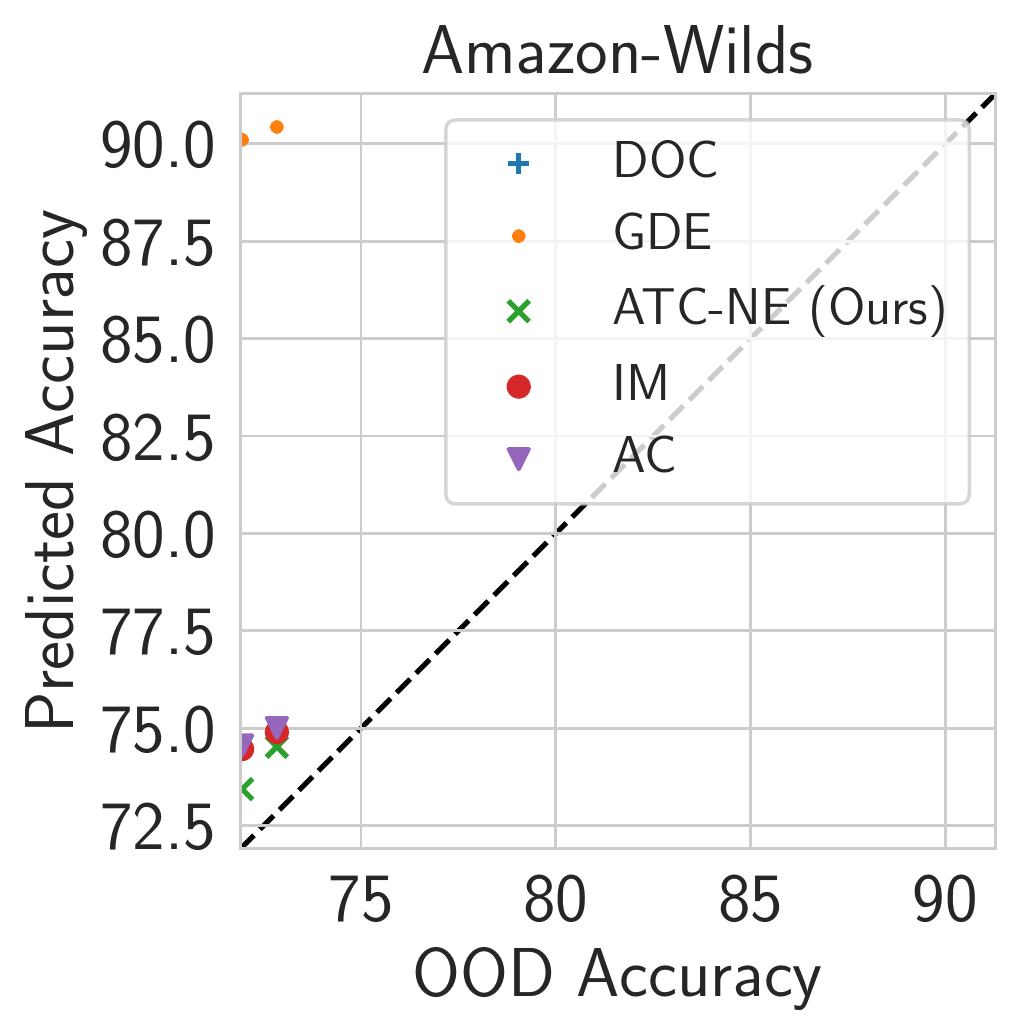}} \hfil
    \subfigure{\includegraphics[width=0.32\linewidth]{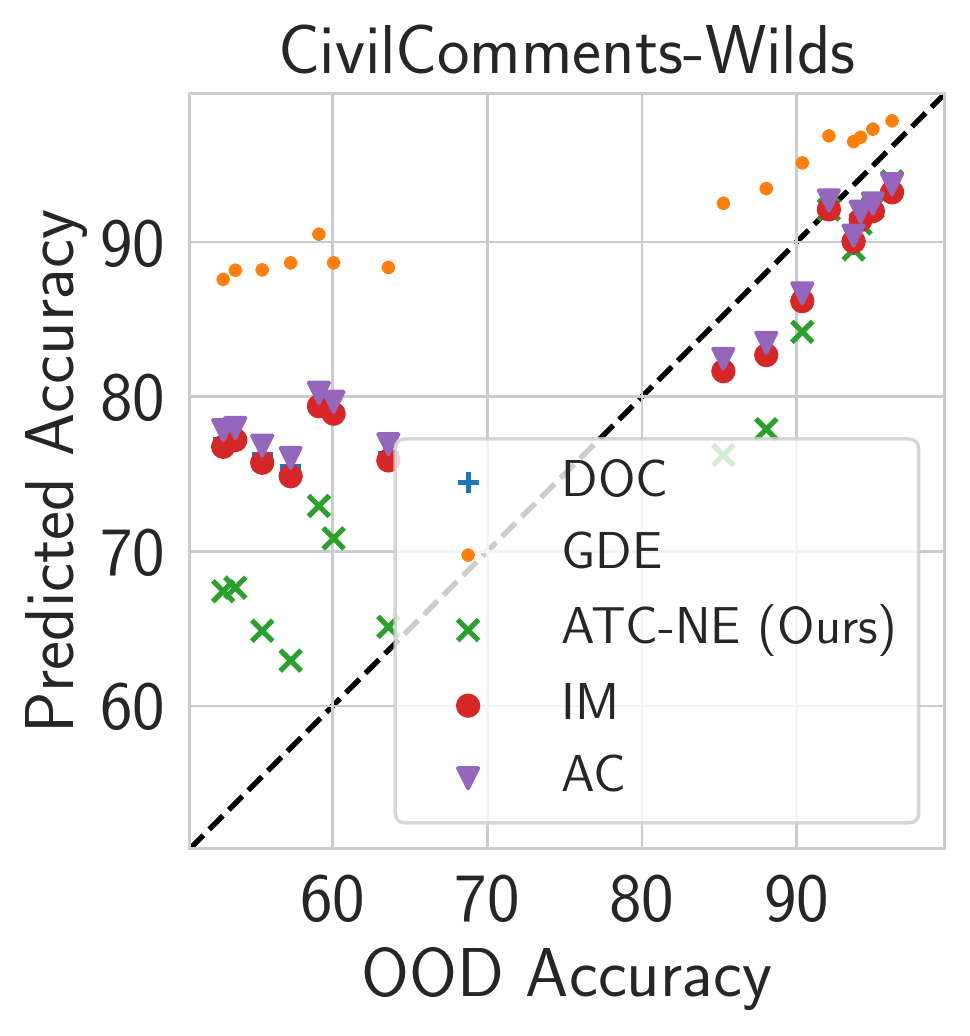}} \hfil
    \subfigure{\includegraphics[width=0.32\linewidth]{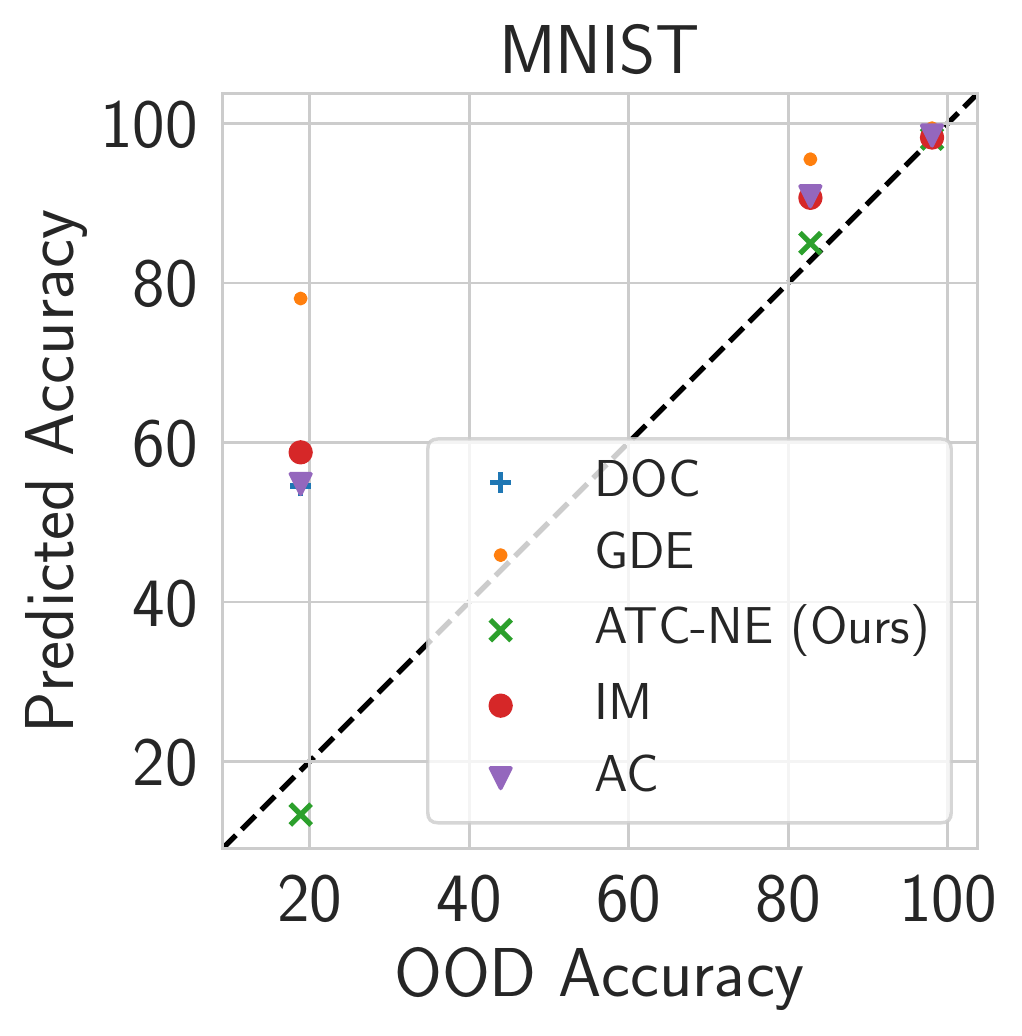}} \hfil
    \subfigure{\includegraphics[width=0.32\linewidth]{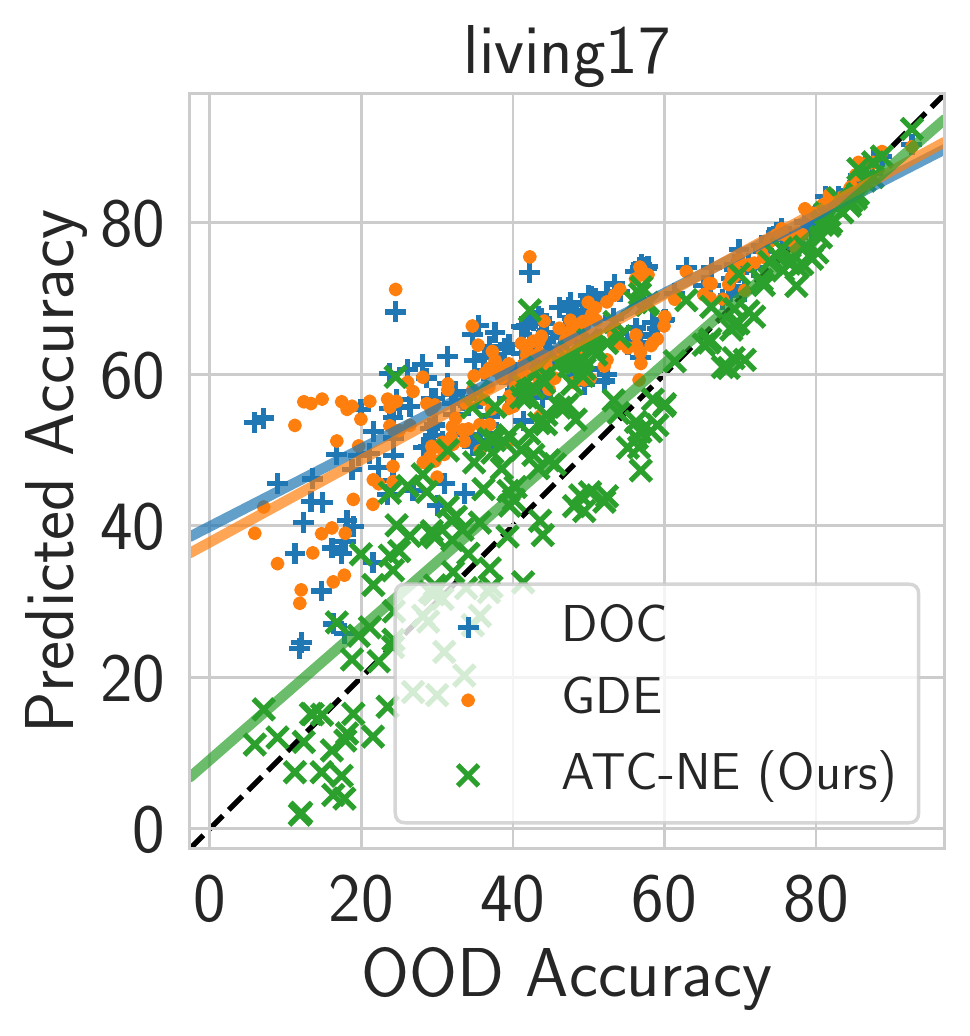}} \hfil
    \subfigure{\includegraphics[width=0.32\linewidth]{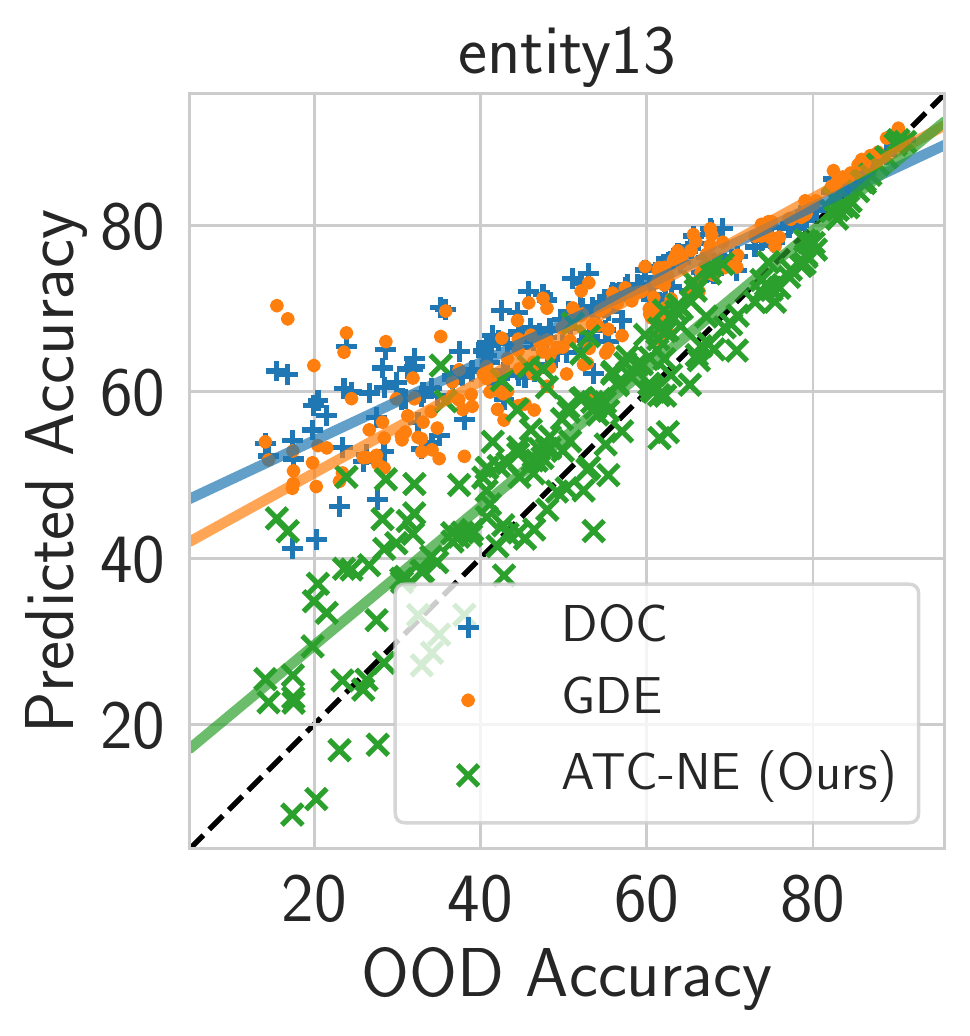}} \hfil   
    \subfigure{\includegraphics[width=0.32\linewidth]{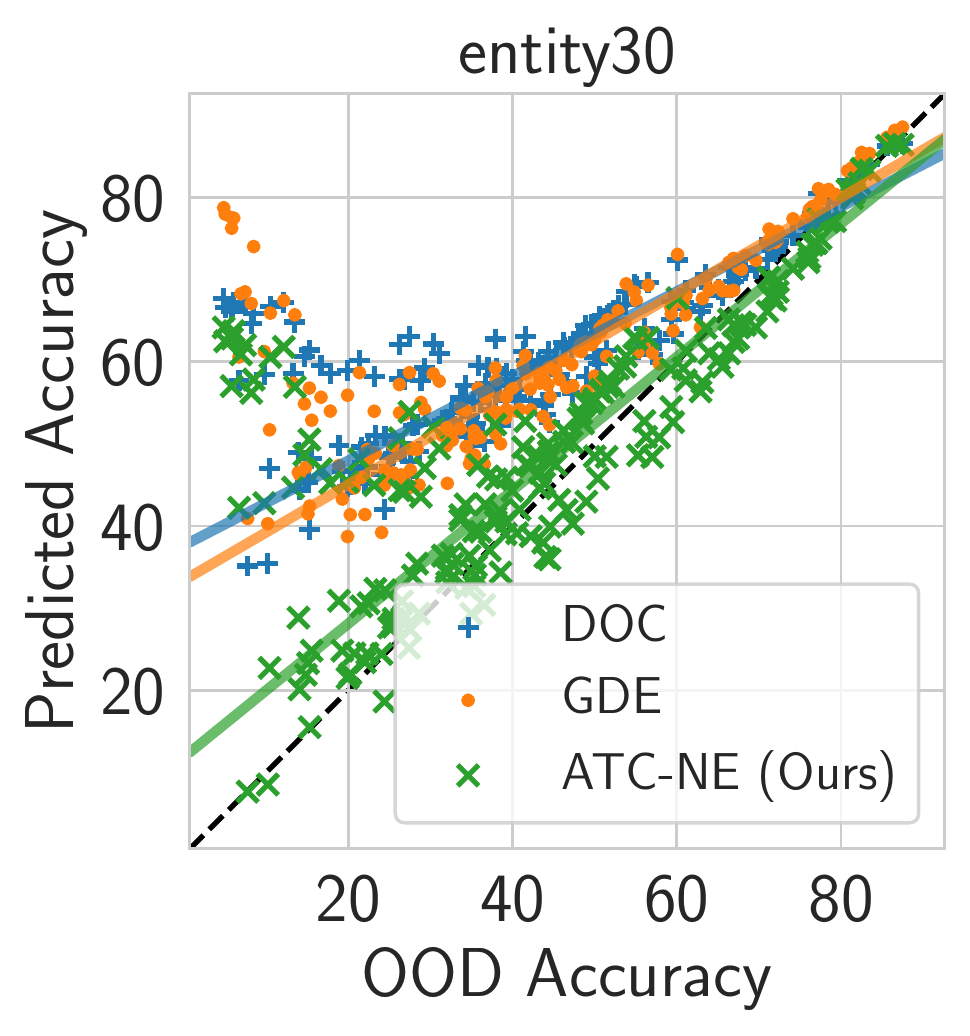}} \hfil
    \caption{\update{Scatter plot of predicted accuracy versus (true) OOD accuracy. For vision datasets except MNIST we use a  DenseNet121 model. For MNIST, we use a FCN. For language datasets, we use DistillBert-base-uncased. Results reported by aggregating accuracy numbers over $4$ different seeds.} } 
    \label{fig:scatter_plot_densenet}
  \end{figure}  
  
\begin{figure}[H]
    \centering
    \subfigure{\includegraphics[width=0.32\linewidth]{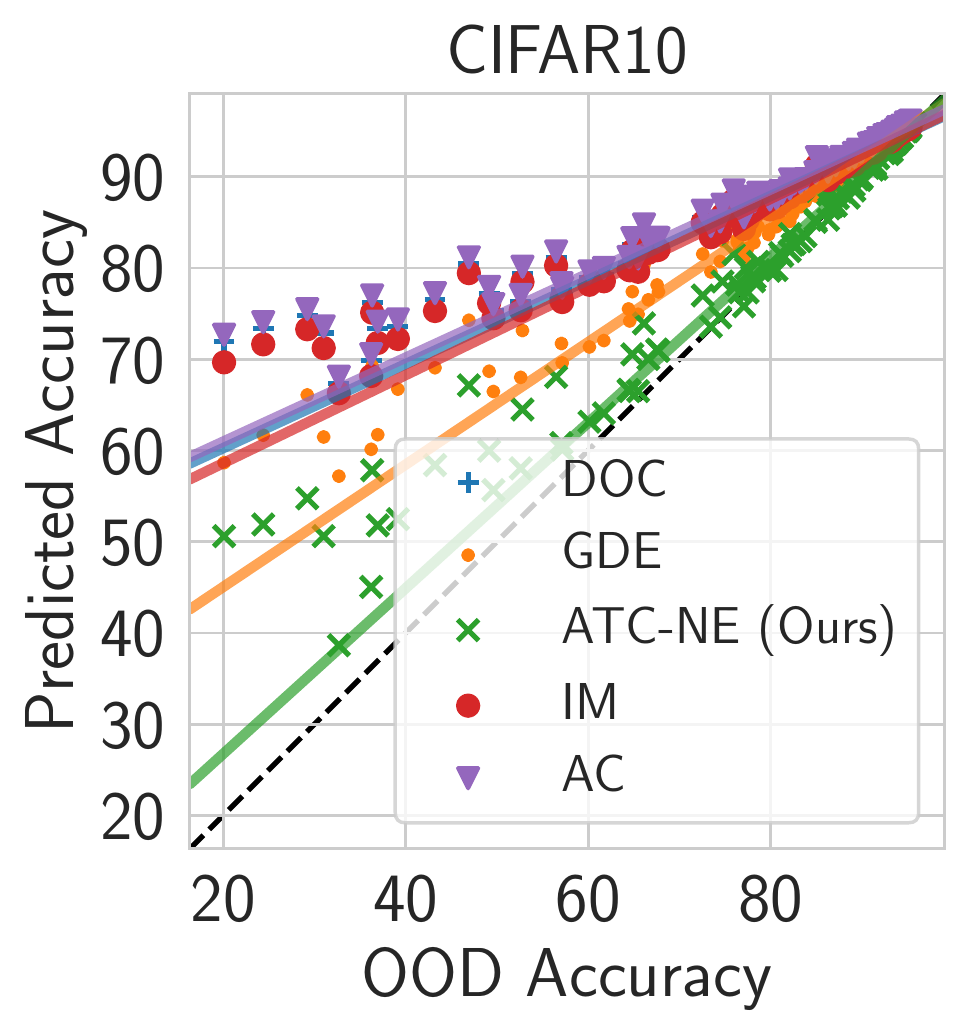}} \hfil
    \subfigure{\includegraphics[width=0.32\linewidth]{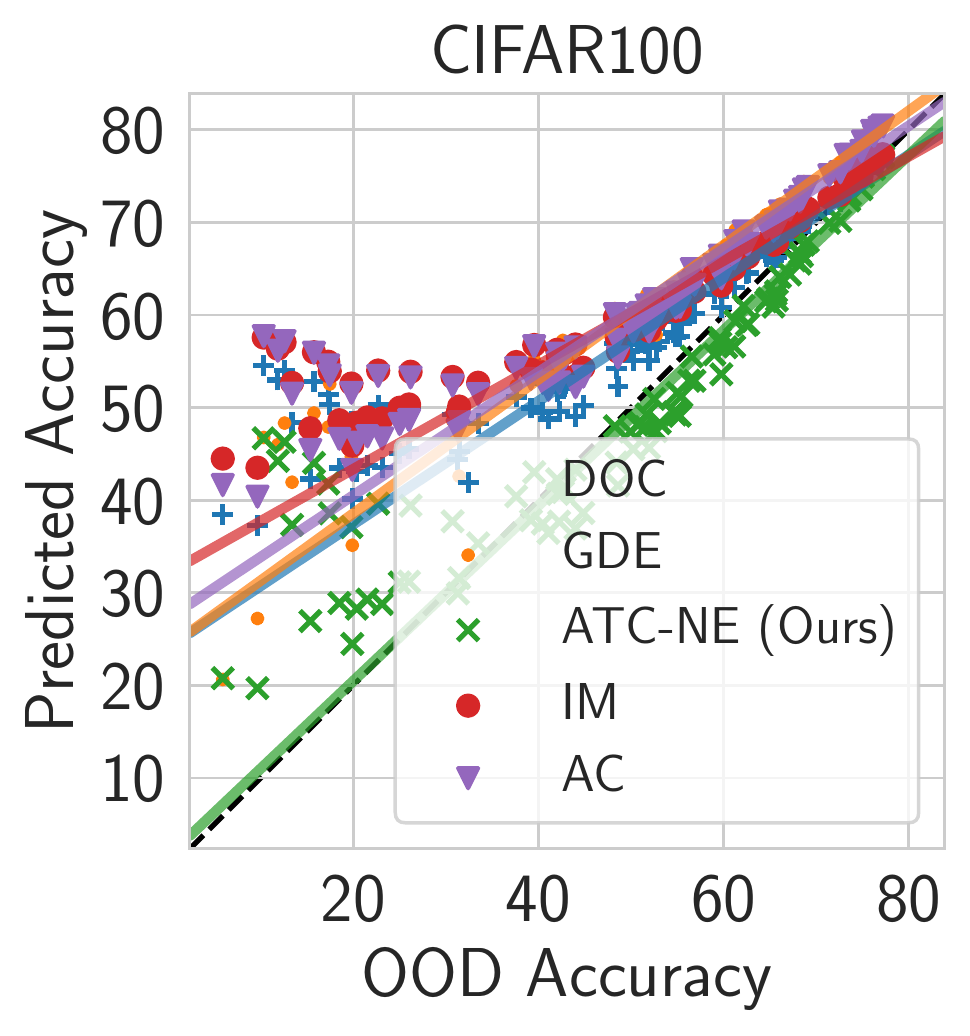}} \hfil
    \subfigure{\includegraphics[width=0.32\linewidth]{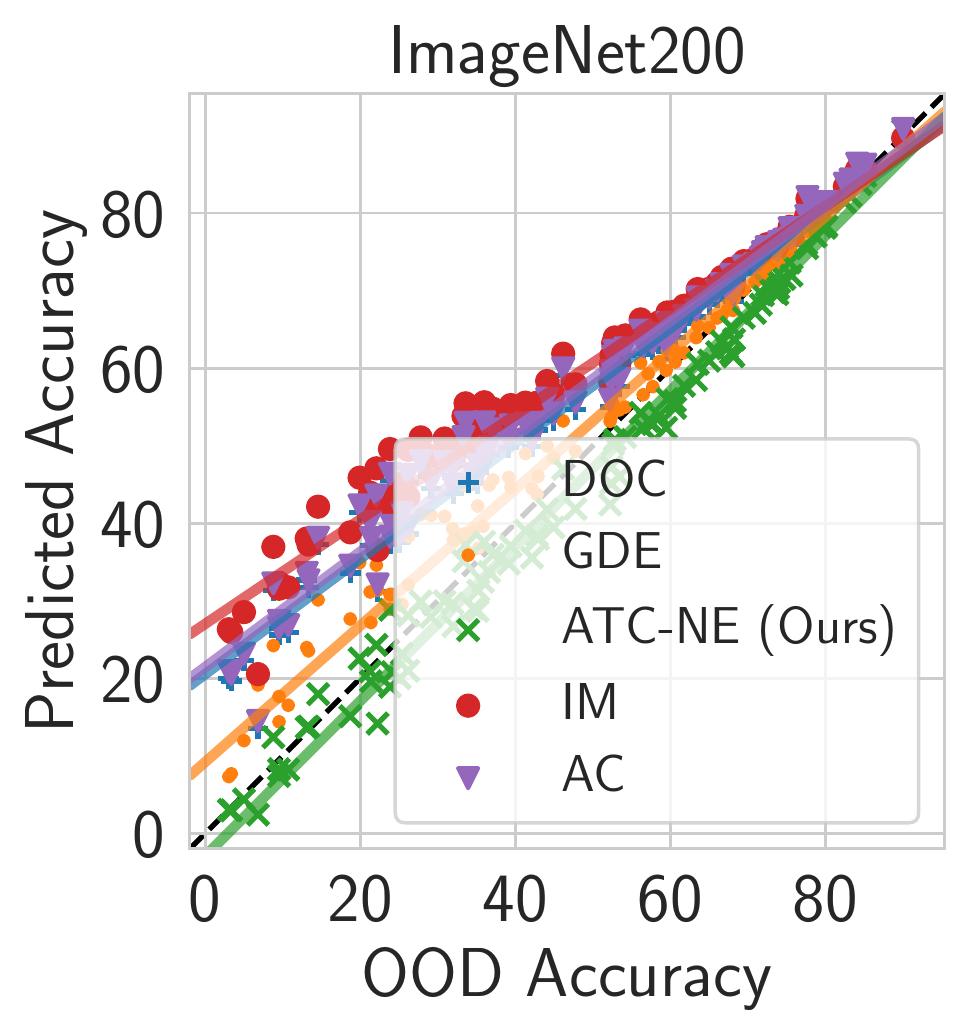}} \hfil   
    \subfigure{\includegraphics[width=0.32\linewidth]{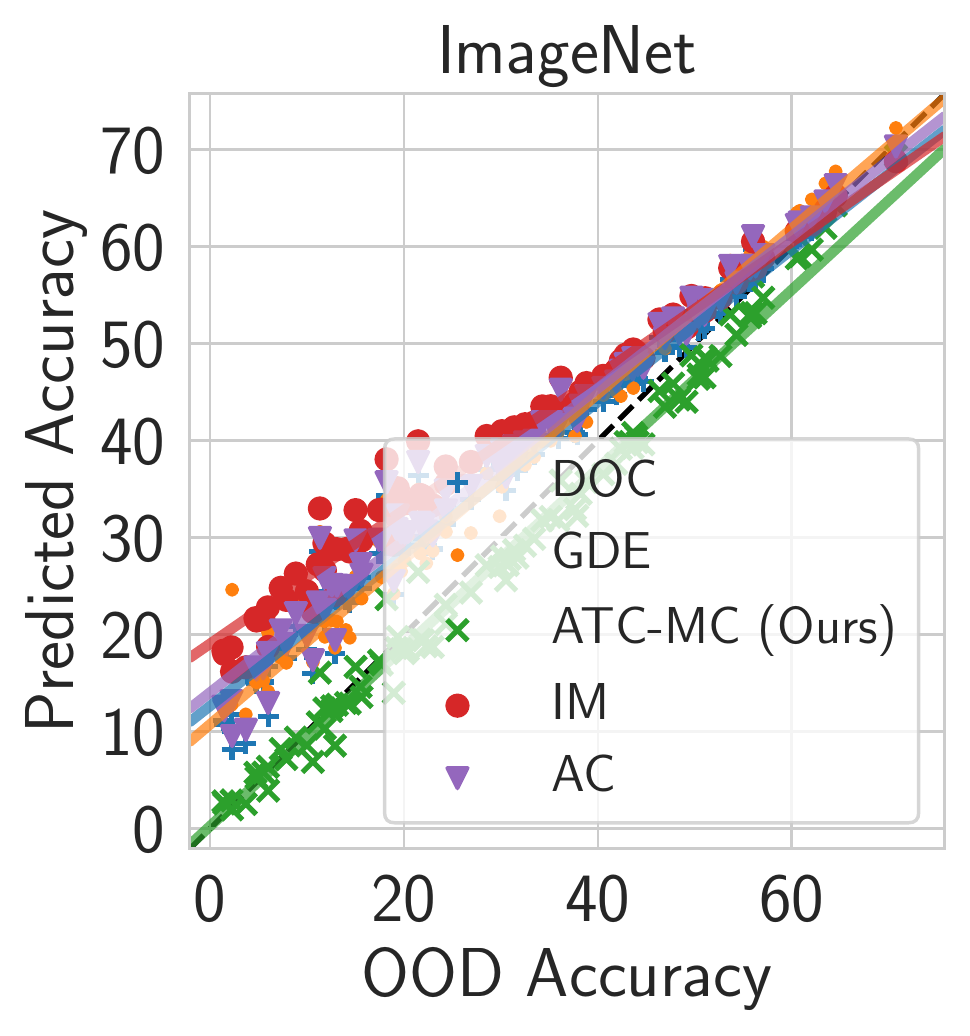}} \hfil
    \subfigure{\includegraphics[width=0.32\linewidth]{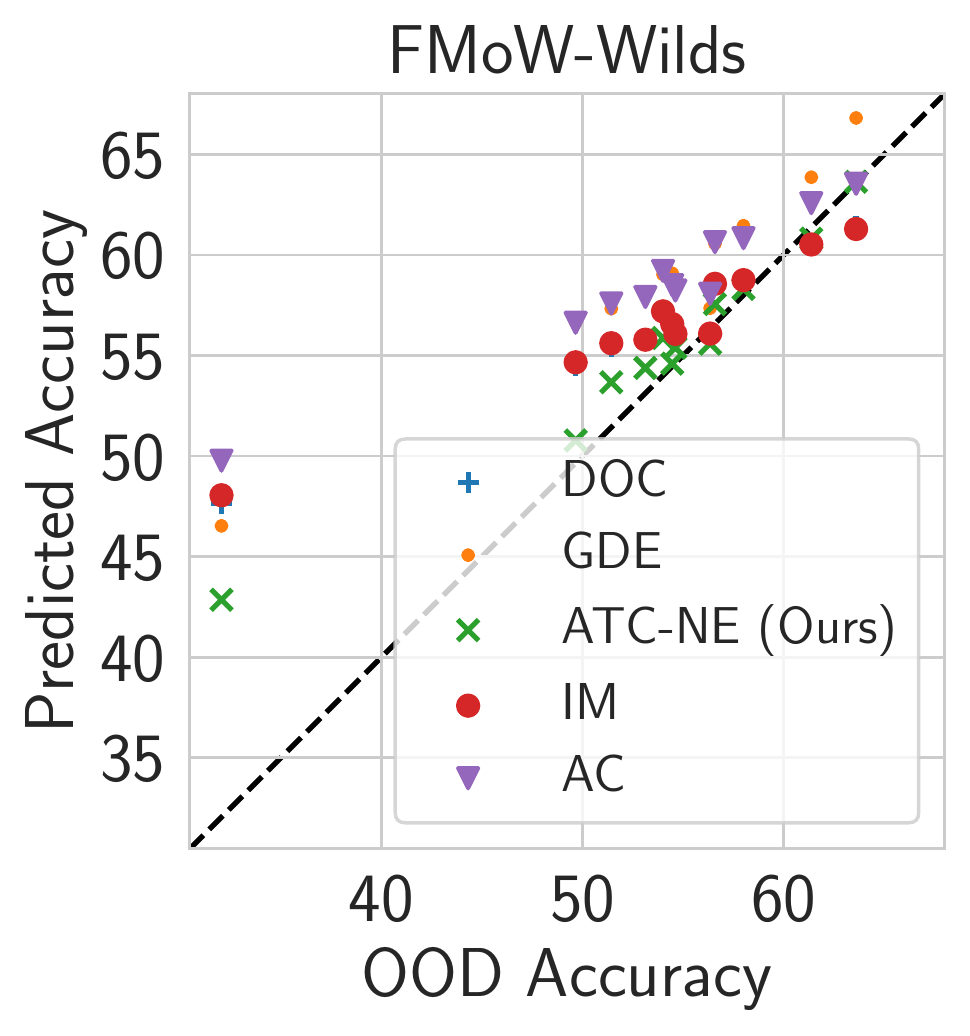}} \hfil
    \subfigure{\includegraphics[width=0.32\linewidth]{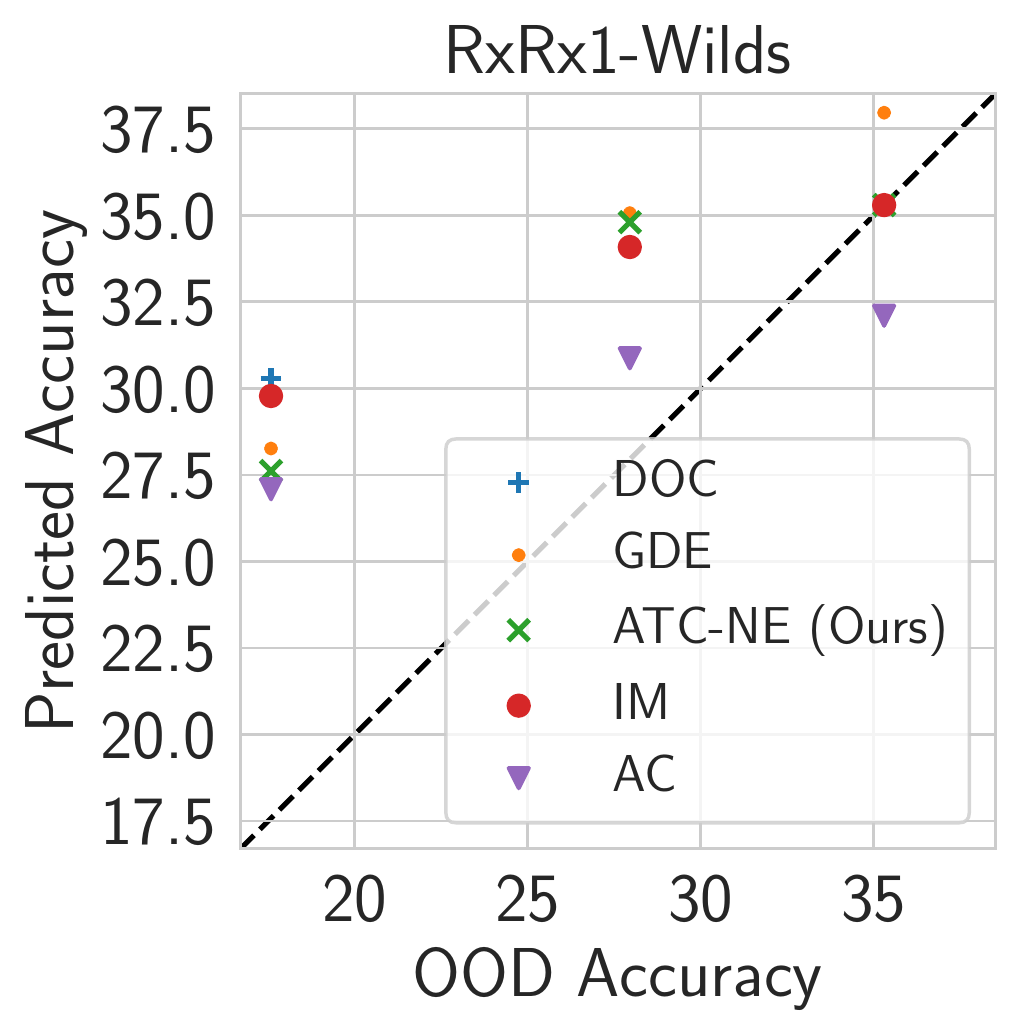}} \hfil
    \subfigure{\includegraphics[width=0.32\linewidth]{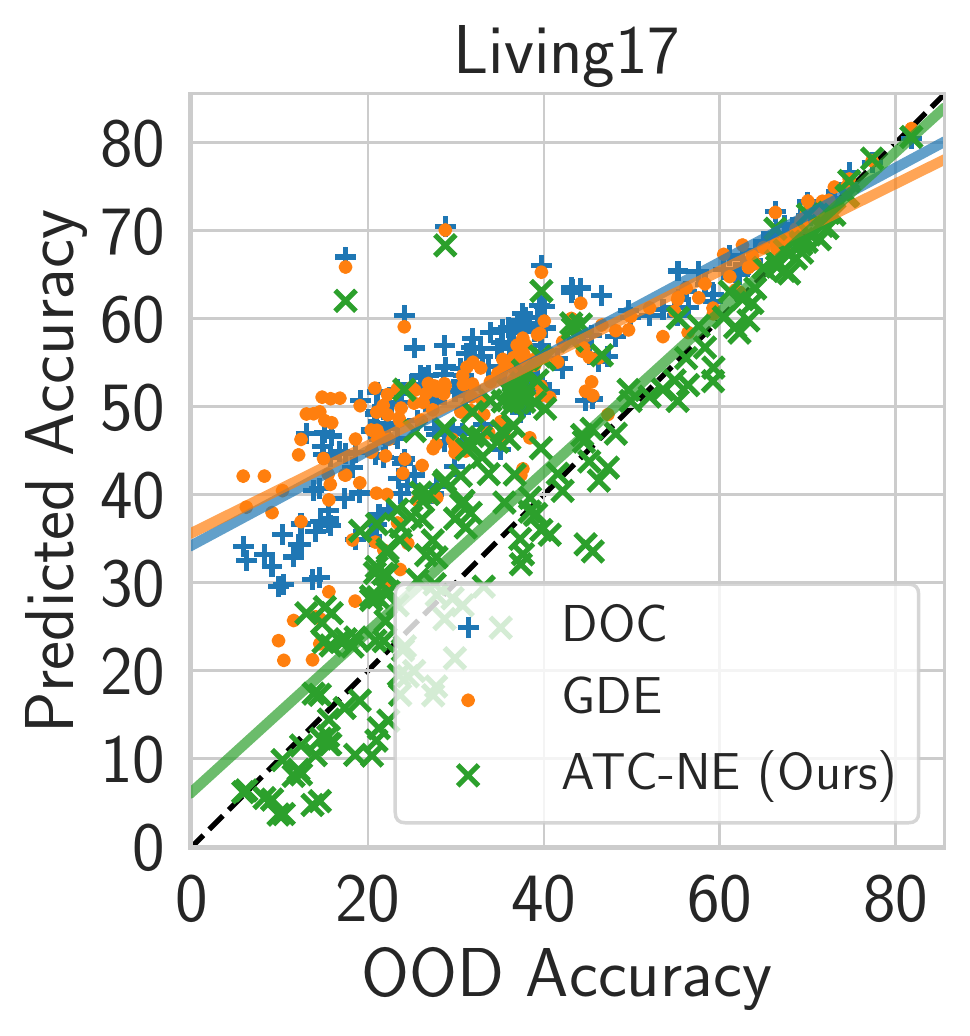}} \hfil
    \subfigure{\includegraphics[width=0.32\linewidth]{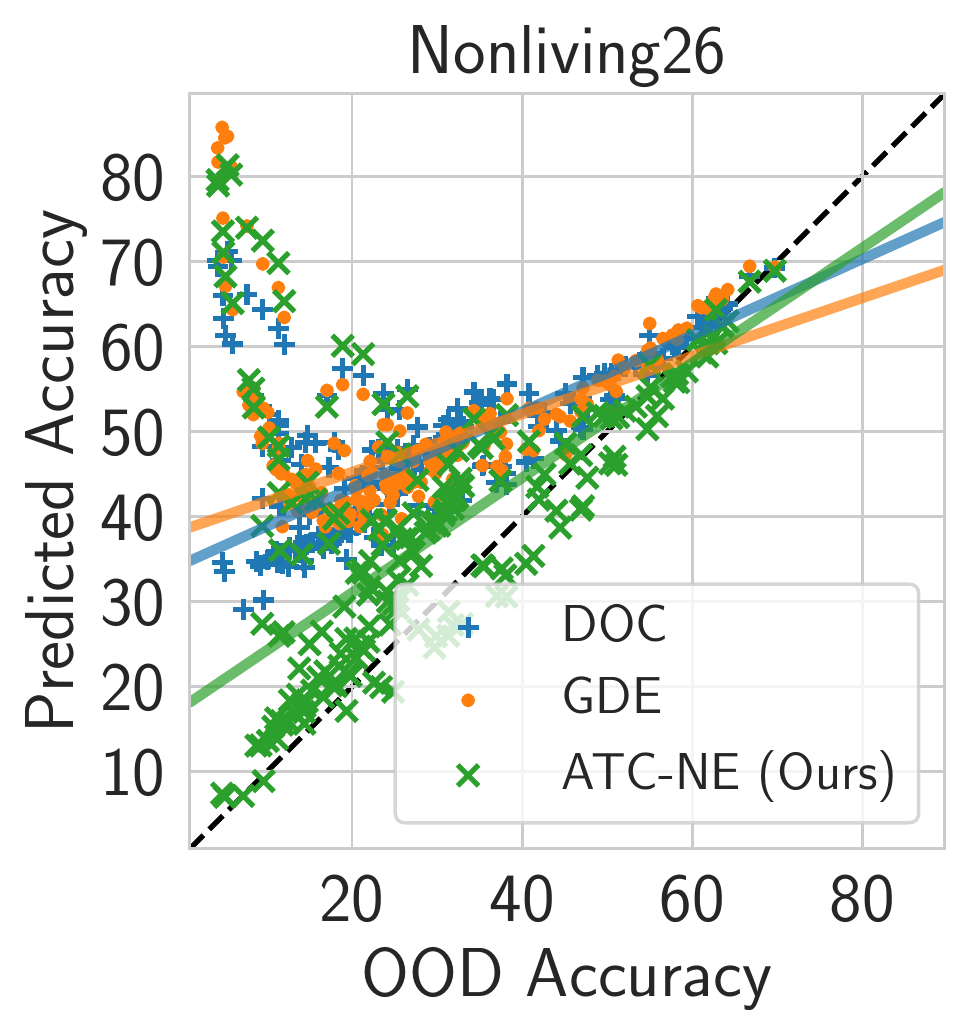}} \hfil
    \subfigure{\includegraphics[width=0.32\linewidth]{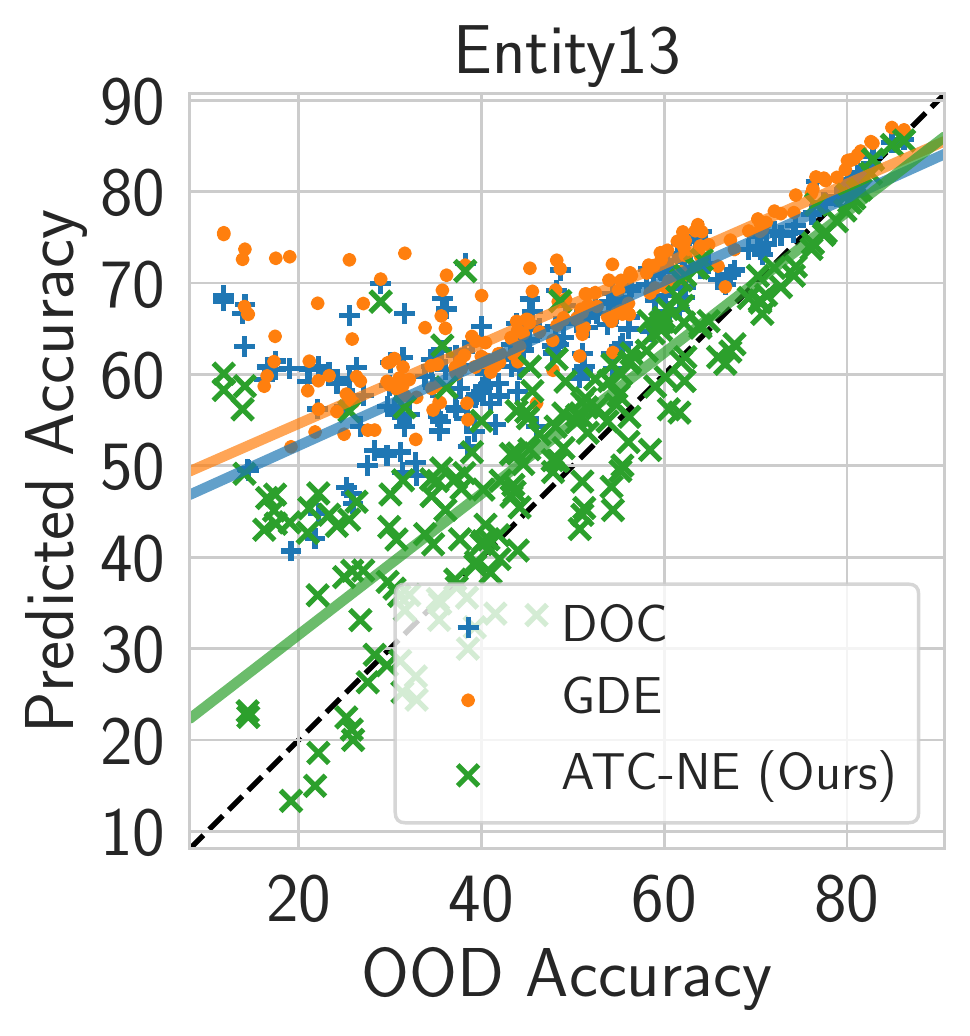}} \hfil   
    \subfigure{\includegraphics[width=0.32\linewidth]{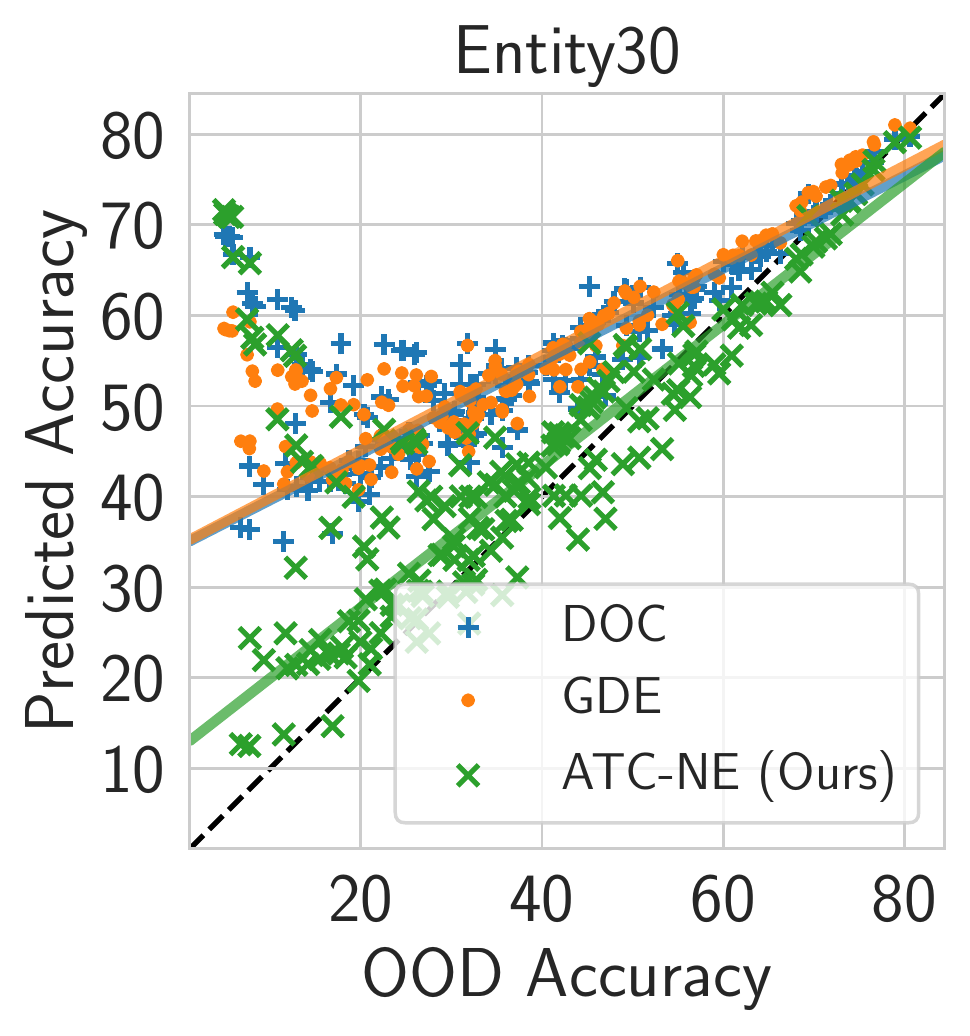}} 
    \vspace{-5pt}
    \caption{\update{Scatter plot of predicted accuracy versus (true) OOD accuracy for vision datasets except MNIST with a ResNet50 model. Results reported by aggregating MAE numbers over $4$ different seeds. }} 
    \vspace{-15pt}
    \label{fig:scatter_plot_resnet}
  \end{figure}

\end{document}